\documentclass[sn-mathphys,Numbered]{sn-jnl}

\usepackage{graphicx}%
\usepackage{multirow}%
\usepackage{amsmath,amssymb,amsfonts}%
\usepackage{mathrsfs}
\usepackage{amsthm}
\usepackage{MnSymbol}

\usepackage{tikz}
\usepackage{bbm}
\usepackage[title]{appendix}%
\usepackage{xcolor}%
\usepackage{textcomp}%
\usepackage{manyfoot}%
\usepackage{booktabs}%
\usepackage{algorithm}%
\usepackage{algorithmicx}%
\usepackage{algpseudocode}%
\usepackage{listings}%
\usepackage{arydshln}

\def\thm@space@setup{%
\thm@preskip=12pt%
\thm@postskip=12pt}
\newtheoremstyle{thmstyleone}
{18pt plus2pt minus1pt}
{18pt plus2pt minus1pt}
{\itshape}
{0pt}
{\bfseries}
{.}
{.5em}
{\thmname{#1} \thmnumber{#2}%
  \thmnote{ {\the\thm@notefont(#3)}}}
\newtheoremstyle{thmstyletwo}
{18pt plus2pt minus1pt}
{18pt plus2pt minus1pt}
{\normalfont}
{0pt}
{\itshape}
{.}
{.5em}
{\thmname{#1} \thmnumber{#2}%
  \thmnote{ {\the\thm@notefont(#3)}}}
\newtheoremstyle{thmstylethree}
{18pt plus2pt minus1pt}
{18pt plus2pt minus1pt}
{\normalfont}
{0pt}
{\bfseries}
{.}
{.5em}
{\thmname{#1} \thmnumber{#2}%
  \thmnote{ {\the\thm@notefont(#3)}}}
\newtheoremstyle{thmstylefour}
{18pt plus2pt minus1pt}
{18pt plus2pt minus1pt}
{\normalfont}
{0pt}
{\itshape}
{.}
{.5em}
{\thmname{#1} \thmnumber{#2}%
  \thmnote{ {\the\thm@notefont(#3)}}}%

\theoremstyle{thmstyleone}%
\newtheorem{theorem}{Theorem}
\newtheorem{proposition}{Proposition}
\newtheorem{lemma}{Lemma}
\newtheorem{corollary}{Corollary}

\theoremstyle{thmstyletwo}%

\theoremstyle{thmstylethree}%
\newtheorem{definition}{Definition}%

\theoremstyle{thmstylefour}
\newtheorem{remark}{Remark}%
\DeclareMathOperator*{\argmin}{\arg\!min}

\begin{document}

\title[On the latent dimension of DL-ROMs for PDEs parametrized by random fields]{On the latent dimension of deep autoencoders for reduced order modeling of PDEs parametrized by random fields}


\author*[1]{\fnm{Nicola Rares}\sur{Franco}}\email{nicolarares.franco@polimi.it}

\author[1]{\fnm{Daniel}\sur{Fraulin}}\email{daniel.fraulin@mail.polimi.it}

\author[1]{\fnm{Andrea} \sur{Manzoni}}\email{andrea1.manzoni@polimi.it}

\author[1]{\fnm{Paolo} \sur{Zunino}}\email{paolo.zunino@polimi.it}

\affil[1]{\orgdiv{MOX, Department of Mathematics}, \orgname{Politecnico di Milano}, \orgaddress{\street{Piazza Leonardo da Vinci, 32}, \city{Milan}, \postcode{20133}, \state{Italy}}\vspace{-0.5em}}


\abstract{
Deep Learning is having a remarkable impact on the design of Reduced Order Models (ROMs) for Partial Differential Equations (PDEs), where it is exploited as a powerful tool for tackling complex problems for which classical methods might fail. In this respect, deep autoencoders play a fundamental role, as they provide an extremely flexible tool for reducing the dimensionality of a given problem by leveraging on the nonlinear capabilities of neural networks. Indeed, starting from this paradigm, several successful approaches have already been developed, which are here referred to as Deep Learning-based ROMs (DL-ROMs). Nevertheless, when it comes to stochastic problems parameterized by random fields, the current understanding of DL-ROMs is mostly based on empirical evidence: in fact, their theoretical analysis is currently limited to the case of PDEs depending on a finite number of (deterministic) parameters. The purpose of this work is to extend the existing literature by providing some theoretical insights about the use of DL-ROMs in the presence of stochasticity generated by random fields. In particular, we derive explicit error bounds that can guide domain practitioners when choosing the latent dimension of deep autoencoders. We evaluate the practical usefulness of our theory by means of numerical experiments, showing how our analysis can significantly impact the performance of DL-ROMs.}




\newcommand{\mup}{\boldsymbol{\mu}}
\newcommand{\mupi}{\mu_{n}}
\newcommand{\sigmapi}{\sigma_{p}}
\newcommand{\nup}{\boldsymbol{\nu}}
\newcommand{\operator}{\mathcal{G}}
\newcommand{\x}{\mathbf{x}}
\newcommand{\y}{\mathbf{y}}
\newcommand{\expe}{\mathbb{E}}
\newcommand{\cov}{\textnormal{Cov}}
\newcommand{\mis}{\mathcal{M}}
\newcommand{\encoders}{\mathcal{E}}
\newcommand{\decoders}{\mathcal{D}}
\newcommand{\wip}{\begin{center}\textit{... Work in progress ...}\end{center}}

\maketitle

\section{Introduction}
\label{sec:intro}
Aside from its striking impact on areas such as data science, language processing, and computer vision, Deep Learning is now becoming ubiquitous in all branches of science, with applications ranging from economics \cite{nelson2017stock, long2019deep, horvath2021deep} to medicine \cite{tandel2019review, massi2020deep, badre2021deep}, biology \cite{angermueller2016deep, wei2018deep, sato2021rna}, chemistry \cite{ziletti2018insightful, ryan2018crystal, schutt2018schnet}, physics \cite{hashimoto2018deep, dalda2019recovering, breen2020newton}, mathematics \cite{daubechies2022nonlinear, gribonval2022approximation, bartolucci2023understanding} and engineering \cite{lahivaara2018deep, rosafalco2021online, pichi2023artificial}. 
In most cases, researchers exploit Deep Learning to make up for our limited understanding and limited computational resources, typically trying to find a suitable compromise between domain knowledge and data-driven approaches. In this contribution, we shall focus our attention on a specific research area, typical of engineering applications, which concerns the development of Reduced Order Models (ROMs).

Simply put, ROMs are model surrogates that aim at replacing expensive numerical simulations with accurate approximations that are obtained at a reduced computational cost. From a practical point of view, ROMs can be remarkably helpful whenever dealing with real-time many-query applications, such as those characterizing digital twins \cite{kapteyn2021probabilistic}, optimal control \cite{ravindran2000reduced} and uncertainty quantification \cite{heinkenschloss2018conditional, cohen2022nonlinear}. As of today, the literature is filled with plenty of different ROM approaches, and whether to choose one or another is typically problem dependent. Here, we shall focus on Deep-Learning based ROMs (DL-ROMs), specifically addressing the framework proposed in \cite{fresca2021comprehensive,franco2022deep}, which leverages on the use of deep autoencoders and has already reported numerous successful applications, see, e.g., \cite{fresca2020deep, fresca2022deep, gobat2023reduced, cicci2023reduced}.
\\\\ 
The driving idea behind the DL-ROM approach is to exploit the nonlinear approximation capabilities of deep autoencoders to perform a suitable dimensionality reduction, allowing the representation of complex high-fidelity solutions as small vectors in some latent space of dimension $n$. In general, the choice of this latent dimension is problem specific and its value reflects the intrinsic properties of the so-called \textit{solution manifold}. In the case of PDEs depending on a finite number of deterministic parameters, this fact has already been thoroughly studied in \cite{franco2022deep}. There, the authors characterize the latent dimension of the DL-ROM by investigating the behavior of the so-called \textit{manifold n-width} \cite{devore1989optimal}. Given a compact parameter space $\Theta\subseteq\mathbb{R}^{p}$, a Hilbert state space $(V,\|\cdot\|)$ and a parameter-to-solution operator $\Theta\ni\mup\mapsto u_{\mup}\in V$, the latter can be written as
\begin{equation}
\label{eq:manifoldwidth}
\inf_{\substack{\Psi'\in\encoders(V,\mathbb{R}^{n})\\\Psi\in\decoders(\mathbb{R}^{n},V)}}\;\sup_{\mup\in\Theta}\|u_{\mup}-\Psi(\Psi'(u_{\mup}))\|
\end{equation}
where $\Psi':V\to\mathbb{R}^{n}$ and $\Psi:\mathbb{R}^{n}\to V$ are the encoder and decoder networks, respectively, each varying in a suitable class of admissible architectures correspondingly named $\encoders(V,\mathbb{R}^{n})$ and $\decoders(\mathbb{R}^{n},V)$.

The purpose of this work is to extend the analysis proposed in \cite{franco2022deep} to address the case of stochastic PDEs, where the deterministic parameters are replaced by some random field $\mu\sim\mathbb{P}$. In general, this extension presents two major challenges. First of all, the realizations of the input field $\mu$ might be arbitrarily large in norm, which removes any form of compactness and makes the arguments in \cite{franco2022deep} inapplicable.
On top of that, the stochasticity introduced by the random field makes the PDE formally depend on an infinite number of parameters, i.e. $p=+\infty$. Then, the bounds provided in \cite{franco2022deep} become meaningless.

To address these difficulties, we shall first replace the optimization problem in \eqref{eq:manifoldwidth} with its probabilistic counterpart, that is
\begin{equation}
\label{eq:probabilisticwidth}
\inf_{\substack{\Psi'\in\encoders(V,\mathbb{R}^{n})\\\Psi\in\decoders(\mathbb{R}^{n},V)}}\;\expe_{\mu\sim\mathbb{P}}\|u_{\mu}-\Psi(\Psi'(u_{\mu}))\|
\end{equation}
where $\expe$ is the expectation operator. Then, \eqref{eq:probabilisticwidth} measures the extent to which an autoencoder with latent dimension $n$ can approximate the high-fidelity solutions arising from the PDE. Our purpose is to characterize the quantitative behavior of \eqref{eq:probabilisticwidth} and thus provide practical insights that can guide domain practitioners in the complex design of DL-ROM architectures.

More precisely, we shall prove the following facts. If the PDE depends on a finite number of stochastic parameters, $p\in\mathbb{N}$, then the autoencoders can achieve arbitrary accuracy in any latent dimension $n\ge p$. In general, this result is much stronger than the one obtained in the deterministic setting, as the optimal bound in that case is $n\ge2p+1$: see Theorem 3 in \cite{franco2022deep} for further details. On the contrary, if the PDE depends on a general random field (informally, $p=+\infty$), then we can characterize the decay of the reconstruction error \eqref{eq:probabilisticwidth} in terms of the eigenvalues of the covariance operator of the input random field, $\mu$, and of the output field, $u_{\mu}$, respectively. In particular, our analysis shows how deep autoencoders can outperform the performance of linear methods by benefiting from the intrinsic regularities available in both fields. \\

The paper is organized as follows. First, in Section \ref{sec:setup}, we formally introduce the problem of reduced order modeling for stochastic PDEs, with a brief overview about the DL-ROM approach, while, in Section \ref{sec:preliminaries}, we present the mathematical tools that are needed for our construction. Our main contribution can be found in Section \ref{sec:theory}, where we derive several results about the latent dimension of deep autoencoders for PDEs parameterized by random fields. Finally, we devote Section \ref{sec:experiments} to the numerical experiments.

\section{Problem setup}
\label{sec:setup}
Let $\Omega\subset\mathbb{R}^{d}$ be a bounded domain. We are given a parametrized boundary value problem with random coefficients, e.g.
\begin{equation}
\label{eq:pde}
    \begin{cases}
\mathcal{A}_{\mu}u_{\mu}=f_{\mu}&\text{in}\;\Omega\\
        \mathcal{B}_{\mu}u_{\mu}=g_{\mu} &\text{on}\;\partial\Omega,
\end{cases}\end{equation}
with $\mathcal{A}_{\mu},\mathcal{B}_{\mu},f_{\mu},g_{\mu}$ parameter dependent operators and problem data, respectively. Here, $\mu$ is a suitable random field that parametrizes the PDE and introduces a corresponding form of stochasticity. For simplicity, we assume that $\mu$ is defined over $\Omega$, although such a restriction is not necessary. Notice also that we limit our attention to steady PDEs: for a deeper discussion about this aspect, we refer the reader to the remark at the end of this Section.

We assume to have at our disposal a trusted high-fidelity numerical solver -- possibly expensive -- that can approximate the solution of \eqref{eq:pde} for each fixed realization of the random field $\mu$. We further assume that all numerical simulations produced by such solver live in a common high-fidelity state space $V_{h}\subset L^{2}(\Omega)$ of dimension $\dim(V_{h})=N_{h}$. 
Let $\{\xi_{i}\}_{i=1}^{N_{h}}$ be a basis of $V_{h}$. Then, for each realization of $\mu$, the numerical solver provides a way to compute a set of basis coefficients 
$$\mathbf{u}_{\mu}^{h}:=[c_{\mu}^{(1)},\dots,c_{\mu}^{(N_{h})}]
\quad\text{such that}\quad
u_{\mu}\approx \sum_{i=1}^{N_{h}}c_{\mu}^{(i)}\xi_{i}.$$
In this sense, the numerical solver, also known as the Full Order Model (FOM), defines a solution operator
$$\operator_{h}:\mu\mapsto \mathbf{u}_{\mu}^{h}\in\mathbb{R}^{N_{h}}$$
in a very natural way. The efficient approximation of such an operator can be of great interest in many-query and real-time applications, where multiple calls to the FOM become computationally unbearable. Thus, the main interest becomes finding a cheaper surrogate $\mu\mapsto\Phi(\mu)$, called the reduced order model (ROM), such that $\Phi(\mu)\approx \mathbf{u}_{\mu}^{h}$. 
More precisely, let $\mathbb{P}$ denote the probability law of $\mu$, and let $\|\cdot\|_{V_{h}}$ be the norm over $\mathbb{R}^{N_{h}}$ induced by the $L^{2}$-norm over the state space $V_{h}$. Then, giving some tollerance $\varepsilon>0$, one seeks to construct a suitable ROM for which
$$\expe_{\mu\sim\mathbb{P}}\|\mathbf{u}_{\mu}^{h}-\Phi(\mu)\|_{V_{h}}<\varepsilon.$$

\subsection{Reduced order modeling and the DL-ROM approach}
\label{subsec:roms}
The development and construction of accurate ROMs is an extremely active area of research: as of today, the literature features a very broad spectrum of approaches to model order reduction, such as the Reduced Basis method \cite{negri2013reduced,quarteroni2015reduced,hesthaven2016certified,taumhas2023reduced} and its non-intrusive data-driven variations, e.g. \cite{hesthaven2018non, guo2019data}, adaptive ROMs \cite{amsallem2011online, carlberg2015adaptive, pagliantini2021dynamical, kazashi2023dynamically}, hybrid techniques based on closure modeling \cite{ivagnes2023hybrid, wang2020recurrent}, deep learning-based ROMs \cite{franco2022deep, fresca2021comprehensive, fresca2022pod, brivio2023error, pichi2023graph}, and many others, each with their own benefits and guarantees.

Here, we shall focus primarily on the case of Deep Learning based ROMs (DL-ROMs), following the framework introduced in \cite{fresca2021comprehensive,franco2022deep}. These can be of particular interest whenever: i) intrusive approaches are not available, either because they would entail expensive subroutines, such as hyperreduction, or because the FOM is not accessible; ii) linear methods based on, e.g., Principal Orthogonal Decomposition (POD), fail because of intrinsic complexities entailed by the solution operator (see, e.g., the well-known phenomenon of \textit{slow-decay in the Kolmogorov $n$-width} \cite{franco2022deep, ohlberger2015reduced, romor2023non}).

For simplicity, let us assume that the random field at input, $\mu$, has been discretized at the same level as PDE solutions, so that each realization $\mu\in L^{2}(\Omega)$ is formally replaced by some $\mup^{h}\in\mathbb{R}^{N_{h}}.$ In their original formulation, DL-ROMs are characterized by three deep neural network architectures, $\Psi',\Psi,\phi$, which operate as
\begin{equation*}
\Psi':\mathbb{R}^{N_h}\to\mathbb{R}^n,\quad\quad\Psi:\mathbb{R}^{n}\to\mathbb{R}^{N_{h}},\quad\quad\phi:\mathbb{R}^{N_h}\to\mathbb{R}^n, 
\end{equation*}
where $n\in\mathbb{N}$ is the so-called latent dimension of the model. The idea is that the first two networks are responsible for learning the fundamental characteristics that characterize PDE solutions; conversely, the third map, $\phi$, is left to learn the way in which the input $\mup^{h}$ affects the output $\mathbf{u}^{h}_{\mup^{h}}$. This is achieved by constructing the three networks so that
\begin{equation}
    \label{eq:dlromapprox1}
\Psi(\Psi'(\mathbf{u}^{h}_{\mup^{h}}))\approx\mathbf{u}^{h}_{\mup^{h}}\quad\text{and}\quad\phi(\mup^{h})\approx\Psi'(\mathbf{u}^{h}_{\mup^{h}}).
\end{equation}
Then, the parameter-to-solution operator is approximated as
\begin{equation}
    \label{eq:dlromapprox2}
\Phi(\mup^{h}):=\Psi(\phi(\mup^{h}))\approx \mathbf{u}_{\mup^{h}}^{h}.
\end{equation}
In this sense, the accuracy of the DL-ROM is ultimately determined by that of the combined network $\Phi:=\Psi\circ\phi$, while the role of $\Psi'$ is only auxiliary. However, such a splitting can be beneficial (see, e.g., Figure 22 in \cite{fresca2021comprehensive}) as it makes the two blocks, $\Psi$ and $\phi$, tackle the different complexities that characterize model order reduction: on the one hand, the spatial complexity of the solutions; on the other hand, the intricate dependency of the solutions with respect to the input parameters. In this sense, the presence of $\Psi'$ is fundamental, as it allows us to decouple the problem.

Following the conventions of the Deep Learning literature, we refer to the composition $\Psi'\circ\Psi$ as an \textit{autoencoder} architecture of latent dimension $n$, while the two maps, $\Psi'$ and $\Psi$, are referred to as encoder and decoder, respectively. The terminology comes from the fact that the two maps ultimately provide a way of representing high-fidelity solutions as small vectors in $\mathbb{R}^{n}$, in fact, in $\Psi(\Psi'(\mathbf{u}^{h}_{\mup^{h}}))\approx\mathbf{u}^{h}_{\mup^{h}}$ and $\Psi'(\mathbf{u}^{h}_{\mup^{h}})\in\mathbb{R}^{n}$. Then it is clear that the choice of the appropriate value of $n$ becomes fundamentally important. The purpose of this work is to provide additional information on this aspect, with a particular focus on the case of PDEs parametrized by random fields. In the next subsection, we shall further motivate this fact with a practical example. \\

For the sake of completeness, before coming to the main objective of this work, let us conclude this overview with a few words about the training and implementation of DL-ROMs. 
The first step consists of exploiting the FOM to generate a collection of trusted samples 
$$\{\mup_{i}^{h},\mathbf{u}_{i}^{h}\}_{i=1}^{N_{\text{train}}},$$
where $\mathbf{u}_{i}^{h}=\operator_{h}(\mup^{h})$, which serve as training data for the three networks in the DL-ROM pipeline. More precisely, once the architectures have been designed, the three modules are trained by minimizing a suitable loss function such as the one below,
\begin{align}
    \label{eq:loss}
    \mathcal{L}(\Psi,\Psi',\phi)= \frac{1}{N_{\text{train}}}\sum_{i=1}^{N_{\text{train}}}
    \Big(&\alpha_{1}\|\mathbf{u}_{i}^{h}-\Psi(\phi(\mup_{i}^{h}))\|_{V_{h}}^{2}
    + \\\nonumber&\alpha_{2}\|\mathbf{u}_{i}^{h}-\Psi(\Psi'(\mathbf{u}^{h}_{i}))\|_{V_{h}}^{2}
    + \\\nonumber&\alpha_{3}\|\Psi'(\mathbf{u}_{i}^{h})-\phi(\mup_{i}^{h})\|_{\mathbb{R}^{n}}^{2}\Big).
\end{align}
Here, $\alpha_{1},\alpha_{2},\alpha_{3}\ge0$ are suitable weights that are used to define the loss function. The idea is that by minimizing \eqref{eq:loss}, one would automatically enforce both \eqref{eq:dlromapprox1} and \eqref{eq:dlromapprox2}. It should be noted that since the term multiplied by $\alpha_{1}$ contains information about the actual accuracy of the ROM, the other two can be seen as regularizers. In general, training of the three networks can be achieved simultaneously or in multiple steps. For example, in \cite{fresca2021comprehensive}, the authors set $\alpha_{2}=0$ and proceed with a single training; conversely, in \cite{franco2022deep}, the authors propose a two-stage training phase: first with $\alpha_{1}=\alpha_{3}=0$ and then with $\alpha_{1}=\alpha_{2}=0$. Here, we shall keep all the weights active $\alpha_{1},\alpha_{2},\alpha_{3}>0$, and opt for a one-shot training routine.

After the training phase, the DL-ROM can efficiently approximate high-fidelity solutions in milliseconds. The quality of such an approximation is typically assessed by relying on a suitable \textit{test set}, that is, by comparing the outputs of the two models, the DL-ROM and the FOM, for new independent realizations of the input field.

\subsection{Choosing the latent dimension: a motivating example}
To further motivate our analysis and anticipate the practical impact of our results, let us look at a simple problem featuring a PDE with finitely many random parameters. As in \cite{franco2022deep}, we consider the following boundary value problem defined over the unit square $\Omega=(0,1)^{2}$
\begin{equation}
\label{eq:cookie}
-\nabla\cdot\left(\sigma_{\mup}\nabla u\right)=f_{\mup}  \quad\text{in}\;\Omega\end{equation}
complemented with a constant Dirichlet boundary condition, $u\equiv0.1$ on $\partial\Omega$. Here, the PDE depends on three random parameters, $\mup=[\mu_{1},\mu_{2},\mu_{3}]$, which affect the permeability field $\sigma=\sigma_{\mup}$ and the right-hand-side $f=f_{\mup}$. Those are defined as
$$\sigma_{\mup}(\x):=\frac{1}{2}+\mu_{1}\mathbbm{1}_{\Omega_{0}}(\x),\quad\quad f_{\mup}(\x):=\frac{1}{2\pi\epsilon^{2}}\exp\left(-\frac{1}{2\epsilon^{2}}|\x-[\mu_{2},\mu_{3}]^{T}|^{2}\right),$$
where $\epsilon:=0.01$ and $\Omega_{0}\subset\Omega$ is as in Figure \ref{fig:cookie}a. 
\begin{figure}
    \centering
    \includegraphics[width=\textwidth]{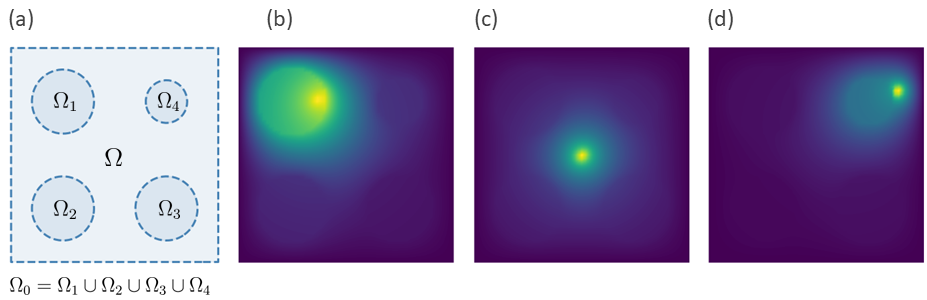}
    \caption{Domain of definition (a) and randomly generated PDE solutions (b-d) for Problem \eqref{eq:cookie}.}
    \label{fig:cookie}
\end{figure}
The first parameter, $\mu_{1}$, can attain any random value in $[1,4]$, and has the effect of modifying the permeability of the subdomains; conversely, $\mu_{2}$ and $\mu_{3}$ are responsible for the random location of the concentrated source $f_{\mup}$, and they are allowed to range from $0.1$ to $0.9$. More precisely, we endow the overall parameter space with a uniform probability distribution supported over $[1,4]\times[0.1,0.9]^{2}\subset \mathbb{R}^{3}$.

To construct our trusted high-fidelity solver, we rely on a Finite Element discretization of \eqref{eq:cookie} using continuous piecewise linear elements defined over a structured triangular grid of stepsize $h\approx0.0236$, which results in a FOM dimension of $N_{h}=3721$. We run the FOM multiple times to generate a total of 2000 random solutions (1500 for training and 500 for testing): we refer to Figure \ref{fig:cookie}b-e for a few examples. Our purpose is to exploit these data in order to explore the behavior of two different dimensionality reduction techniques. POD, which here stands as a representative of linear ROMs, and deep autoencoders, which, instead, are at the core of DL-ROMs.

Precisely, we compute the reconstruction error (estimated via classical Monte Carlo over the test set)
$$\mathbb{E}_{\mup\sim\mathbb{P}}\|\mathbf{u}_{\mup}^{h}-\Psi(\Psi'(\mathbf{u}_{\mup}^{h}))\|_{V_{h}},$$
obtained by employing POD (i.e., by letting $\Psi'\equiv\mathbf{V}^{T}$ and $\Psi\equiv\mathbf{V}$ be nothing but the POD projectors) or deep autoencoders (in which case both $\Psi$ and $\Psi'$ are DNNs). We repeat this computation for different choices of the latent dimension, namely $n=1,2,\dots,40$, to investigate the impact of the ROM dimension on the quality of the reconstruction. We mention that, while the POD matrix is uniquely determined by $n$, this is not the case for deep autoencoders, as users can play with several components of the architecture (number of layers, intermediate neurons, layer type, etc.). To highlight this fact, we train three different autoencoder architectures for each value of $n$: for technical details on their design, we refer the reader to the appendix, Section \ref{sec:architectures}. The results are shown in Figure \ref{fig:motivating}. 

\begin{figure}
    \centering
    \includegraphics[width=0.8\textwidth]{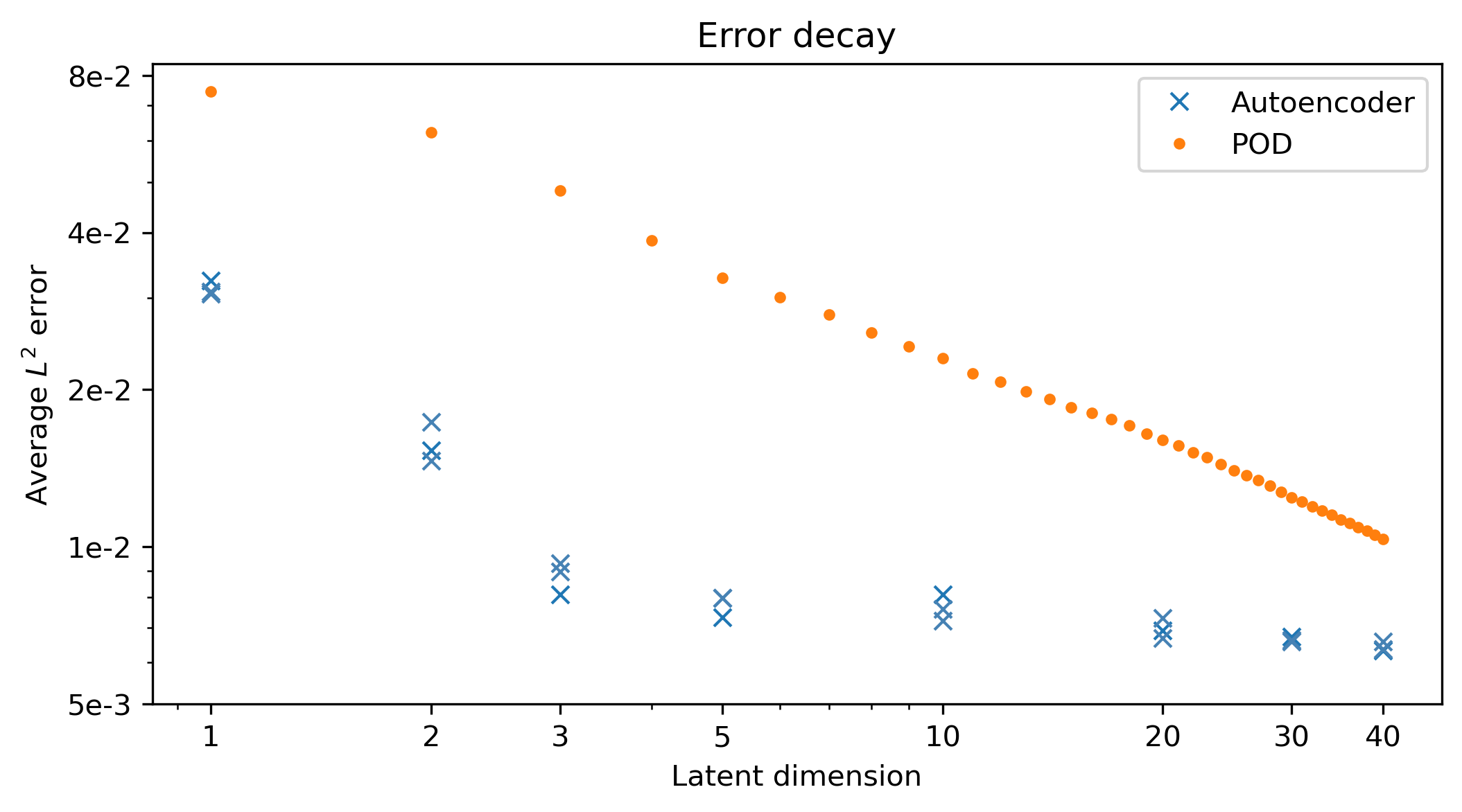}
    \caption{Error decay for Problem \eqref{eq:cookie} as a function of the ROM dimension: linear reduction through POD vs nonlinear reduction through autoencoders.}
    \label{fig:motivating}
\end{figure}

In general, autoencoders quickly outperform POD: in fact, their performances at $n=3$ remain unmatched even when the POD considers as much as $n=40$ latent variables. However, here we are not really interested in the actual values of the error, but rather in \textit{how} the error decays.

In the linear case, we can observe a very clear and stable trend, where the error decays as some power of the latent dimension $n$, namely $n^{-\gamma}$ for some $\gamma>0$. However, things change quite a bit when moving to deep autoencoders. Here, in fact, although the error always decreases as a function of $n$, the rate of such decay is not constant. Indeed, we can spot at least two different trends: first, we have a rapid decay from $n=1$ to $n=3$, which is then followed by a much slower one from $n=3$ on. In this sense, the latent dimension $n=3$ appears to be a turning point: after that, it is not really worth increasing the ROM dimension; rather, the performances could be improved by considering more complex architectures (see, e.g., the discussion in \cite{franco2022deep}, especially Section 4.1 and Figure 7).

Since for the case at hand the PDE depends on $p=3$ random parameters, it is natural to ask whether this is a simple coincidence. As we shall prove in Theorem \ref{theorem:finite}, this is not the case. It is shown here that when $n=p$, there is an ideal balance between reduction and precision. However, we must mention that several authors had already conjectured this fact and applied it as a rule of thumb; see, e.g., \cite{lee2020model,fresca2021comprehensive}; furthermore, in the case of finitely supported probability measures, some first insights were already provided in \cite{franco2022deep}.

Here, we extend the analysis proposed in \cite{franco2022deep}, showing that this optimality is preserved even if the parameter space is endowed with a probability measure with unbounded support. Furthermore, we shall generalize the idea to the case of random fields, where the PDE formally depends on the parameters $p=+\infty$. From a practical point of view, this will allow domain practitioners to choose the latent dimension of deep autoencoders beforehand, thus avoiding the tedious procedures based on trial and error. For instance, when facing a problem such as \eqref{eq:cookie}, one can safely let $n=3$ without having to repeat the analysis in Figure \ref{fig:motivating}.

\section{Preliminaries}
\label{sec:preliminaries}
In the current Section, the reader can find all the key properties and the mathematical concepts that are required for properly following the theoretical analysis proposed in Section \ref{sec:theory}. Precisely: in Section \ref{subsec:dnns}, we shall introduce the main ingredient of our recipe, i.e., DNN architectures; then, in Section \ref{subsec:local}, we familiarize ourselves with the notion of \textit{local variation}, a mathematical concept that we later use to characterize locally Lipschitz operators; finally, in Section \ref{subsec:gaussian}, we summarize several properties about Gaussian processes that are relevant to our analysis in Section
\ref{sec:theory}. 

In order to keep the paper self-contained, most of the proofs related to this Section have been postponed to the Appendix, Section \ref{sec:appendix:A}.

\subsection{Expressivity of Deep Neural Networks}
\label{subsec:dnns}
DNNs are computational units based on the composition of affine transformations and nonlinear activations, with the latter being applied componentwise on all vector entries. In the context of Deep Learning, in fact, it is very common to encounter the following notation
    \begin{equation}
        \label{eq:componentwise}
        \rho(\x):=[\rho(x_{1}),\dots,\rho(x_{n})],
    \end{equation}
where $\x\in\mathbb{R}^{n}$ and $\rho:\mathbb{R}\to\mathbb{R}$. Here, with little abuse of notation, we agree to adopt the same convention.

In particular, DNNs are obtained via the composition of several maps, called \textit{layers}, of the form $\x\to\rho(\mathbf{W}\x+\mathbf{b})$. Each layer is characterized by its own weight matrix $\mathbf{W}$, and its bias vector $\mathbf{b}$, two learnable parameters that are usually optimized during the training phase. All layers except the terminal layer (which usually comes without an activation function) are called \textit{ hidden layers}, and their total number defines the \textit{depth} of the architecture. In mathematical terms, we may synthesize these notions as follows.

\begin{definition}
    Let $\rho:\mathbb{R}\to\mathbb{R}$ and $m,n\in\mathbb{N}$. We define the family of layers from $\mathbb{R}^{m}$ to $\mathbb{R}^{n}$ with activation function $\rho$ as
    \begin{multline}
        \mathscr{L}_{\rho}(\mathbb{R}^{m},\mathbb{R}^{n})=\{f:\mathbb{R}^{m}\to\mathbb{R}^{n}\text{ s.t.}\;\;f(\x)=\rho(\mathbf{W}\x+\mathbf{b}),\;\mathbf{W}\in\mathbb{R}^{m\times n},\;\mathbf{b}\in\mathbb{R}^{n}\}.
    \end{multline}
    Similarly, we define the family of Deep Neural Networks (DNNs) from $\mathbb{R}^{m}\to\mathbb{R}^{n}$ with activation function $\rho$ as
    \begin{equation}
        \begin{aligned}
        \mathscr{N}_{\rho}(\mathbb{R}^{m},\mathbb{R}^{n})=\{f_{l+1}\circ\dots\circ f_{1},\text{ with }\;&f_{l+1}\in\mathscr{L}_{\mathbf{id}}(\mathbb{R}^{n_{l}},\mathbb{R}^{n}),\;l\ge1,\\&f_{i}\in\mathscr{L}_{\rho}(\mathbb{R}^{n_{i-1}}\mathbb{R}^{n_{i}}),\;i=1,\dots,l,\\&n_{i}\in\mathbb{N},\;\text{and}\;n_{0}=m\},
    \end{aligned}
    \end{equation}
    where $\mathbf{id}:\mathbb{R}\to\mathbb{R}$ is the identity map, $\mathbf{id}(x)=x.$
    \end{definition}

When embedded in classical functional spaces, DNNs can provide remarkable expressivity. In fact, with very few hypothesis on their activation function, DNNs become able to approximate continuous maps over compact sets, as well as integrable maps over finite measure spaces. These kinds of results are known as Universal Approximation Theorems. Here, since we are dealing with stochastic quantities and probability measures, we are interested in the Universal Approximation Theorem proved by Hornik in 1991 \cite{hornik1991approximation}. In particular, we report below a slightly different result, which is a direct consequence of Hornik's Theorem.

\begin{lemma}
    \label{lemma:density}
    Let $\mathbb{P}$ be a probability distribution on $\mathbb{R}^{p}$. Let $\rho:\mathbb{R}\to\mathbb{R}$ be a continuous map. Assume that either one of the following holds:
    \begin{itemize}
        \item [i)] $\rho$ is bounded and nonconstant;
        \item [ii)] $\rho$ is bounded from below and  $\rho(x)\to+\infty$ as $x\to+\infty$;
        \item [iii)] there exists some $a,b\in\mathbb{R}$ such that $x\to a\rho(x)+b\rho(-x)$ satisfies (ii).
    \end{itemize}
    Then, for every $\varepsilon>0$ and every measurable map $f:\mathbb{R}^{p}\to\mathbb{R}^{n}$
    with $\mathbb{E}_{\x\sim\mathbb{P}}|f(\x)|<+\infty$, there exists $\Phi\in\mathscr{N}_{\rho}(\mathbb{R}^{p},\mathbb{R}^{n})$ such that
    \begin{equation*}
        \mathbb{E}_{\x\sim\mathbb{P}}|f(\x)-\Phi(\x)|<\varepsilon.
    \end{equation*}
\end{lemma}

\begin{proof} See Appendix \ref{sec:appendix:A}.   
\end{proof}

\begin{remark}
    The result in Lemma \ref{lemma:density} applies to most activation functions used in deep learning applications. For instance, the statement holds for maps such as the sigmoid or the tanh activation, as they satisfy (i), but it also holds for other popular maps, such as the ReLU, the SELU and the swish activation, which are lower-bounded. Finally, the Lemma can also be applied to maps such as the $\alpha$-leaky ReLU activation,
    \begin{equation}
        \label{eq:leakyReLU}
        \rho(x)=\begin{cases}
        x & x\ge0 \\
        \alpha x & x<0
    \end{cases},\end{equation}
    where $|\alpha| < 1.$ In fact, the latter satisfies (iii) with $a=1$ and $b=-\alpha$. 
\end{remark}

\subsection{Local variation of (nonlinear) operators}
\label{subsec:local}
As we mentioned in the Introduction, the manifold $n$-width, Eq. \eqref{eq:manifoldwidth}, provides a suitable way to analyze the reconstruction error of autoencoders in a (deterministic) parametric setting, provided that the parameter space is bounded. In that case, previous work has shown that the continuity of the operator is not enough to guarantee a reasonable behavior of the manifold $n$-width, while the stronger assumption of Lipschitz continuity suffices (see Theorem 3 and the corresponding remark in \cite{franco2022deep}). 

However, when considering a stochastic scenario, the realizations of the input field may be arbitrarily large in norm. For example, one may have $\text{supp}(\mu)=L^{q}(\Omega)$ when $\mu$ is a random field, or $\text{supp}(\mup)=\mathbb{R}^{p}$ when $\mup$ is a random vector. Then, requiring the operator $\operator$ to be (globally) Lipschitz continuous becomes too restrictive if one wants to derive any meaningful result on the matter. For instance, even in the finite rank case, an extremely regular operator such as $\operator(\mup)=|\mup|^{2}$ would not fit the assumptions. For this reason, it becomes convenient to consider alternative properties such as \textit{local} Lipschitz continuity, where the Lipschitz constant of $\mathcal{G}$ is allowed to change from point to point (clearly, the rate of this change should also be taken into account). All these considerations bring us to the definition below.

\begin{definition} 
\textnormal{\textbf{(Local variation)}}
\label{def:locvar}
    Let $\operator:(W,\|\cdot\|_{W})\to(V,\|\cdot\|_{V})$ be an operator between normed spaces. We define its local variation as the map $\|\partial\operator\|_{V}:W\to[0,+\infty]$ given by
    \begin{multline}
        \|\partial\operator\|_{V}(w):=\limsup_{h\to0}\;\frac{\|\operator(w+h) - \operator(w) \|_{V}}{\|h\|_{W}}=\\
        =\lim_{r\to0^{+}}\;\sup_{0< \|h\|_{W} \le r} \frac{\|\operator(w+h) - \operator(w) \|_{V}}{\|h\|_{W}}.
    \end{multline}
\end{definition}

Ultimately, the \textit{local variation} provides a way to bound the Lipschitz constant 
near every point of the input space, and it can be computed for any (nonlinear) operator. In fact, our definition can be traced back to the notion of \textit{absolute condition number}, which is a concept commonly encountered in error and numerical analysis; see, e.g., \cite{quarteroni2010numerical}. Here, however, we insist on using a different notation to better reflect our context and avoid any sort of ambiguity.


The concept of local variation can be further understood by considering its relationship to other well-established mathematical concepts, such as Lipschitz continuity and Frechét differentiability. These facts are briefly summarized by the Proposition below.

\begin{proposition}
    \label{prop:lip}
    Let $\operator:(W,\|\cdot\|_{W})\to(V,\|\cdot\|_{V})$ be an operator between two normed spaces. Then, the following properties hold: 
    \begin{itemize}
        \item [i)] $\|\partial\operator\|_{V}(w)<+\infty$ for all $w\in K\subset W$ $\iff$ $\operator$ is locally Lipschitz over $K$;\\
        
        \item [ii)] if $K\subset W$ is compact and $\|\partial\operator\|_{V}(w)<+\infty$ for all $w\in K$, then 
        $\mathcal{G}$ is Lipschitz over $K$;\\

        \item [iii)] if $C\subseteq W$ is convex, then 
        $$L_{C}:=\sup_{w\in C}\|\partial\operator\|_{V}(w)<+\infty$$
        if and only if $\operator$ is $L_{C}$-Lipschitz over $C$.\\

        \item [iv)] if $\operator$ is Fréchet differentiable at $w\in W$, then $\|\partial\operator\|_{V}(w)$ coincides with the operator norm of the Fréchet derivative of $\operator$ at $w.$\\

        \item [v)] given any $\mathcal{F}:(V,\|\cdot\|)_{V}\to(Y,\|\cdot\|_{Y})$, one has the chain-rule inequality
        $$\|\partial(\mathcal{F}\circ \operator)\|_{Y}(w)\le\|\partial \mathcal{F}\|_{Y}(\operator(w))\cdot\|\partial\operator\|_{V}(w),$$
        for all $w\in W.$
    \end{itemize}
\end{proposition}
\begin{proof}
See Appendix \ref{sec:appendix:A}.
\end{proof}

As a straightforward consequence of Proposition \ref{prop:lip}, we also have the following Corollary, which can be thought of as a form of Taylor-Lagrange inequality.

\begin{corollary}
    \label{cor:lipgen}
    Let $\operator:(W,\|\cdot\|_{W})\to(V,\|\cdot\|_{V})$ be any operator between two normed spaces. Then, for all $w,w'\in W$ we have
    \begin{equation}
        \label{eq:lipgen}
        \|\operator(w)-\operator(w')\|_{V}\le
        \left(\sup_{0\le t\le 1}\|\partial\operator\|_{V}(tw+(1-t)w')\right)\|w-w'\|_{W}.
    \end{equation}
\end{corollary}

\begin{proof}
Given $w,w'\in W$, let $K$ be the segment between the two points and define $L:=\sup_{v\in K}\|\partial\operator\|_{V}.$ If $L=+\infty$, then \eqref{eq:lipgen} is obvious; conversely, if $L<+\infty$, then the conclusion follows by (iii) in Proposition \ref{prop:lip}.
\end{proof}

Equation \eqref{eq:lipgen} provides a way to control the discrepancy between two different outputs of the operator, but it may also be applied to derive growth conditions. For example, it can be shown that if the local variation grows at most exponentially, then so does the operator. We formalize this consideration below.

\begin{corollary}
    \label{cor:exp}
    Let $\operator:(W,\|\cdot\|_{W})\to(V,\|\cdot\|_{V})$ be an operator between two normed spaces. The following holds true,
    \begin{multline*}
        \exists M,\beta>0\textnormal{\;\;s.t.\;\;}
        \|\partial\operator\|_{V}(w)\le Me^{\beta\|w\|_{W}}\;\;\forall w\in W,\\
        \implies
        \exists \tilde{M},\gamma>0\textnormal{\;\;s.t.\;\;}
        \|\operator(w)\|_{V}\le \tilde{M}e^{\gamma\|w\|_{W}}\;\;\forall w\in W.
    \end{multline*}
\end{corollary}
\begin{proof}
    Let $c:=\operator(0)$ and fix any $w\in W$. By Corollary \ref{cor:lipgen}, we have
    \begin{multline*}
        \|\operator(w)\|_{V}\le c + \|\operator(w)-\operator(0)\|_{V}
        \le c +\left(\sup_{0\le t\le 1}\|\partial\operator\|_{V}(tw)\right)\|w\|_{W}\le\\
        \le c +\left(\sup_{0\le t\le 1}Me^{\beta\|tw\|_{W}}\right)\|w\|_{W} = c+Me^{\beta\|w\|_{W}}\|w\|_{W}.
    \end{multline*}
Since $a < e^{a}$ for all $a\in\mathbb{R}$, and $c\le ce^{b}$ for all $b\ge0$, we have
    \begin{equation*}
        \|\operator(w)\|_{V}\le\dots\le (c+M)e^{(\beta+1)\|w\|_{W}}. \qedhere
    \end{equation*}
\end{proof}

\noindent To conclude, we present a practical example of an operator whose local variation grows at most exponentially. Despite its simplicity, we believe this example to be of high interest, as it describes the case of a stochastic Darcy flow, which, in particular, has direct implications in the study of porous media.

\begin{proposition}
    \label{prop:darcy}
    Let $\Omega\subset\mathbb{R}^{d}$ be a bounded domain with Lipschitz boundary, and let $f\in H^{-1}(\Omega)$ be given. For any $\sigma\in L^{\infty}(\Omega)$, let $u=u_{\sigma}$ be the solution to the following boundary value problem,
    \begin{equation}
    \begin{cases}
    -\nabla\cdot(e^{\sigma}\nabla u)=f &  \textnormal{in}\;\Omega\\
            u=0 & \textnormal{on}\;\partial\Omega.
    \end{cases}\end{equation}
    Let $\operator: L^{\infty}(\Omega)\to L^{2}(\Omega)$ be the operator that maps $\sigma\mapsto u$. Then, for all $\sigma,\sigma'\in L^{\infty}(\Omega)$
  \begin{equation}
    \label{eq:darcybound}
      \|\operator(\sigma)-\operator(\sigma')\|_{L^{2}(\Omega)}\le C\|f\|_{H^{-1}(\Omega)} e^{3\|\sigma\|_{L^{\infty}(\Omega)}+3\|\sigma'\|_{L^{\infty}(\Omega)}}\|\sigma-\sigma'\|_{L^{\infty}(\Omega)}
  \end{equation}
    and, in particular, \begin{equation*}
    \|\partial\operator(\sigma)\|_{L^{2}(\Omega)}\le C\|f\|_{H^{-1}(\Omega)}e^{6\|\sigma\|_{L^{\infty}(\Omega)}},\end{equation*}
    where $C=C(\Omega)$ is some positive constant.
\end{proposition}
\begin{proof}
See Appendix \ref{sec:appendix:A}.
\end{proof}

\subsection{Regularity and energy estimates for Gaussian processes}
\label{subsec:gaussian}

We now take the opportunity to recall some fundamental facts about stochastic processes and random fields. More specifically, here we shall limit our analysis to Gaussian processes: for further comments about this choice, we refer the reader to the discussion at the end of Section \ref{subsec:infinite}, see Remark \ref{remark:gaussian}. 
We start with a result that links the regularity of the covariance kernel of a Gaussian process with that of the trajectories of the random field. In doing so, we also include some estimates on the norms of the process, which we shall exploit later on. 

For the sake of simplicity, here, expected values will be directly denoted as $\mathbb{E} [ \, \cdot \, ]$, without any explicit declaration of the integration variable and its underlying probability distribution -- given the context, there should be no ambiguity.

\begin{lemma}
    \label{lemma:bound}
    Let $\Omega\subset\mathbb{R}^{d}$ be pre-compact, and let $Z$ be a mean zero Gaussian random field defined over $\Omega$. Assume that, for some $0<\alpha\le 1$, the covariance kernel of the process, $\cov:\Omega\times\Omega\to\mathbb{R},$ defined as
    \begin{equation*}
        \cov(\x,\y):=\mathbb{E}\left[Z(\x)Z(\y)\right],
    \end{equation*}
    is $\alpha$-Hölder continuous, with Hölder constant $L>0$. Then, $Z$ is sample-continuous, that is $\mathbb{P}(Z\in\mathcal{C}(\Omega))=1.$ Furthermore, for
    $\sigma^{2}:=\max_{\x\in\Omega}\cov(\x,\x),$
    one has
    \begin{equation}
        \label{eq:supint}
                \mathbb{E}^{1/2}\|Z\|_{L^{\infty}(\Omega)}^{2} \le c_{1}\sigma\left(1+\sqrt{\log^{+}(1/\sigma)}\right)\quad\textnormal{and}\quad
        \mathbb{E}\left[e^{\beta\|Z\|_{L^{\infty}(\Omega)}}\right]=c_{2}<+\infty,
    \end{equation}
    for all $\beta>0$, where $c_{1}=c_{1}(d,L,\alpha,\Omega)$ and $c_{2}=c_{2}(d,L,\alpha,\sigma,\beta,\Omega)$ are constants that depend continuously on their parameters (domain excluded). Here, $$\log^{+}(a):=\max\{\log a,0\}.$$
\end{lemma}
\begin{proof} See Appendix \ref{sec:appendix:A}.
\end{proof}

Another fundamental result that we need for our construction is a Corollary of Mercer's Theorem, known as the KKL series expansion. We report it below, together with some considerations about the covariance kernel and its truncations.

\begin{lemma}
    \label{lemma:mercer}
    Let $\Omega\subset\mathbb{R}^{d}$ be a compact subset and let $Z$ be a mean zero Gaussian random field defined over $\Omega$. Assume that the covariance kernel of $Z$, $\cov$, is Lipschitz continuous. Then, there exists a nonincreasing summable sequence $\lambda_{1}\ge\lambda_{2}\ge\dots\ge0$ and a sequence of Lipschitz continuous maps, $\{\varphi_{i}\}_{i=1}^{+\infty}$, forming an orthonormal basis of $L^{2}(\Omega)$, such that
    \begin{equation}
    \label{eq:mercer}
    \cov(\x,\y)=\sum_{i=1}^{+\infty}\lambda_{i}\varphi_{i}(\x)\varphi_{i}(\y)
    \end{equation}
    for all $\x,\y\in\Omega$. Furthermore, there exists a sequence of independent standard normal random variables, $\{\eta_{i}\}_{i=1}^{+\infty}$, such that
    \begin{equation} 
    \label{eq:karhunen} Z=\sum_{i=1}^{+\infty}\sqrt{\lambda_{i}}\eta_{i}\varphi_{i}\end{equation}
    almost surely. Finally, the truncated kernels,
    $$\cov_{p,q}(\x,\y):=\sum_{i=p}^{q}\lambda_{i}\varphi_{i}(\x)\varphi_{i}(\y),$$
    defined for varying $1\le p\le q\le+\infty$, 
    \begin{itemize}
        \item [i)] converge uniformly as $p,q\to+\infty$;
        \item [ii)] are all 1/2-Hölder continuous, with a common Hölder constant.
    \end{itemize}
\end{lemma}

\begin{proof}
See Appendix \ref{sec:appendix:A}.
\end{proof}

Both the result in Lemma \ref{lemma:bound} and that in Lemma \ref{lemma:mercer} require some form of uniform continuity of the covariance kernel, however they have the advantage of yielding useful estimates to treat the $L^{\infty}$-case. At the same time, these properties may be far too restrictive if one moves to the simpler scenario in which the trajectories of the random field are only assumed to be square-integrable. In light of this, we report below a different result specifically tailored for the $L^{2}$-case, which can be seen as an adaptation of the previous Lemmas.

\begin{lemma}
    \label{lemma:mercerl2v}
    Let $\Omega\subset\mathbb{R}^{d}$ be a compact subset and let $Z$ be a mean zero Gaussian random field defined over $\Omega$. Assume that the covariance kernel of $Z$, $\cov$, is square-integrable over $\Omega\times\Omega$. Then, there exists a nonincreasing summable sequence $\lambda_{1}\ge\lambda_{2}\ge\dots\ge0$ and an orthonormal basis of $L^{2}(\Omega)$, $\{\varphi_{i}\}_{i=1}^{+\infty}$,  such that
    \begin{equation*}
    \cov(\x,\y)=\sum_{i=1}^{+\infty}\lambda_{i}\varphi_{i}(\x)\varphi_{i}(\y)
    \end{equation*}
    for almost every $(\x,\y)\in\Omega\times\Omega$. Furthermore, there exists a sequence of independent standard normal random variables, $\{\eta_{i}\}_{i=1}^{+\infty}$, such that
    \begin{equation*}  Z=\sum_{i=1}^{+\infty}\sqrt{\lambda_{i}}\eta_{i}\varphi_{i}\end{equation*}
    almost surely. Finally, the $L^{2}$-norm of the process is exponentially integrable, i.e.
    \begin{equation}
        \label{eq:expint}
        \mathbb{E}\left[e^{\beta\|Z\|_{L^{2}(\Omega)}}\right]<+\infty\quad\quad\forall\beta>0.
    \end{equation}
\end{lemma}

\begin{proof}
See Appendix \ref{sec:appendix:A}.
\end{proof}

The main difference between Lemma \ref{lemma:mercer} and Lemma \ref{lemma:mercerl2v} lies in the regularity that one requires to the covariance kernel. Clearly, stronger assumptions about the latter result in stronger estimates about the random field and its norms. 

In conclusion, we also report an abstract version of the KKL expansion for generic Hilbert-valued random variables. In this case, it is convenient to considered an \textit{uncentered} KKL expansion, as the latter retains useful optimality properties: in fact, it is the abstract equivalent of the POD algorithm.

\begin{lemma}
    \label{lemma:abstractKKL}
    Let $(V,\|\cdot\|)$ be a separable Hilbert space and let $u$ be a squared integrable $V$-valued random variable, $\mathbb{E}\|u\|^{2}<+\infty$. Then, there exists an orthonormal basis $\{v_{i}\}_{i=1}^{+\infty}\subset V$, a sequence of (scalar) random variables $\{\omega_{i}\}_{i=1}^{+\infty}$, with $\mathbb{E}[\omega_{i}\omega_{j}]=\delta_{i,j}$, and
    a nonincreasing summable sequence $\lambda_{1}\ge\lambda_{2}\ge\dots\ge0$
    such that    $$u=\sum_{i=1}^{+\infty}\sqrt{\lambda_{i}}\omega_{i}v_{i},$$
    almost-surely.
\end{lemma}
\begin{proof}
    Up to adaptations, this is a standard result; see, e.g. Theorem 3.14 in \cite{lanthaler2022error}. Nevertheless, the interested reader can find a detailed proof in Appendix \ref{sec:appendix:A}.
\end{proof}

\section{Autoencoder-based nonlinear reduction for PDEs parametrized by random fields}
\label{sec:theory}
We are now ready to put things into action and present the main results of this work. We shall start with a preliminary consideration about the expressivity of deep autoencoders and introduce a suitable notion of \textit{admissibility} that will help us in avoiding unrealistic/pathological situations. Then, we shall derive error bounds for the reconstruction error of deep autoencoders in the case of: i) PDEs depending on a finite number of random coefficients (Section \ref{subsec:finite}); ii) PDEs parameterized by Gaussian random fields (Section \ref{subsec:infinite}).

For better readability, in what follows, we shall drop the dependency of expected values with respect to their underlying probability distribution. In particular, since all randomness will be encoded in the input variable, which is either $\mup\sim\mathbb{P}$ or $\mu\sim\mathbb{P}$, we shall simply write $\expe$ in place of $\expe_{\mup\sim\mathbb{P}}$.

\subsection{Admissible autoencoders and density results}
\label{sec:admissible}

As already mentioned, our main interest is to investigate how the choice of the latent dimension affects the optimization of the reconstruction error, and thus to provide guidelines for the design of deep autoencoders. To this end, we must note that DNN spaces lack many of the properties usually holding for classical functional spaces; furthermore, their topology can easily become rather involved, see, e.g., the discussion in \cite{petersen2019structure}. For this reason, it can be convenient to recast the optimization problem over a broader class of functions, e.g. by relying on suitable density results. For instance,  in \cite{franco2022deep}, the authors consider a more general framework where the encoder and decoder are allowed to be any pair continuous maps. Here, we relax these hypotheses even further. In fact, as a direct consequence of Lemma \ref{lemma:density}, it is easy to see that the only property that is actually needed is measurability. More precisely, we have the following result.

\begin{theorem}
    \label{theorem:equiv}
    Let $\mup$ be a random vector in $\mathbb{R}^{p}$. Let $\operator:\mathbb{R}^{p}\to V\cong\mathbb{R}^{N_{h}}$ be a measurable operator and define $\mathbf{u}_{\mup}:=\operator(\mup)$. Let $n\in\mathbb{N}$. If $\expe\|\mathbf{u}_{\mup}\|<+\infty$, then
    \begin{equation}
    \label{eq:density}
    \inf_{\substack{\Psi'\in\mis(V,\;\mathbb{R}^{n})\\\Psi\in\mis(\mathbb{R}^{n},\;V)}}\expe\|\mathbf{u}_{\mup}-\Psi(\Psi'(\mathbf{u}_{\mup}))\|= \inf_{\substack{\hat{\Psi}'\in\mathscr{N}_{\rho}(V,\;\mathbb{R}^{n})\\\hat{\Psi}\in\mathscr{N}_{\rho}(\mathbb{R}^{n},\;V)}}\expe\|\mathbf{u}_{\mup}-\hat{\Psi}(\hat{\Psi}'(\mathbf{u}_{\mup}))\|,
    \end{equation}
    for all Lipschitz continuous activations $\rho$ satisfying the hypothesis of Lemma \ref{lemma:density}.
\end{theorem}

\begin{proof}
    We assume the left-hand side to be finite, as the statement would be trivially true otherwise.
    Let $\Psi'\in\mis(V,\;\mathbb{R}^{n})$ and $\Psi\in\mis(\mathbb{R}^{n},\;V)$ be such that
    \begin{equation}
        \label{eq:finite}
        \mathbb{E}\|\mathbf{u}_{\mup}-\Psi(\Psi'(\mathbf{u}_{\mup}))\|<+\infty,
    \end{equation}
    and fix any $\varepsilon>0$. Let $\rho=\tanh$ be the hyperbolic tangent activation. Clearly, $\Psi$ and $\Psi'$ have the same reconstruction error as $\Psi_{\rho}:=\Psi\circ\rho^{-1}$ and $\Psi'_{\rho}:=\rho\circ\Psi'$. Therefore, up to replacing $\Psi$ with $\Psi_{\rho}$ and $\Psi'$ with $\Psi'_{\rho}$, we may assume that $\Psi'$ is bounded. 
    
    Now, let us define the random vector $\mathbf{v}_{\mup}:=\Psi'(\mathbf{u}_{\mup})$. Since $\mathbb{E}\|\mathbf{u}_{\mup}\|<+\infty$ and \eqref{eq:finite} hold, it follows by linearity that $\mathbb{E}\|\Psi(\mathbf{v}_{\mup})\|<+\infty$. In particular, we may apply Lemma \ref{lemma:density} in order to find some $\hat{\Psi}\in\mathscr{N}_{\rho}(\mathbb{R}^{n},V)$ such that
    \begin{equation*}
        \mathbb{E}\|\Psi(\mathbf{v}_{\mup})-\hat{\Psi}(\mathbf{v}_{\mup})\|<\varepsilon/2,
    \end{equation*}
    Of note, $\hat{\Psi}$ inherits the Lipschitz continuity of $\rho$. Thus, for any $\hat{\Psi}'\in\mathscr{N}_{\rho}(V,\mathbb{R}^{n})$,
    \begin{multline*}
        \left|\mathbb{E}\|\mathbf{u}_{\mup}-\Psi(\Psi'(\mathbf{u}_{\mup}))\|-\mathbb{E}\|\mathbf{u}_{\mup}-\hat{\Psi}(\hat{\Psi}'(\mathbf{u}_{\mup}))\|\right|\le\mathbb{E}\|\Psi(\Psi'(\mathbf{u}_{\mup}))-\hat{\Psi}(\hat{\Psi}'(\mathbf{u}_{\mup}))\|\le\\\le\mathbb{E}\|\Psi(\Psi'(\mathbf{u}_{\mup}))-\hat{\Psi}(\Psi'(\mathbf{u}_{\mup}))\|+\mathbb{E}\|\hat{        \Psi}(\Psi'(\mathbf{u}_{\mup}))-\hat{\Psi}(\hat{\Psi}'(\mathbf{u}_{\mup}))\|\le\\\le \frac{1}{2}\varepsilon+L\mathbb{E}\|\Psi'(\mathbf{u}_{\mup})-\hat{\Psi}'(\mathbf{u}_{\mup})\|,
    \end{multline*}
    where $L>0$ is the Lipschitz constant of $\hat{\Psi}.$ Then, by lemma \ref{lemma:density}, we may choose $\hat{\Psi}'$ so that $\mathbb{E}\|\Psi'(\mathbf{u}_{\mup})-\hat{\Psi}(\mathbf{u}_{\mup})\|<\varepsilon/2L,$ and thus
    \begin{equation}
        \label{eq:impl}
    \left|\mathbb{E}\|\mathbf{u}_{\mup}-\Psi(\Psi'(\mathbf{u}_{\mup}))\|-\mathbb{E}\|\mathbf{u}_{\mup}-\hat{\Psi}(\hat{\Psi}'(\mathbf{u}_{\mup}))\|\right|\le\varepsilon.
    \end{equation}
    Note that, this time, we were able to apply Lemma \ref{lemma:density} due to the boundness of $\Psi'$ (which in turn ensures its integrability with respect to any probability measure). As $\varepsilon>0$ is arbitrary, the inequality in \eqref{eq:impl} suffices to prove the identity in \eqref{eq:density}. 
\end{proof}


The result in Theorem \ref{theorem:equiv} is a double-edged sword. On the one hand, it allows us to reframe the optimization problem on a less restrictive class of functions, giving us the possibility, e.g., to study the behavior of the reconstruction error without having to worry about the discretization of the state space: in fact, the spaces $\mis(V,\mathbb{R}^{n})$ and $\mis(\mathbb{R}^{n},V)$ are well defined even if $V$ is infinite-dimensional. On the other hand, the autoencoders in Eq. \eqref{eq:density} can become extremely irregular and thus more difficult to capture. In fact, one can show that in most cases
\begin{equation*}
\inf_{\substack{\Psi'\in\mis(V,\;\mathbb{R}^{n})\\\Psi\in\mis(\mathbb{R}^{n},\;V)}}\expe\|\mathbf{u}_{\mup}-\Psi(\Psi'(\mathbf{u}_{\mup}))\|= 0,
    \end{equation*}
for all $n\ge1$, as there always exists a suitable space-filling curve that provides a lossless embedding. However, such a representation would be completely useless, as it would correspond to an architecture that is either impossible to reproduce or train: see, e.g., the discussion at the end of \cite{cohen2022optimal} by Cohen et al.

One way to overcome all these issues is to impose certain additional assumptions on the regularity of the autoencoder architecture. Here, we proceed as follows. We define the (enlarged) class of admissible encoders $V\to\mathbb{R}^{n}$ as
\begin{equation}
    \label{eq:encoders}
    \encoders_{B,M}(V,\;\mathbb{R}^{n}):=\left\{\Psi'\in\mis(V,\;\mathbb{R}^{n})\;\;\textnormal{s.t.}\;\;\mathbb{E}|\Psi'(\mathbf{u}_{\mup})|<+\infty\;\;\text{and}\;\;\sup_{\mathbf{v}\in B}|\Psi'(\mathbf{v})|\le M\right\},
\end{equation}
where $B\subset V$ is a \textit{control set} and $M>0$ is a suitable upper bound. Both $B$ and $M$ are to be considered as hyperparameters: their role is to ensure that, at least in the control set $B$, the encoder networks are uniformly well behaved. On the contrary, we define the (enlarged) family of admissible decoders as
\begin{equation}
    \label{eq:decoders}
    \decoders_{M,L}(\mathbb{R}^{n},\;V):=\left\{\Psi\in\mis(\mathbb{R}^{n},\;V)\;\;\textnormal{s.t.}\;\;\sup_{\mathbf{c}\in [-M,M]^{n}}\|\partial \Psi\|(\mathbf{c})\le L\right\},
\end{equation}
where $L>0$ is a suitable hyperparameter that controls the regularity of the decoder. In fact, the condition in \eqref{eq:decoders} forces $\Psi$ to be $L$-Lipschitz continuous over the hypercube $[-M,M]^{n}$: cf. (iii) in Proposition \ref{prop:lip}.
\\\\
This setup allows us to avoid the phenomenon of space-filling curves and, at the same time, to regain interest in the optimization of the reconstruction error. Furthermore, this formulation comes with a natural adaptation of Theorem \ref{theorem:equiv}, which we report below.

\begin{theorem}
    \label{theorem:equiv2}
    Let $\mup$ be a random vector in $\mathbb{R}^{p}$. Let $\operator:\mathbb{R}^{p}\to V\cong\mathbb{R}^{N_{h}}$ be a measurable map and define $\mathbf{u}_{\mup}:=\operator(\mup)$. Let $\rho$ be the $\alpha$-leaky ReLU activation, $|\alpha|<1$. Let $B\subset V$ be a bounded set, $M,L>0$ and $n\in\mathbb{N}$. Assume that the probability law of $\mathbf{u}_{\mup}$ is absolutely continuous. If $\expe\|\mathbf{u}_{\mup}\|<+\infty$, then
    \begin{equation}
       \label{eq:density2}
    \inf_{\substack{\Psi'\in\encoders_{B,M}(V,\;\mathbb{R}^{n})\\\Psi\in\decoders_{M,L}(\mathbb{R}^{n},\;V)}}\expe\|\mathbf{u}_{\mup}-\Psi(\Psi'(\mathbf{u}_{\mup}))\|=\inf_{\substack{\hat{\Psi}'\in\mathscr{N}_{\rho}^{e}(V,\;\mathbb{R}^{n})\\\hat{\Psi}\in\mathscr{N}_{\rho}^{d}(\mathbb{R}^{n},\;V)}}\expe\|\mathbf{u}_{\mup}-\hat{\Psi}(\hat{\Psi}'(\mathbf{u}_{\mup}))\|,
    \end{equation}
    where $\mathscr{N}_{\rho}^{e}(V,\;\mathbb{R}^{n}):=\mathscr{N}_{\rho}(V,\;\mathbb{R}^{n})\cap\encoders_{B,M}(V,\;\mathbb{R}^{n})$ and
    $\mathscr{N}_{\rho}^{d}(\mathbb{R}^{n},\;V):=\mathscr{N}_{\rho}(\mathbb{R}^{n},\;V)\cap\decoders_{M,L}(\mathbb{R}^{n},\;V)$
    are all admissible encoder and decoder networks, respectively. 
\end{theorem}
\begin{proof}
    The proof is roughly the same as the one of Theorem \ref{theorem:equiv}, up to replacing the use of Lemma \ref{lemma:density} with stronger results that can ensure the preservation of the constraints in Equations \eqref{eq:encoders} and \eqref{eq:decoders}. For further details, we refer to Lemmas \ref{lemma:densityK}, Lemma \ref{lemma:densityL} and Corollary \ref{corollary:density} in Appendix \ref{sec:appendix:B}.
\end{proof}

The identity in Theorem \ref{theorem:equiv2} has two major implications. First, it shows how general autoencoders can be, thus further motivating their usage for nonlinear reduction. Second, it allows us to reframe the optimization problem in an abstract way. This can be very useful, as it allows us to adopt a more general perspective where $V$ can be either discrete, $V\cong\mathbb{R}^{N_{h}}$ or continuous, e.g. $V=L^{2}(\Omega)$. In fact, the families $\encoders_{B,M}(V,\mathbb{R}^{n})$ and $\decoders_{M,L}(\mathbb{R}^n,V)$ can be defined without the need to discretize the state space. 

In light of this, for the next few pages, we shall drop our assumption on $V$ being finite-dimensional. This means that from now on, all the results reported will hold true both in the discrete and in the continuous setting. For better readability, elements of the state space will be indicated as $u\in V$, to emphasize the fact that such elements can be vectors (for which the notation $\mathbf{u}$ would be more fitting) or functions.

\begin{remark}
    Although the two definitions in \eqref{eq:encoders} and \eqref{eq:decoders} may sound quite technical, they actually mirror some of the practical strategies that researchers and data scientists commonly use. For example, when training a DNN model, it is very common to use Tychonoff regularizations to avoid excessive growth of the DNN weights. However, this procedure is equivalent to imposing admissibility constraints via a Lagrange multiplier. To see this, let $\Phi$ be some DNN. For simplicity, assume that $\Phi$ only has one hidden layer, so that $\Phi(\x)=\mathbf{W}_{2}\rho(\mathbf{W}_{1}\x+\mathbf{b}_{1})+\mathbf{b}_{2}$, where $\rho$ is some $\ell$-Lipschitz activation function. With little abuse of notation, let us denote by $|\cdot|$ both the Euclidean norm and the Frobenius norm. We have
    $$|\Phi(\x)-\Phi(\y)|\le\ell|\mathbf{W}_{2}|\cdot|\mathbf{W}_{1}|.$$
    In particular, $\Phi$ is Lipschitz continuous and the logarithm of its Lipschitz constant is bounded by $\log\ell+\log|\mathbf{W}_{2}|+\log|\mathbf{W}_{1}|$. It is then clear that penalizing the mass of weights matrices has a direct impact on the Lipschitz constant of the whole network, thus mimicking our condition on the decoder, Eq. \eqref{eq:decoders}. Furthermore, the same argument applies to the constraint for the encoder, Eq. \eqref{eq:encoders}. In fact, in our construction, the control set $B$ is always assumed to be bounded. In particular,
    $$|\Phi(\x)|\le|\Phi(0)|+|\Phi(\x)-\Phi(0)|\le|\mathbf{W}_{2}|\cdot|\rho(0)|+\ell|\mathbf{W}_{2}|\cdot|\rho(\mathbf{b}_{1})|+|\mathbf{b}_{2}|+R\ell|\mathbf{W}_{2}|\cdot|\mathbf{W}_{1}|,$$
    where $R>0$ is any radius for which one has $|\x|\le R$ for all $v\in B$. Thus, the same reasoning can be applied up to including an additional penalty for the biases $\mathbf{b}_{i}$.
\end{remark}

\subsection{Bounds on the latent dimension: finite rank case}
\label{subsec:finite}
We start by addressing the finite-rank case, in which the source of randomness is given by some random vector $\mup$. We report our main result in the following, which, in a way, can be seen as a generalization of Theorem 3 in \cite{franco2022deep} to the (unbounded) probabilistic setting.

\begin{theorem}
\label{theorem:finite}
    Let $\operator:\mathbb{R}^{p}\to V$ be a locally Lipschitz operator, where $(V,\|\cdot\|)$ is a given Banach space. Let $\mup$ be a random vector in $\mathbb{R}^{p}$ and let $u_{\mup}:=\operator(\mup)$ be the $V$-valued random variable obtained by mapping $\mup$ through $\operator$.
    Assume that the latter is Bochner integrable, that is, $\expe\|u_{\mup}\|<+\infty$. Denote by $B:=\operator(\{|\mathbf{c}|\le1\})$ the image of the unit ball. Then, there exists $L_{0}>0$ such that for all $M\ge1$ and all $L\ge L_{0}$ one has
    \begin{equation}
    \label{eq:zero}
    \inf_{\substack{\Psi'\in\encoders_{B,M}(V,\mathbb{R}^{n})\\\Psi\in\decoders_{M,L}(\mathbb{R}^{n},V)}}\expe\|u_{\mup}-\Psi(\Psi'(u_{\mup}))\|=0,
    \end{equation}  
    for all $n\ge p$.
    
    \begin{proof}
        It is sufficient to prove the case $n=p$. Let $M\ge1$. 
        Since $\operator$ is continuous, the latter admits a measurable right-inverse $g':V\to\mathbb{R}^{n}$, that is, a map for which $$\operator(g'(\operator(\mup)))=\operator(\mup).$$ Furthermore, the latter can be constructed such that $g'(B)\subseteq \{|\mathbf{c}|\le1\}$: for a detailed proof we refer the reader to the Appendix, particularly to Corollary \ref{corollary:selection}. Let now $\rho:\mathbb{R}\to\mathbb{R}$ be the following activation function
        $$\rho(x)=\begin{cases}
            x&|x|\le M\\
            M\tanh(x)/\tanh(M)&|x|>M,
        \end{cases}$$
        and define $\Psi':=\rho\circ g'$, where the action of $\rho$ is aimed at components. Then,
        \begin{equation}
        \label{eq:psiprime}
        \expe|\Psi'(u_{\mup})|\le M/\tanh(M) <+\infty\quad\text{and}\quad\Psi'(B)=g'(B)\subseteq \{|\mathbf{c}|\le1\}\subset[-M,M]^{n},
        \end{equation}
        in fact, $g'(B)\subset[-M,M]^{n}$ and $\rho$ act as the identity over $[-M,M]^{n}$. In particular, it follows from \eqref{eq:psiprime} that $\Psi'\in\encoders_{B,M}(V,\mathbb{R}^{n}).$ Now, noting that $\rho$ is invertible, we let $\Psi:=\operator\circ\rho^{-1}$, with the convention that $\rho^{-1}\equiv+\infty$ outside of $\rho(\mathbb{R})$. It is straightforward to see that the pair $(\Psi',\Psi)$ produces a lossless compression: thus, it is sufficient to prove that $\Psi\in\decoders_{M,L}(\mathbb{R}^{n},V)$ for a suitable choice of $L$. To this end, we note that $\Psi$ is locally Lipschitz over $[-M,M]^{n}$. Then, since the latter is compact, we may simply set
        $$L_{0}:=\sup_{\mathbf{c}\in[-M,M]^{n}}\|\partial\Psi\|(\mathbf{c})<+\infty,$$
        see, e.g., (ii) in Proposition \ref{prop:lip}.
    \end{proof}
\end{theorem}

The takeaway from Theorem \ref{theorem:finite} is that if the PDE depends on $p$ (scalar) random variables, then an autoencoder with latent dimension $n=p$ can compress and reconstruct solutions with arbitrary (average) accuracy. Clearly, since Equation \eqref{eq:zero} features an infimum, for this to work, the autoencoder must be sufficiently expressive in the remaining parts of the architecture, especially in the decoder. In this respect, domain practitioners can find a valuable help in recent works that address the approximation capabilities of DNNs in high-dimensional spaces; see, for instance \cite{franco2023approximation, schwab2019deep}.

\subsection{Bounds on the latent dimension: infinite dimensional case}
\label{subsec:infinite}
We are now set to discuss the infinite-dimensional case in which the input variable is given by a random field. In particular, we shall deal with random fields $\mu$ having trajectories in $L^{q}(\Omega)$ and operators of the form $\operator:L^{q}(\Omega)\to V$, for $V$ a Hilbert state space, $1\le q\le +\infty$, and $\Omega$ a bounded domain.
We shall focus on two different scenarios: one in which $q=2$ and one in which $q=+\infty$. 
To this end, we recall that for any $q\le\tilde{q}$ one has continuous embedding
\begin{equation*}
    L^{\tilde{q}}(\Omega)\hookrightarrow L^{q}(\Omega),
\end{equation*}
where each space is considered with its canonical norm. In particular, it follows that $\mathcal{C}(L^{2}(\Omega),V)\subset\mathcal{C}(L^{\infty}(\Omega), V),$ meaning that case $q=+\infty$ allows for a broader class of operators and is consequently much harder to handle. However, it is worth addressing both situations, as the two can lead to very different analyses and error bounds. We start with the simpler $L^{2}$-case. For the sake of readability, both proofs have been postponed to Section \ref{sec:proofs}.

\begin{theorem}
    \label{theorem:fieldsl2v}
    \textnormal{($L^{2}$-version)} Let $\Omega$ be a bounded domain and let $(V,\|\cdot\|)$ be a Hilbert space. Let $\mu$ be a Gaussian random field defined over $\Omega$, with a square integrable mean $m:\Omega\to\mathbb{R}$ and a square integrable covariance kernel $\cov:\Omega\times\Omega\to\mathbb{R}.$ Finally, let $\operator:L^{2}(\Omega)\to V$ be an operator satisfying the growth condition below,
        \begin{equation*}
            \|\partial\operator\|(\nu)\le Ae^{\beta\|\nu\|_{L^{2}(\Omega)}}\textnormal{ for all }\nu\in L^{2}(\Omega),
        \end{equation*}
    for some constants $A,\beta>0$.  According to the Lemmae \ref{lemma:mercer} and \ref{lemma:abstractKKL}, let
    \begin{equation}
        \label{eq:kll2}
        \mu = \expe[\mu]+\sum_{i=1}^{+\infty}\sqrt{\lambda_{i}^{\mu}}\eta_{i}\varphi_{i}
        \quad\textnormal{and}\quad
        u = \sum_{i=1}^{+\infty}\sqrt{\lambda_{i}^{u}}\omega_{i}v_{i},
    \end{equation}
    be the KKL expansions of $\mu$ and $u$, respectively. Consider the control set $$B:=\operator\left(\left\{m+\sum_{i=1}^{+\infty}\sqrt{\lambda_{i}^{\mu}}\nu_{i}\varphi_{i}\;\textnormal{s.t.}\;\sum_{i=1}^{+\infty}|\nu_{i}|^{2}\le1\right\}\right),$$
    and let $\mathbb{P}$ be the probability law of $\mu$. Then, there exist positive constants $C=C(d,\operator,\mathbb{P})$ and $L_{0}=L_{0}(n,\operator)$, such that for all $M\ge1$ and $L\ge L_{0}$ one has
    \begin{equation}
        \label{eq:field2}
     \inf_{\substack{\Psi'\in\encoders_{B,M}(V,\mathbb{R}^{n})\\\Psi\in\decoders_{M,L}(\mathbb{R}^{n},V)}}\mathbb{E}\|u_{\mu}-\Psi(\Psi'(u_{\mu}))\|\le C\min\left\{\sqrt{\sum_{i>n}\lambda_{i}^{\mu}},\;\sqrt{\sum_{i>n}\lambda_{i}^{u}}\right\},
    \end{equation}    
    for all latent dimensions $n\ge 1$.
\end{theorem}

We now report on the $L^{\infty}$-counterpart of Theorem \ref{theorem:fieldsl2v}. As we mentioned, this is a much more difficult case that requires additional care.

\begin{theorem}
    \label{theorem:fields}
    \textnormal{($L^{\infty}$-version)} Let $\Omega$ be a bounded domain and let $(V,\|\cdot\|)$ be a Hilbert space. Let $\mu$ be a Gaussian random field defined over $\Omega$, with a bounded mean $m:\Omega\to\mathbb{R}$ and a Lipschitz continuous covariance kernel $\cov:\Omega\times\Omega\to\mathbb{R}.$ Finally, let $\operator:L^{\infty}(\Omega)\to V$ be an operator satisfying the growth condition below,
        \begin{equation}
            \label{eq:thmgrowth}
            \|\partial\operator\|(\nu)\le Ae^{\gamma\|\nu\|_{L^{\infty}(\Omega)}}\textnormal{ for all }\nu\in L^{\infty}(\Omega),
        \end{equation}
    for some constants $A,\gamma>0$. According to the Lemmae \ref{lemma:mercer} and \ref{lemma:abstractKKL}, let
    \begin{equation}
        \label{eq:kl}
        \mu = \expe[\mu]+\sum_{i=1}^{+\infty}\sqrt{\lambda_{i}^{\mu}}\eta_{i}\varphi_{i}
        \quad\textnormal{and}\quad
        u = \sum_{i=1}^{+\infty}\sqrt{\lambda_{i}^{u}}\omega_{i}v_{i},
    \end{equation}
    be the KKL expansions of $\mu$ and $u$, respectively. Consider the control set $$B:=\operator\left(\left\{m+\sum_{i=1}^{+\infty}\sqrt{\lambda_{i}^{\mu}}\nu_{i}\varphi_{i}\;\textnormal{s.t.}\;\sum_{i=1}^{+\infty}|\nu_{i}|^{2}\le1\right\}\right),$$
    and let $\mathbb{P}$ be the probability law of $\mu$. Fix any small $0<\epsilon<1/2$.
    Then, there exist two positive constants, $C=C(d,\operator,\mathbb{P})$ and $L_{0}=L_{0}(n,\operator)$, such that for all $M\ge1$ and $L\ge L_{0}$ one has
    \begin{equation}
        \label{eq:field}
     \inf_{\substack{\Psi'\in\encoders_{B,M}(V,\mathbb{R}^{n})\\\Psi\in\decoders_{M,L}(\mathbb{R}^{n},V)}}\mathbb{E}\|u_{\mu}-\Psi(\Psi'(u_{\mu}))\|\le C\sqrt{\log(1/\epsilon)}\min\left\{\left\|\sum_{i>n}\lambda_{i}^{\mu}\varphi_{i}^{2}\right\|_{L^{\infty}(\Omega)}^{1/2},\;\sqrt{\sum_{i>n}\lambda_{i}^{u}}\right\},
    \end{equation}    
    for all latent dimensions $n\ge 1$ satisfying either $\sum_{i>n}\lambda_{i}^{\mu}\ge\epsilon$ or $\sum_{i>n}\lambda_{i}^{\mu}=0.$
\end{theorem}

The strength of the two Theorems lies in that they show how nonlinear autoencoders can simultaneously benefit from the regularity of both the input and the output fields, something that is clearly not possible with linear methods alone: see, e.g., the discussion by Lanthaler et al. in \cite{lanthaler2022error}, Section 3.4.1. 
Furthermore, the results in Theorem \ref{theorem:fieldsl2v} and \ref{theorem:fields} are fairly general, since they are both framed in a purely abstract fashion with mild assumptions on the regularity of the forward operator. For example, following our previous discussion in Section \ref{subsec:local}, we note that the result in Theorem \ref{theorem:fields} can be easily applied to the case of Darcy flows in porous media.

\begin{remark}
    \label{remark:gaussian}
    All the results reported within this Section are limited to the case of Gaussian processes. While this is an extremely broad class of stochastic processes, one may wonder whether similar results can be obtained for other probability distributions. In general, the $L^{2}$ case, namely Theorem \ref{theorem:fieldsl2v}, can be readily applied to any random field $\mu$ that satisfies
    \begin{equation}
        \label{eq:expintl2}
        \mathbb{E}\left[e^{\beta\|\mu\|_{L^{2}(\Omega)}}\right]<+\infty\text{ for all }\beta>0.
    \end{equation}
    In fact, with such an exponential integrability, it is straightforward to see that these processes admit a KKL expansion and that the proof of Theorem \ref{theorem:fieldsl2v} can be easily adapted. For the $L^{\infty}$-case, instead, stronger assumptions are required. In particular, these should be sufficiently demanding to ensure that the properties analogous to those of Lemmas \ref{lemma:bound} and \ref{lemma:mercer} hold.

    Clearly, one may also go the other way around, i.e. by restricting the analysis to more regular operators, with the advantage of allowing for a larger class of probability distributions. For example, Theorems \ref{theorem:fieldsl2v} and \ref{theorem:fields} impose an exponential bound on the local variation, $\|\partial\operator\|$. This condition is trivially satisfied by all Lipschitz continuous operators, as Lipschitz continuity is a far more stringent property (cf. Lemma \ref{prop:lip}). In particular, if one restricts the attention to such operators, the proof of Theorem \ref{theorem:fields} can be re-adapted with weaker assumptions on the random field. That is, one needs to have $\mathbb{E}\|\mu\|_{L^{\infty}(\Omega)}<+\infty$ and
    \begin{equation*}
    \mathbb{E}\|\mu-\mupi\|_{L^{\infty}(\Omega)} = O\left(\|\text{Cov}_{\mu-\mupi}\|_{L^{\infty}(\Omega)}\right),
    \end{equation*}
    where $\text{Cov}_{\mu-\mupi}$ is the covariance kernel of $\mu-\mupi.$. Conversely, the case $L^{2}$ becomes trivial, as one can replace the condition in \eqref{eq:expintl2} with $\mathbb{E}\|\mu\|_{L^{2}(\Omega)}<+\infty$.
\end{remark}

\begin{remark}
    \label{remark:boundness}
    Theorems \ref{theorem:fieldsl2v} and \ref{theorem:fields} provide different bounds for the reconstruction error. However, the two become very similar if the eigenfunctions of the input field, $\varphi_{i}$, are uniformly bounded. In fact, if $\sup_{i}\|\varphi_{i}\|_{L^{\infty}(\Omega)}\le D$ for some $D>0$, then
    $$\left\|\sum_{i>n}\lambda_{i}^{\mu}\varphi_{i}^{2}\right\|^{1/2}_{L^{\infty}(\Omega)}\le\sqrt{\sum_{i>n}\lambda_{i}^{\mu}\|\varphi_{i}^{2}\|_{L^{\infty}(\Omega)}}\le\sqrt{\sum_{i>n}\lambda_{i}^{\mu}D^{2}}= D\sqrt{\sum_{i>n}\lambda_{i}^{\mu}}.$$
    In particular, up to fixing the value of $\epsilon$ in Theorem \ref{theorem:fields}, and adjusting the value of the multiplicative constant $C>0$, one can replace the bounding expression in Theorem \ref{theorem:fields} with that in Theorem \ref{theorem:fieldsl2v}.
    However, whether such a uniform boundness holds or not depends on the problem itself. In fact, although erroneously stated by some authors, this property is not directly implied by the regularity of the covariance kernel: see, e.g., \cite{zhou2002covering} for an instructive counterexample.
\end{remark}

\subsection{Proofs of Theorems \ref{theorem:fieldsl2v} and \ref{theorem:fields}}
\label{sec:proofs}
The interested reader can find below the proofs of the two Theorems, which we have postponed here due to their lengths and technicalities. We start with the proof of Theorem \ref{theorem:fields}, as that is arguably the most difficult and most interesting one; the $L^{2}$ case will then follow quite easily.

\subsubsection*{Proof of Theorem \ref{theorem:fields}}

To begin, we note that the case $\sum_{i>n}\lambda_{i}^{\mu}=0$ is already covered by Theorem \ref{theorem:finite}, as it falls into the finite-rank context. Thus, from now on we shall focus on proving the error bound for the remaining case, that is, when $\sum_{i>n}\lambda_{i}^{\mu}\ge\epsilon$.

Without loss of generality, we assume that $\expe[\mu]\equiv0$. Before starting with the proof, we note that the existence of a KKL expansion for $u$ is guaranteed by the exponential growth condition for $\operator$. In fact, by \eqref{eq:thmgrowth} and Corollary \ref{cor:exp}, it follows that
$$\expe\|u_{\mu}\|^{2}\le \expe\left[\left(A'e^{\gamma'\|\mu\|_{L^{\infty}(\Omega)}}\right)^{2}\right] =  (A')^{2}\expe\left[e^{2\gamma'\|\mu\|_{L^{\infty}(\Omega)}}\right]<+\infty,$$
for some $A',\gamma'>0$, where the last inequality is a direct consequence of Lemma \ref{lemma:bound}. In particular, $u_{\mu}$ is a squared-integrable $V$-valued random variable, and thus admits a KKL expansion (cf. Lemma \ref{lemma:abstractKKL}).
\\\\
Let now $n\in\mathbb{N}$, with $n\ge1$, and let
$$\delta_{n,M,L}:=\inf_{\substack{\Psi'\in\encoders_{B,M}(V,\mathbb{R}^{n})\\\Psi\in\decoders_{M,L}(\mathbb{R}^{n},V)}}\mathbb{E}\|u_{\mu}-\Psi(\Psi'(u_{\mu}))\|.$$
We shall split the proof into several steps. More precisely, we shall prove the following.
\\\\
\textbf{Claim 1}. The definition of the control set, $B$, is well-posed.
\\\\
\textbf{Claim 2.} $\exists\ell_{0}=\ell_{0}(n,\operator,\mathbb{P})$ such that $\delta_{n,M,L}\le\sqrt{\sum_{i>n}\lambda_{i}^{u}}$ for all $M\ge1$ and $L\ge \ell_{0}$.
\\\\
\textbf{Claim 3.} $\exists\ell_{0}'=\ell_{0}'(n,\operator,\mathbb{P})$ such that $\delta_{n,M,L}\le\expe\|u-u_{\mupi}\|$ for all $M\ge1$ and $L\ge \ell_{0}'$. Here, $u_{\mupi}$ is the operator image where the input field $\mu$ has been replaced by its $n$ th KKL truncation, $\mu_{n}$.
\\\\
\textbf{Claim 4.} $\expe\|u-u_{\mupi}\|\le c\expe^{1/2}\|\mu-\mu_{n}\|_{L^{\infty}(\Omega)}^{2}$ for some $c=c(d,\operator,\mathbb{P})$.
\\\\
\textbf{Claim 5.} $\expe^{1/2}\|\mu-\mu_{n}\|_{L^{\infty}(\Omega)}^{2}\le c'\sqrt{\log(1/\epsilon)}\left\|\sum_{i>n}\lambda_{i}^{\mu}\varphi_{i}^{2}\right\|_{L^{\infty}(\Omega)}^{1/2}$ for some $c'=c'(d,\mathbb{P})$.
\\\\
Clearly, once all of the above have been proven, setting $C=\max\{cc',1\}$ and $L_{0}:=\max\{\ell_{0},\ell_{0}'\}$ quickly yields the conclusion. Thus, we now proceed to prove the five claims one by one.
\\\\
\textit{Proof of Claim 1.} Let $B_{1}:=\{\sum_{i=1}^{+\infty}\sqrt{\lambda_{i}^{\mu}}\nu_{i}\varphi_{i}\;|\;\sum_{i}|\nu_{i}|^{2}\le1\}$. Then, for every element of $B_{1}$ and every $\x\in\Omega$, we have
$$\left|\sum_{i=1}^{+\infty}\sqrt{\lambda_{i}^{\mu}}\nu_{i}\varphi_{i}(\x)\right|\le\sqrt{\sum_{i=1}^{+\infty}|\nu_{i}|^{2}}\sqrt{\sum_{i=1}^{+\infty}\lambda_{i}^{\mu}\varphi_{i}(\x)^{2}}\le\sqrt{\cov(\x,\x)}.$$
Since $\cov$ is bounded, this shows that $B_{1}$ is a bounded subset of $L^{\infty}(\Omega)$. In particular, since $B=\operator(B_{1})$, the definition of the control set is well-posed. Furthermore, the latter is $\|\cdot\|$ bounded, as $\operator$ maps bounded sets onto bounded sets (cf. Corollary \ref{cor:exp}).
\qed 
\\\\
\textit{Proof of Claim 2.} Let $P:\mathbb{R}^{n}\to V$ and $P^{\dagger}:V\to\mathbb{R}^{n}$ be the linear operators below
$$P:[c_{1},\dots,c_{n}]\mapsto\sum_{i=1}^{n}\sqrt{\lambda_{i}^{u}}c_{i}v_{i},\quad\quad P^{\dagger}:v\mapsto\left[\frac{1}{\sqrt{\lambda_{1}^{u}}}\langle v,v_{1}\rangle,\dots,\frac{1}{\sqrt{\lambda_{n}^{u}}}\langle v,v_{n}\rangle\right].$$
Since $P^{\dagger}$ is both linear and continuous, and $B$ is bounded, the image $P^{\dagger}(B)$ is also bounded. Thus, let $M_{0}:=\sup_{\mathbf{c}\in P^{\dagger}(B)}|\mathbf{c}|<+\infty$. Similarly, in light of the linearity and continuity of $P$, let $\text{Lip}(P)$ be the Lipschitz constant of $P$. Define the maps
$$\tilde{P}^{\dagger}:= M_{0}^{-1}P,\quad\quad \tilde{P}:=M_{0}P.$$
It is straightforward to see that $\tilde{P}^{\dagger}\in\encoders_{B,M}(V,\mathbb{R}^{n})$ and $\tilde{P}^{\dagger}\in\decoders_{M,L}(\mathbb{R}^{n},V)$, for all $M\ge1$ and all $L\ge \ell_{0}:=M_{0}\text{Lip}(P)$, where $\ell_{0}$ ultimately depends on $n,$ $\operator$ and $\mathbb{P}$. Since
$$\expe\|u-\tilde{P}\tilde{P}^{\dagger}u\|\le\sqrt\expe\|u-PP^{\dagger}u\|^{2}=\expe^{1/2}\left\|\sum_{i>n}\sqrt{\lambda_{i}^{u}}\omega_{i}v_{i}\right\|^{2}=\sqrt{\sum_{i>n}\lambda_{i}^{u}},$$
this proves Claim 2.\qed 
\\\\
\textit{Proof of Claim 3.} 
As before, it is useful to define the mappings 
$Q^{\dagger}:\mathbb{R}^{n}\to L^{\infty}(\Omega)$ and 
$Q^{\dagger}:L^{\infty}(\Omega)\to\mathbb{R}^{n}$ as
$$Q:[c_{1},\dots,c_{n}]\mapsto\sum_{i=1}^{n}\sqrt{\lambda_{i}^{\mu}}c_{i}\varphi_{i},
\quad\quad Q^{\dagger}:\nu\to\left[\frac{1}{\sqrt{\lambda_{1}^{\mu}}}\langle \nu,\varphi_{1}\rangle,\dots,\frac{1}{\sqrt{\lambda_{n}^{\mu}}}\langle \nu,\varphi_{n}\rangle\right],$$
which are both linear and continuous. Following our previous notation, for any $R>0$, let $B_{R}=\{\|\nu\|_{L^{\infty}(\Omega)}\le r\}\subseteq L^{\infty}(\Omega)$ be the closed ball of radius $R$.
We define the map $\Psi'_{R}:V\to \mathbb{R}^{n}$ as any measurable selection of the optimization problem below,
\begin{equation*}
    \Psi^\prime_R: v\mapsto \argmin_{\mathbf{\nup}\in Q^{\dagger}(B_{R})}\|v-\operator(Q\nup)\|,
\end{equation*}
The existence of such a map is a straightforward consequence of standard results in set-valued analysis. In fact: 
\begin{itemize}
    \item [i)] the set of minimizers, $Q^{\dagger}(B_{R})\subset\mathbb{R}^{n}$, is compact. This is because $B_{R}$ is both closed and bounded, while the map $Q^{\dagger}$ is linear, continuous, and has finite rank;\\
    \item [ii)] the map objective functional, $(v,\nup)\to\|v-\operator(Q\nup)\|$, is continuous.
\end{itemize}  
Then, these two properties are enough to guarantee the existence of a measurable map acting as "minimal selection": for the interested reader, we refer to the Appendix, Lemma \ref{lemma:optselec}.
Let now $$R_{0}:=\|\cov\|^{1/2}_{L^{\infty}(\Omega\times\Omega)}.$$ In light of our calculations at the beginning of the proof, we note that the control set $B$ is a subset of $\operator(B_{R})$ for all $R\ge R_{0}$. Thus, assuming $R\ge R_{0}$, we set
\begin{equation*}
    \tilde{\Psi}^\prime_{R}(v):=\mathbbm{1}_{B}(v)\Psi^\prime_{R_{0}}(v)+\mathbbm{1}_{\operator(B_{R})\setminus B}(v)\Psi^\prime_R (v).
\end{equation*}
Since both sets $B$ and $\operator(B_{R})$ are measurable (cf. Lemma \ref{lemma:measurability} in the Appendix), and $\tilde{\Psi}'_{R}(V)\subseteq Q^{\dagger}(B_{R})$, the above is both measurable and bounded (thus integrable). Furthermore, the above construction ensures that $\tilde{\Psi}'_{R}\equiv0$ outside of $\operator(B_{R})$, and, most importantly
$$|\tilde{\Psi'_{R}}(v)|\le \frac{n|\Omega|^{1/2}R_{0}}{\sqrt{\lambda_{n}^{\mu}}}$$
for all $v\in B$. In fact, for any $\nu$ with $\|\nu\|_{L^{\infty}(\Omega)}\le R_{0}$ one has
$$\left|\frac{1}{\sqrt{\lambda_{i}^{\mu}}}\langle \nu, \varphi_{i}\rangle\right|\le \frac{|\Omega|^{1/2}}{\sqrt{\lambda_{i}^{\mu}}}R_{0},$$
by the Cauchy-Schwartz inequality. In particular, $\tilde{\Psi}'_{R}$ is bounded on $B$, with a constant that is independent of $R$. Furthermore, by very definition,
\begin{equation}
\label{eq:optimality}
\|\operator(\nu)-\operator(Q\tilde{\Psi}'_{R}(\nu))\|\le\|\operator(\nu)-\operator(QQ^{\dagger}\nu)\|\end{equation}
for all $\nu\in B_{R}$. Now, exploiting the compactness of $[-1,1]^{n}$, let $\text{Lip}(\operator\circ Q)$ be the Lipschitz constant of $\operator\circ Q$ over $[-1,1]^{n}$ (recall that the latter has a finite due composition; see, e.g., (v) and (ii) in Proposition \ref{prop:lip}). Define the maps
$$\hat{\Psi}_{R}':v\mapsto \frac{\sqrt{\lambda_{n}^{\mu}}}{n|\Omega|^{1/2} R_{0}}\tilde{\Psi}_{R}',\quad\quad\Psi:\mathbf{c}\to \operator\left(\frac{n|\Omega|^{1/2}R_{0}}{\sqrt{\lambda_{n}^{\mu}}}Q\mathbf{c}\right).$$
Then the couple $(\hat{\Psi}_{R}',\;\Psi)$ forms an admissible encoding-decoding pair for all $M\ge1$ and $L\ge\ell_{0}'$, where $$\ell_{0}':=\text{Lip}(\operator\circ Q)n|\Omega|^{1/2}R_{0}/\sqrt{\lambda_{n}^{\mu}}$$
depends only on $n$, $\operator$ and $\mathbb{P}$. Consequently, for all $M\ge1$ and $L\ge\ell_{0}'$, one has
\begin{multline} 
\label{eq:daje}
\delta_{n,M,L}\le\mathbb{E}\|u_{\mu}-\operator_{n}(\tilde{\Psi}'_{R}(u_{\mu}))\|\le\\\le \mathbb{E}\left[\mathbbm{1}_{B_{R}}(\mu)\|u_{\mu}-\operator_{n}(\tilde{\Psi}'_{R}(u_{\mu}))\|\right]+\mathbb{E}\left[\mathbbm{1}_{B_{R}^{c}}(\mu)\|u_{\mu}-\operator_{n}(\tilde{\Psi}'_{R}(u_{\mu}))\|\right]. 
\end{multline}
For better readability, we now write $u_{\mupi}:=\operator(QQ^{\dagger}\mu)$, so that $u_{\mupi}$ is the image of the operator that is obtained by replacing the input field $\mu$ with its $n$-truncation. Then, by the very definition of $\tilde{\Psi}_{R}'$ and thanks to \eqref{eq:optimality}, we can continue \eqref{eq:daje} as
\begin{equation*} 
\delta_{n,M,L}\le\dots\le \mathbb{E}\|u_{\mu}-u_{\mupi}\|+\mathbb{E}\left[\mathbbm{1}_{B_{R}^{c}}(\mu)\|u_{\mu}-u_{0}\|\right]. 
\end{equation*}
Since the above holds for every $R\ge R_{0}$, we can let $R\to+\infty$. In doing so, we note that $\|u_{\mu}-u_{0}\|$ is an integrable random variable. Thus, $\mathbb{E}\left[\mathbbm{1}_{B_{R}^{c}}(\mu)\|u_{\mu}-u_{0}\|\right]\to0$ by dominated convergence. Claim 3 follows.\qed
\\\\
\textit{ Proof of Claim 4.}
As a direct consequence of Corollary \ref{cor:lipgen} and Eq. \eqref{eq:thmgrowth}, we have
\begin{align*}
     \|u_{\mu}-u_{\mupi}\|&\le\left(\sup_{0\le t\le 1}Ae^{\gamma\|t\mu+(1-t)\mupi\|_{L^{\infty}(\Omega)}}\right)\|\mu-\mupi\|_{L^{\infty}(\Omega)}\\&\le Ae^{\gamma\|\mu\|_{L^{\infty}(\Omega)}+\gamma\|\mupi\|_{L^{\infty}(\Omega)}}\|\mu-\mupi\|_{L^{\infty}(\Omega)}.
     \end{align*}
Then, by the Cauchy-Schwarz inequality, we have
\begin{equation*}
\mathbb{E}\|u_{\mu}-u_{\mupi}\|\le\\\le A \mathbb{E}^{1/4}\left[e^{4\gamma\|\mu\|_{L^{\infty}(\Omega)}}\right]\mathbb{E}^{1/4}\left[e^{4\gamma\|\mupi\|_{L^{\infty}(\Omega)}}\right]\mathbb{E}^{1/2}\|\mu-\mupi\|_{L^{\infty}(\Omega)}^{2}. 
\end{equation*}
For better readability, let now
$$k_{q}(\x,\y)=\sum_{i=1}^{q}\lambda_{i}^{\mu}\varphi_{i}(\x)\varphi_{i}(\y),$$
so that $k_{\infty}=\cov$ and $k_{n}$ is the covariance kernel of $\mupi$. Of note, as a straightforward consequence of the Cauchy-Schwartz inequality, one has
\begin{equation}
    \label{eq:covinfinity}
    \|k_{q}\|_{L^{\infty}(\Omega^{2})} = \max_{x\in\Omega}k_{q}(\x,\x).
\end{equation}
Now, since both kernels are Lipschitz continuous, by Lemmas \ref{lemma:bound} and \ref{lemma:mercer}, we have
\begin{equation*}
\mathbb{E}^{1/4}\left[e^{4\gamma\|\mu\|_{L^{\infty}(\Omega)}}\right]=c_{2}(d,H(k_{\infty}),1/2,\|k_{\infty}\|_{L^{\infty}(\Omega^{2})},4\gamma,\Omega)\end{equation*}
and 
\begin{equation*}\mathbb{E}^{1/4}\left[e^{4\gamma\|\mupi\|_{L^{\infty}(\Omega)}}\right]=c_{2}(d,H(k_{n}),1/2,\|k_{n}\|_{L^{\infty}(\Omega^{2})},4\gamma,\Omega),\end{equation*}
    where $H(k_{\infty})$ and $H(k_{n})$ are the Hölder constants of the two kernels. We now recall that, as shown in Lemma \ref{lemma:mercer}, $k_{n}\to k_{\infty}$ uniformly and $\sup_{q}L(k_{q})<+\infty$. In particular, $\sup_{q}\|k_{q}\|_{L^{\infty}(\Omega^{2})}<+\infty$, and, since $c_{2}$ depends continuously on its parameters, we have
$$\mathbb{E}^{1/4}\left[e^{4\gamma\|\mu\|_{L^{\infty}(\Omega)}}\right]\cdot\mathbb{E}^{1/4}\left[e^{4\gamma\|\mupi\|_{L^{\infty}(\Omega)}}\right]\le \tilde{c},$$
    for some $\tilde{c}=\tilde{c}(d,\gamma,\mathbb{P})=\tilde{c}(d,\operator,\mathbb{P}).$
    In particular, up to letting $c:=A\tilde{c}$, we may rewrite our previous bound as,
    \begin{equation*}
    \mathbb{E}\|u_{\mu}-u_{\mupi}\|\le \tilde{C}\mathbb{E}^{1/2}\|\mu-\mupi\|_{L^{\infty}(\Omega)}^{2}. 
    \end{equation*}\qed 
\\\\
\textit{Proof of Claim 5.} Following the same notation as above, we note that $k_{\infty}-k_{n}$ is the covariance kernel of the random field $\mu-\mupi$. Thus, by Lemma \ref{lemma:bound},
    \begin{equation}
    \label{eq:logbound}
  \mathbb{E}^{1/2}\|\mu-\mupi\|_{L^{\infty}(\Omega)}^{2}\le
  \tilde{c}'\|k_{\infty}-k_{n}\|^{1/2}_{L^{\infty}(\Omega^{2})}\left(1+\sqrt{\log^{+}(1/\|k_{\infty}-k_{n}\|_{L^{\infty}(\Omega^{2})})}\right)\end{equation}
  where, by applying the same arguments as before, $\tilde{c}'$ can be chosen to depend only on $d$ and $\mathbb{P}$. 
  We now note that
  $$\|k_{\infty}-k_{n}\|_{L^{\infty}(\Omega^{2})}\ge \|k_{\infty}-k_{n}\|_{L^{2}(\Omega^{2})}|\Omega|=\sqrt{\sum_{i>n}\lambda_{i}}|\Omega|\ge\epsilon^{1/2}|\Omega|.$$ The latter can be then combined with \eqref{eq:logbound} to prove that, up to replacing $\tilde{c}'$ with a suitable $c'_{0}$, independent of $\epsilon$, one has  
  $$\mathbb{E}^{1/2}\|\mu-\mupi\|_{L^{\infty}(\Omega)}^{2}\le
  c'_{0}\|k_{\infty}-k_{n}\|^{1/2}_{L^{\infty}(\Omega^{2})}\left(1+\sqrt{\log^{+}(1/\epsilon)}\right),$$
  as the map $a\mapsto\log^{+}(1/a)$ is monotone nonincreasing.
  Furthermore, $$\epsilon<1/2\implies\log^{+}(1/\epsilon)=\log(1/\epsilon)\ge \log(2)>0.$$ Thus, up to further replacing $c'_{0}$ with a proper $c'=c'(d,\mathbb{P})$, we may write
    \begin{equation*}
  \mathbb{E}^{1/2}\|\mu-\mupi\|_{L^{\infty}(\Omega)}^{2}\le
  c'\sqrt{\log(1/\epsilon)}\|k_{\infty}-k_{n}\|_{L^{\infty}(\Omega^{2})}^{1/2}=c'\sqrt{\log(1/\epsilon)}\left\|\sum_{i>n}\lambda_{i}^{\mu}\varphi_{i}^{2}\right\|_{L^{\infty}(\Omega)}^{1/2},\end{equation*}
  where the last equality follows from \eqref{eq:covinfinity}. \qed
\\\\
Finally, putting together the five Claims yields the inequality in \eqref{eq:field}, and thus proves the statement in the Theorem.\qed

\subsection*{Proof of Theorem \ref{theorem:fieldsl2v}}
Let $n\ge1$. We notice that, mutatis mutandis, all the steps in the proof of Theorem \ref{theorem:fields} can be carried out following the same ideas. The only part that actually changes is the estimate in Claim 5, which now concerns the quantity
$$\expe^{1/2}\|\mu-\mupi\|^{2}_{L^{2}(\Omega)},$$
where $\mupi\approx\mu$ is the $n$th KKL truncation of the random field $\mu.$ However, the latter can be estimated trivially as, by orthonormality, one has
\begin{equation*}
    \mathbb{E}\|\mu-\mupi\|_{L^{2}(\Omega)}^{2} = \mathbb{E}\left[\sum_{i>n}\lambda_{i}^{\mu}\eta_{i}^{2}\right]= \sum_{i>n}\lambda_{i}^{\mu},
\end{equation*}
so that the conclusion follows.\qed

\section{Numerical experiments}
\label{sec:experiments}
The purpose of this Section is to assess the error estimates in Theorems \ref{theorem:fieldsl2v} and \ref{theorem:fields} through a set of numerical experiments. To do so, we proceed in a schematic way so that we may synthesize as follows. First, we introduce the PDE of interest, together with the corresponding solution operator $\operator:\mu\to u_{\mu}$, and a given probability distribution $\mathbb{P}$ defined over the input space, $\mu\sim\mathbb{P}$. Then, we fix a suitable high-fidelity discretization of the input-output spaces, typically via Finite Elements or Finite Volumes, so that the operator under study becomes $\operator_{h}:\mup^{h}\to\mathbf{u}_{\mup^{h}}^{h}$ (here, the superscript $h$ is used to emphasize the presence of a spatial discretization). 
The discrete operator is then evaluated relying on a given numerical solver, which we exploit to generate a suitable \textit{training set}, that is, a collection of random independent realizations $\{\mup^{h}_{i},\mathbf{u}^{h}_{i}\}$, where
$$
\mathbf{u}^{h}_{i}:=\operator_{h}(\mup^{h}_{i})\in\mathbb{R}^{N_{h}},
$$
and the $\mup^{h}_{i}$'s are sampled according to $\mathbb{P}$. 

To check whether the error bounds in Theorems \ref{theorem:fieldsl2v} and \ref{theorem:fields} are observed in practice, we consider a sequence of "nested" autoencoder architectures, $\Psi_{n}\circ\Psi_{n}'$, characterized by an increasing latent dimension $n$, so that
$$\Psi_{n}':\mathbb{R}^{N_{h}}\to\mathbb{R}^{n},\quad\quad\Psi_{n}:\mathbb{R}^{n}\to\mathbb{R}^{N_{h}}.$$
The architectures are nested in that they share the same depth and widths, except for the innermost layers (that is, those mapping to/from the latent space, respectively). All the architectures are then trained over the same training set, and their performances are evaluated in terms of the test error below
\begin{equation}
\label{eq:testerror}
E_{\text{test}}(n):=\frac{1}{N_{\text{test}}}\sum_{j=1}^{N}\|\mathbf{u}^{h}_{j,\text{test}}-\Psi_{n}(\Psi_{n}'(\mathbf{u}^{h}_{j,\text{test}}))\|_{V_{h}},\end{equation}
which serves as a Monte Carlo estimate of $\mathbb{E}\|\mathbf{u}^{h}_{\mup^{h}}-\Psi_{n}(\Psi'_{n}(\mathbf{u}^{h}_{\mup^{h}}))\|_{V_{h}}.$ Here, $\{\mup^{h}_{j,\text{test}},\mathbf{u}^{h}_{j,\text{test}}\}_{j=1}^{N_{\text{test}}}$ is a suitable test set, generated independently of the training set. On the contrary, the norm $\|\cdot\|_{V_{h}}$ corresponds to the discretized $L^{2}$-norm associated to the underlying Finite Element (or Finite Volume) space, $V_{h}\subseteq L^{2}(\Omega)$, i.e.
$$\|\mathbf{v}\|_{V_{h}}:=\sqrt{\mathbf{v}^{T}\mathbf{M}\mathbf{v}}$$
for all $\mathbf{v}\in\mathbb{R}^{N_{h}}$, where $\mathbf{M}\in\mathbb{R}^{N_{h}\times N_{h}}$ is the so-called \textit{mass matrix}. The purpose is then to compare the behavior of the three quantities below
\begin{equation}
\label{eq:quantities}
E_{\text{test}}(n),\quad\quad\sqrt{\sum_{i>n}\lambda_{i}^{\mu}}, \quad\quad\sqrt{\sum_{i>n}\lambda_{i}^{u}},
\end{equation}
for varying $n$, as suggested by Theorems \ref{theorem:fieldsl2v} and \ref{theorem:fields}. To this end, we exploit the POD algorithm to approximate the first $n$ eigenvalues of the (uncentered) covariance operators of the two fields. Then, the two tails can be easily approximated by noting that
$$\sum_{i>n}\lambda_{i}^{\mu}=\sum_{i=1}^{+\infty}\lambda_{i}^{\mu}-\sum_{i=1}^{n}\lambda_{i}^{\mu} \approx \text{Var}\|\mup^{h}\|_{V_{h}}^{2}-\sum_{i=1}^{n}\lambda_{i}^{\mup^{h}},$$
where $\lambda_{i}^{\mup^{h}}\approx \lambda_{i}^{\mu}$ are the eigenvalues computed via POD. The same can be done for $u$. For our analysis, we typically let $n=1,2,\dots,6$, so that the resulting autoencoder architectures are fairly light. This allows us to keep external sources of error out of the way (such as, e.g., inaccuracies due to an inexact optimization of the loss function or shortage in the training data), and thus provide cleaner results.
\\\\
Once the theoretical error bounds have been assessed, we also take the opportunity to implement a complete DL-ROM surrogate: while this is not directly related to our analysis in Section \ref{sec:theory}, it serves the purpose of showing how the whole machinery can be put into action to provide an operative ROM. For simplicity, this step is only repeated once with a fixed latent dimension of choice.

Last but not least, we mention that all the code supporting the forthcoming analysis has been written in Python 3, specifically relying on the FEniCS and Pytorch libraries. The code is available upon request to the authors.

\subsection{Problems description}
\label{sec:problems}
We start by introducing the three case studies one by one: all the results are then reported and discussed at the end of this Section.
For the sake of our analysis, we have selected two prototypical problems that aim to show different possible behaviors of the solution manifold. The first case study features a diffusive phenomenon, which results in a highly regularizing solution operator: as a consequence, this problem models a scenario in which the eigenvalues of the solution field, $\lambda_{i}^{u}$, decay faster than those of the input field, $\lambda_{i}^{\mu}$. Conversely, the second case study concerns a nonlinear advection characterized by the presence of shock waves and thus constitutes a remarkable example of the opposite situation (i.e., $\lambda_{i}^{\mu}\to0$ faster than $\lambda_{i}^{u}$). 

We believe that despite their simplicity, these problems suffice to show the practical counterpart of our theoretical analysis in Section \ref{sec:theory}. In this respect, we remark that this work does not aim at showcasing the abilities of DL-ROMs in handling complex problems: in fact, the effectiveness of the DL-ROM approach has already been reported elsewhere; see, e.g., \cite{fresca2022pod} for problems concerning fluid dynamics or multiphysics.

\subsubsection{Stochastic Darcy flow with random permeability}
\label{subsec:diff}
To start, we consider an elliptic problem that describes diffusion in a porous medium with random permeability, in the same spirit as our previous discussion in Section \ref{sec:preliminaries}, cf. Proposition \ref{prop:darcy}. More precisely, let $\Omega:=(0,1)^{2}$ be the unit square. We consider the boundary value problem below
\begin{equation}
\label{eq:diffusion}
\begin{cases}
    -\nabla\cdot(e^{\mu}\nabla u)=10 &  \textnormal{in}\;\Omega\\
            u=0 & \textnormal{on}\;\partial\Omega
\end{cases},
\end{equation}
where $\mu$ is a centered Gaussian random field with covariance
$$\cov(\x,\y):=\exp(-|\x-\y|^{2}).$$
Here, we focus our attention on the solution operator $\operator:\mu\to u_{\mu}$
that maps the log-permeability of the medium onto the corresponding  solution to \eqref{eq:diffusion}. We discretize the problem using piecewise linear continuous Finite Elements over a structured triangular mesh of step size $h=\sqrt{2}/50$, resulting in a state space $V_{h}\subseteq L^{2}(\Omega)$ with $N_{h}=2601$ degrees of freedom. We generate a total of 5000 snapshots, split between training and testing with a 90:10 ratio.
For technical details on the DL-ROM architectures and their training, we refer the reader to the appendix, Section \ref{sec:architectures}.

\subsubsection{Burger's equation with random data}
\label{subsec:burger}
On the segment $\Omega:= (0,L)$, $L=5$, we consider the inviscid Burger equation
\begin{equation}
\label{eq:burgers}
\displaystyle
  \frac{\partial v}{\partial t} + \frac{1}{2}v\frac{\partial v}{\partial x} = 0,   
\end{equation} 
describing the nonlinear transport of a given solute $v=v(x,t)$ with random initial condition $v(\cdot,0)=\mu$. We complement Eq. \eqref{eq:burgers} with a stationary inflow condition on the left boundary, that is, $v(0,t)\equiv v(0,0)$ for all $t\ge0$. Here, we model the trajectories of the random field $\mu$ as
\begin{equation}\label{eq:burgersmu}\mu(x) = \frac{1}{2}\rho\left(\phi_{0}(x)+\sum_{k=1}^{+\infty}\frac{1}{k^{2}}\eta_{k}\sin\left(\frac{k\pi x}{L}\right)\right)\end{equation}
where $\{\eta_{i}\}_{i=1}^{+\infty}$ are i.i.d. standard Gaussians,  $\rho(x):=\min\{\max\{x,0\},1\}$ is a suitable transformation that clamps the data within $[0,1]$, and
$$\phi_{0}(x):=(x-1)(2-x)\mathbbm{1}_{[1,2]}(x).$$
We are interested in approximating the following parameter-to-solution operator
\begin{equation}
    \label{eq:burgersoperator}
    \operator:\mu\to u_{\mu}:=v(\cdot,T),
\end{equation}
which maps any given initial profile onto the state of the system at time $T=2.$ We discretize the problem using the Finite Volume method, with a temporal step $\Delta t=0.01$ and cells size $h=0.01$, resulting in a high-fidelity state space of dimension $N_{h}=500.$ We exploit the FOM to generate a total of 2000 snapshots, split between training and testing with a 90:10 ratio. 
Once again, to keep the paper self-contained, we postpone all the technical details about the DL-ROM architectures and their training to the Appendix, Section \ref{sec:architectures}.

\subsection{Results}

We start with the first case study, Section \ref{subsec:diff}. As we mentioned, the parameter-to-solution operator of this problem is highly regularizing because of the diffusive term in the PDE. Indeed, when comparing the tails of the eigenvalues, we see that those of the output field decay three times faster than those of the input field; see Figure \ref{fig:decays}. The same rate is also achieved by autoencoders. This result is in agreement with our theory since, in this case, Theorem \ref{theorem:fields} applies (cf. \ref{prop:darcy}).

Since linear methods, such as POD, can directly exploit the decay of the eigenvalues of the output field, and they require fewer data with respect to deep autoencoders, these approaches should be favored when dealing with problems of this type. It should be noted that this conclusion can already be derived from Theorem \ref{theorem:fields} without having to test multiple autoencoder modules.
Nevertheless, if provided with enough data, one may still choose to use autoencoders: at worst, they will match the same accuracy as linear methods (up to optimization errors).

The second case study, on the contrary, shows the opposite situation. This time, the eigenvalues of the output field have a tail that decays 25\% slower than that of the input field. Then, our theory suggests that a nonlinear method based on deep autoencoders can exploit this hidden regularity to provide better approximation capabilities. Indeed, this is what we observe in practice; see Figure \ref{fig:decays}. 
It should be noted that this result is particularly interesting if we consider that the setup for this problem only partially fulfills the hypothesis of Theorem \ref{theorem:fields}. Indeed, while it can be shown that the parameter-to-solution operator in \eqref{eq:burgersoperator} enjoys a form of $L^{\infty}\to L^{2}$ Lipschitz continuity, the probability law of the input field is non-Gaussian: see also our previous discussion in Remark \ref{remark:gaussian}. 

\begin{figure}
    \centering
    \includegraphics[width=0.495\textwidth]{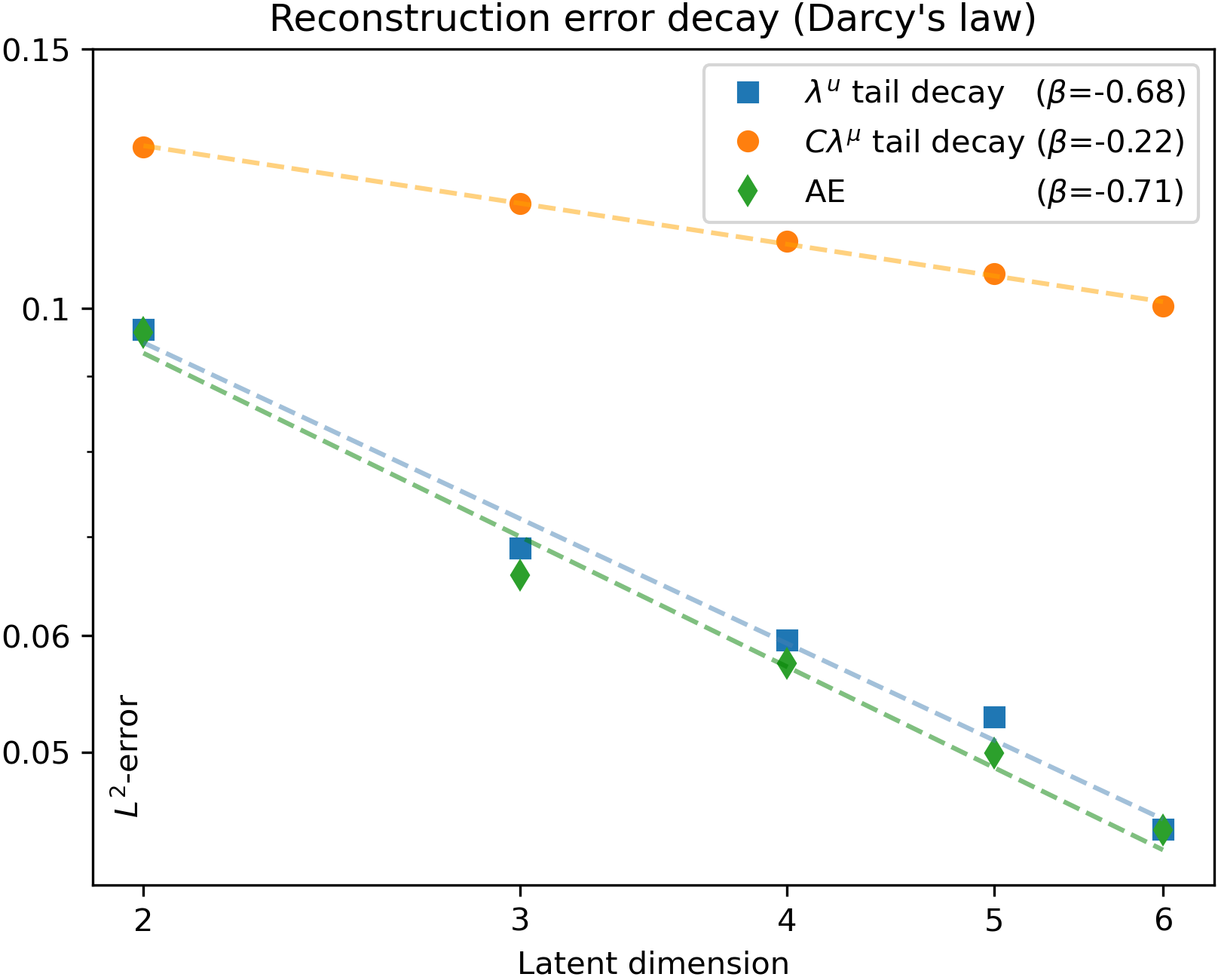}
    \includegraphics[width=0.495\textwidth]{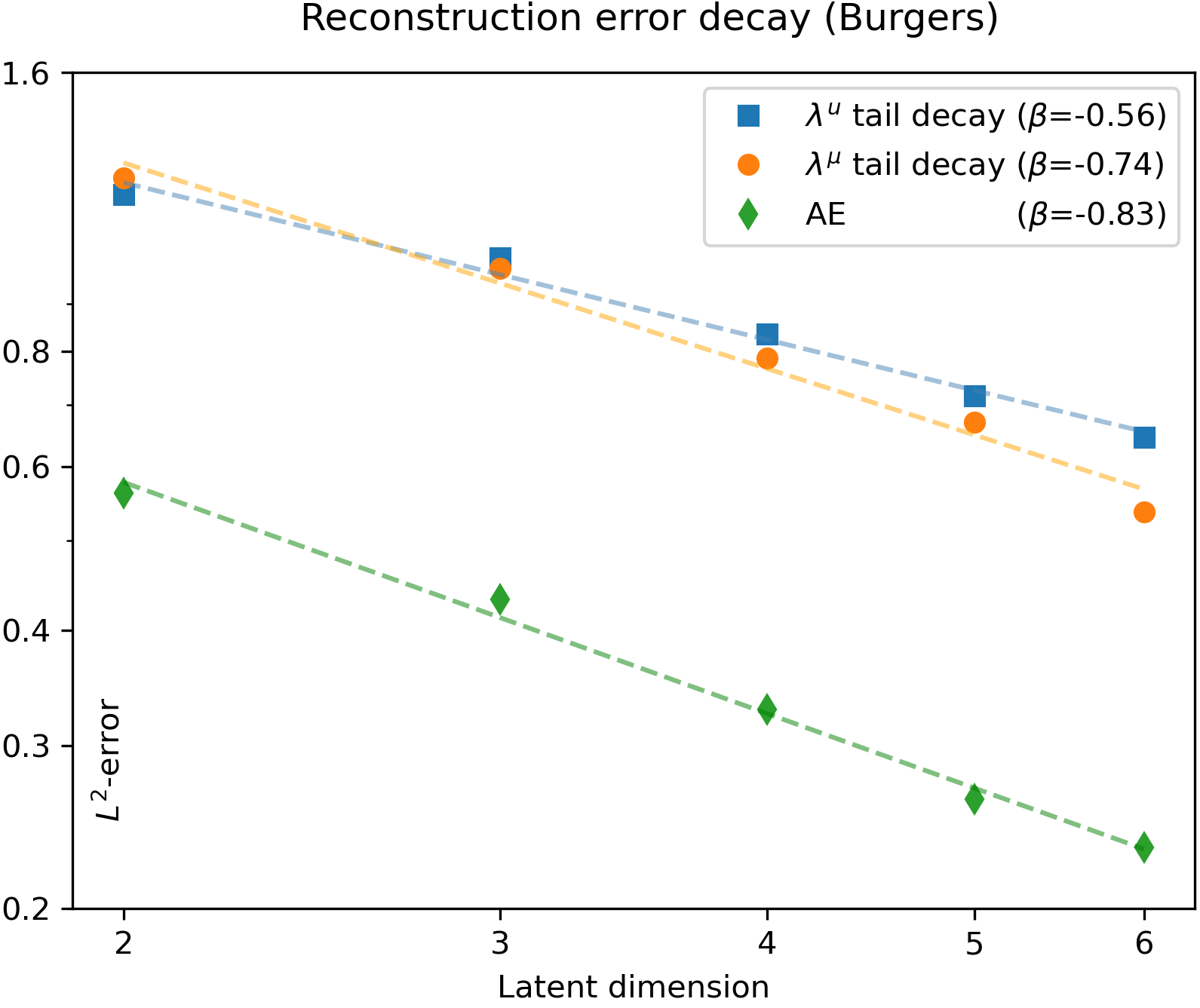}
    \caption{Decay of the reconstruction error for Darcy's law §\ref{subsec:diff} (left) and Burger's equation §\ref{subsec:burger}, in comparison with the tails of the eigenvalues of the input and output fields, respectively: cf. Equations \eqref{eq:testerror} and \eqref{eq:quantities}. In the Darcy flow example, the eigenvalues of the input field are scaled by a factor $C>0$ to improve readability. Dashed lines are obtained through least-squares in the loglog space. $\beta$ = rate of decay, computed as the slope of the dashed lines.}
    \label{fig:decays}
\end{figure}

In this concern, we also mention that although Theorem \ref{theorem:fields} would require the computation of an $L^{\infty}$-tail, that is,
$\left\|\sum_{i>n}\lambda_{i}^{\mu}\varphi_{i}^{2}\right\|_{\infty},$
it is sufficient to monitor the behavior of the $L^{2}$-tail, that is, $\sum_{i>n}\lambda_{i}^{\mu}$. This is because, in this case, the eigenfunctions of the input field are uniformly bounded in the $L^{\infty}$-norm, fact that we can easily appreciate from the plot in Figure \ref{fig:burgers-norms}: then, the two tails can be shown to decay at the same rate (see also our discussion in Remark \ref{remark:boundness}).
\\\\

\begin{figure}
    \centering
    \includegraphics[width=0.6\textwidth]{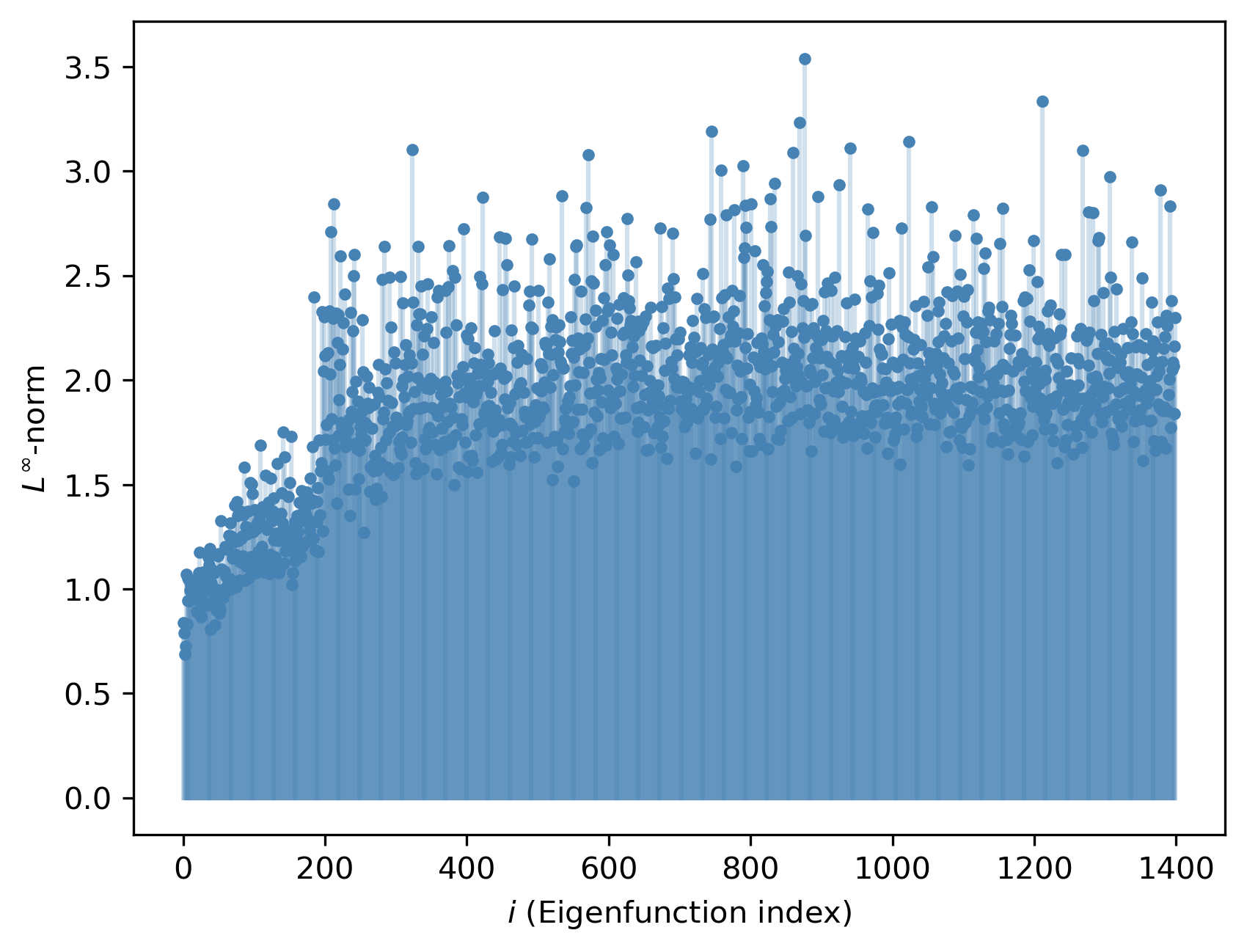}
    \caption{$L^{\infty}$-norms of the eigenfunctions in the KKL expansion of the input field $\mu$ for the Burger's equation §\ref{subsec:burger}: see Equations \eqref{eq:kl} and \eqref{eq:burgersmu}, respectively.}
    \label{fig:burgers-norms}
\end{figure}

For completeness, Table \ref{tab:performance} shows the overall performances attained by two DL-ROM surrogates for the problems at hand. While this analysis is not directly related to the theory developed in Section \ref{sec:theory}, it might still be of interest as it goes back to the global picture (that is, that of reduced order modeling). This time, we train all the architectures in the DL-ROM pipeline simultaneously following our initial discussion in Section \ref{subsec:roms}: see, in particular, Equation \eqref{eq:loss}. This means that the autoencoder module, $\Psi\circ\Psi'$, is trained together with all the remaining parts of the DL-ROM: consequently, although we are minimizing the reconstruction error made by the autoencoder, there are also other quantities driving the optimization (such as, e.g., the approximation error of the DL-ROM).

In the Darcy flow example, the DL-ROM reports an average $L^{2}$-error of 5.43\%. At the same time, its autoencoder module shows similar accuracy, with a reconstruction error of 4.94\%. As expected, this performance is also comparable with the one achieved by a linear approach such as POD: as we mentioned previously, in fact, we do not really need a nonlinear technique to reduce a problem of this type. However, things become quite different when we move to Burger's equation. In this case, the autoencoder module is almost twice as accurate as its linear counterpart, with an average $L^{2}$ error of 5.32\%. This fact is also reflected in the overall performance of the DL-ROM, which reports an average test error of 5.98\%. The interested reader can also find a few examples in Figure \ref{fig:burgerstest}, where we compare ground-truth simulations and DL-ROM outputs for new random realizations of the input field $\mu$.
\\\\
\renewcommand{\arraystretch}{1.5}
\begin{table}
    \centering
    \begin{tabular}{lllll}
    \hline
     \textbf{Problem} & \textbf{$n$} & \textbf{POD $L^{2}$-error} & \textbf{AE $L^{2}$-error} & \textbf{DL-ROM $L^{2}$-error} \\
     \hline\hline
     Darcy's law  §\ref{subsec:diff} & 16 & 4.46\% & 4.94\% & 5.43\%\\
     Burger's §\ref{subsec:burger} & 16 & 9.30\% & 5.32\% & 5.98\%\\
     \hline\\
    \end{tabular}
    \caption{Average test errors for the three case studies in Section \ref{sec:experiments}. Here, the latent dimension $n$ is fixed. POD = projection error, AE = reconstruction error, DL-ROM = model error.}
    \label{tab:performance}
\end{table}
Inspired by the results in Table \ref{tab:performance}, we conclude with a short digression on the interplay between reconstruction errors and approximation errors. In general, when it comes to DL-ROMs, the two quantities are only indirectly related. To better explain this fact, let us first consider the opposite case of POD-based ROMs. Ultimately, the POD projector $\mathbf{V}\in\mathbb{R}^{N_{h}\times n}$ operates as a linear autoencoder module, as
$$\mathbf{u}_{\mup}^{h}\approx\mathbf{V}\mathbf{V}^{T}\mathbf{u}_{\mup}^{h},$$
which allows one to represent each $\mathbf{u}_{\mup}^{h}$ with the corresponding set of projection coefficients $\mathbf{V}^{T}\mathbf{u}_{\mup}^{h}\in\mathbb{R}^{n}$. Then, any POD-based ROM, such as, e.g., POD-Galerkin \cite{quarteroni2015reduced}, POD-NN \cite{hesthaven2018non}, POD-DeepONet \cite{lu2022comprehensive} and POD-GPR \cite{guo2019data}, can be written abstractly as $$\mathbf{u}_{\mup}^{h}\approx\mathbf{V}\phi_{\text{POD}}(\mup),$$ where $\phi_{\text{POD}}:\mathbb{R}^{N_{h}}\to\mathbb{R}^{n}$ is some black-box procedure that maps parameters onto reduced coefficients: the latter can be a neural network model, as in POD-NN and POD-DeepONet, a Gaussian process approximator, as in POD-GPR, or a suitable numerical method solving the projected PDE, as in POD-Galerkin. Then, by orthogonality, the error of any such method can be bounded from below as
$$\|\mathbf{u}_{\mup}^{h}-\mathbf{V}\phi_{\text{POD}}(\mup)\|^{2}=\|\mathbf{u}_{\mup}^{h}-\mathbf{V}\mathbf{V}^{T}\mathbf{u}_{\mup}^{h}\|^{2}+\|\mathbf{V}^{T}\mathbf{u}_{\mup}^{h}-\phi_{\text{POD}}(\mup)\|^{2}\ge\|\mathbf{u}_{\mup}^{h}-\mathbf{V}\mathbf{V}^{T}\mathbf{u}_{\mup}^{h}\|^{2}.$$
That is: the approximation error of any such ROM is bounded from below by the projection error of the POD. Furthermore, this fact is not an intrinsic property of the POD basis; instead, it is a feature common to all projection-based methods.

In general, the same is not true for DL-ROMs. In fact, since all the architectures in the DL-ROM pipeline are optimized simultaneously, cf. Equation \eqref{eq:loss}, the approximation error of the DL-ROM might as well be smaller than the reconstruction error of the autoencoder, and vice versa. This can happen, for example, if it turns out to be simpler to describe latent variables using the input parameters (thus, through $\phi$) rather than using the output features (that is, through the encoder $\Psi'$). 

Still, when constructing a DL-ROM, a good rule of thumb is to prefer those architectures for which the two errors behave similarly: this, in fact, ensures a stronger connection between inputs and outputs, and it increases the interpretability of the DL-ROM as a whole.

\begin{figure}
    \centering
    \includegraphics[width=\textwidth]{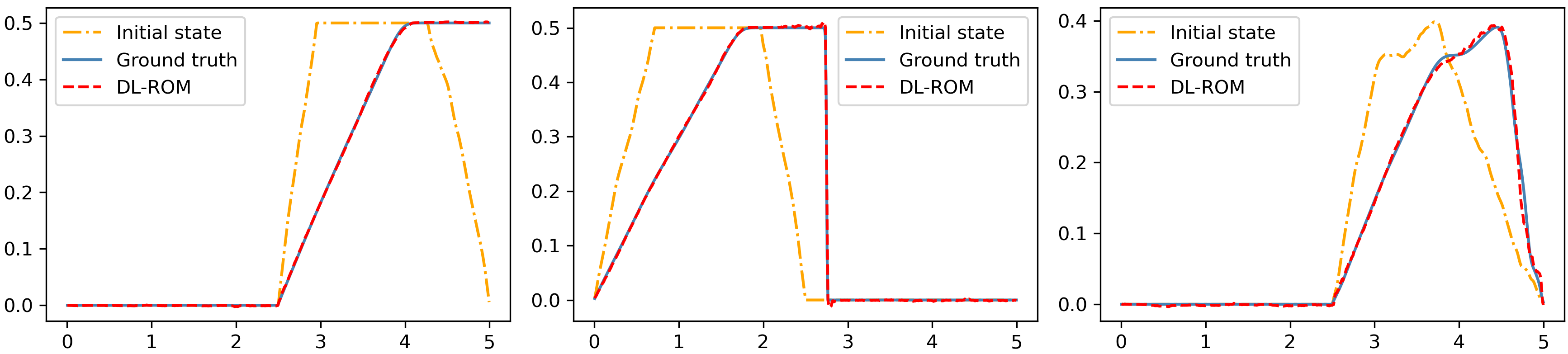}
    \caption{DL-ROM predictions for three unseen initial states of Burger's equation, §\ref{subsec:burger}.}
    \label{fig:burgerstest}
\end{figure}

\section{Conclusions}

This work addresses the practical problem of designing deep autoencoders, specifically in terms of their latent dimension, in the context of reduced order modeling for PDEs parametrized by random fields. This topic is of particular interest, since deep autoencoders have recently emerged as the pivotal element of Deep Learning based ROMs (DL-ROMs), a novel class of approaches that exploit manifold learning to approximate the solution manifold of a parametrized operator,
equipping them with the ability to tackle complex problems for which traditional methods may fall short.
The presented research is novel in that it addresses both theoretically and practically the case of stochastic PDEs, a scenario that is particularly relevant to applications involving, e.g., uncertainty quantification or Bayesian inversion, aspects hitherto unexplored in the DL-ROM literature.

Our main contribution is well summarized by Theorems \ref{theorem:finite}-\ref{theorem:fields}, in which we provide explicit error bounds that can aid domain practitioners in selecting the appropriate latent dimension for deep autoencoders. Our findings are highly interpretable, as they demonstrate the capacity of deep autoencoders to match or surpass the performances achieved by linear techniques. 
Numerical experiments agree with our theory, effectively highlighting the practical importance of our analysis in the intricate design of DL-ROMs. Future work may include the derivation of complementary results on the complexity of deep autoencoders and their training, possibly exploring a compromise between data-driven and physics-based approaches, in the same spirit of other recent works; see, e.g., \cite{de2022generic, franco2023approximation}.



\section*{Declarations}

\bmhead{Acknowledgements} The present research is part of the activities of project Dipartimento di Eccellenza 2023-2027, funded by MUR, and of project FAIR (Future Artificial Intelligence Research) project, funded by the NextGenerationEU program within the PNRR-PE-AI scheme (M4C2, Investment 1.3, Line on Artificial Intelligence).

\bmhead{Funding} This project has received funding under the ERA-NET ERA PerMed/ FRRB grant agreement No ERAPERMED2018-244, RADprecise - Personalized radiotherapy: incorporating cellular response to irradiation in personalized treatment planning to minimize radiation toxicity.

\bmhead{Conflict of interest} The authors declare that they have no competing interests. 

\bmhead{Data and code availability} Data and source code are available upon request: all related inquiries should be directed to the authors.


\begin{appendices}

\section{Technical proofs for Section \ref{sec:preliminaries}}
\label{sec:appendix:A}

\setcounter{lemma}{0}
\begin{lemma}
    Let $\mathbb{P}$ be a probability distribution over $\mathbb{R}^{p}$. Let $\rho:\mathbb{R}\to\mathbb{R}$ be a continuous map. Assume that either one of the following holds:
    \begin{itemize}
        \item [i)] $\rho$ is bounded and nonconstant;
        \item [ii)] $\rho$ is bounded from below and  $\rho(x)\to+\infty$ as $x\to+\infty$.
        \item [iii)] there exists some $a,b\in\mathbb{R}$ such that $x\to a\rho(x)+b\rho(-x)$ satisfies (ii).
    \end{itemize}
    Then, for every $\varepsilon>0$ and every measurable map $f:\mathbb{R}^{p}\to\mathbb{R}^{n}$
    with $\mathbb{E}|f(x)|<+\infty$, there exists $\Phi\in\mathscr{N}_{\rho}(\mathbb{R}^{p},\mathbb{R}^{n})$ such that
    \begin{equation*}
        \mathbb{E}_{\x\sim\mathbb{P}}|f(\x)-\Phi(\x)|<\varepsilon.
    \end{equation*}
\end{lemma}

\begin{proof}
    If (i) holds, then this is just a consequence of Hornik's Theorem \cite{hornik1991approximation}. Then, let (ii) hold. Since $\rho$ is bounded from below, 
    $\rho(\mathbb{R})\subset[A,+\infty)$ for some $A\in\mathbb{R}$.  
    Let
    \begin{equation*}
        \sigma(x):=\rho(-\rho(x)+\alpha),
    \end{equation*}
    where $\alpha$ is some parameter to be fixed later on. We have,    
    \begin{multline*}
    \sup_{x\in\mathbb{R}}|\sigma(x)|=\sup_{x\in\mathbb{R}}|\rho(-\rho(x)+\alpha)|=\sup_{y\in\rho(\mathbb{R})}|\rho(-y+\alpha)|\le\\\le\sup_{y\ge A}|\rho(-y+\alpha)|=\sup_{z\le \alpha-A}|\rho(z)|<+\infty,
    \end{multline*}
    in fact, by continuity, $\rho$ is bounded on all intervals of the form $(-\infty, c]$, with $c\in\mathbb{R}$. In particular, $\sigma$ is bounded. Furthermore, since $\rho(x)\to+\infty$ as $x\to+\infty$,
    \begin{equation*}
        \limsup_{x\to\infty}\sigma(x)\le\limsup_{z\to-\infty}\rho(z)=c_{0}<+\infty.
    \end{equation*}
    Let now $\alpha$ be such that
    \begin{equation*}
        \sigma(0)=\rho(\alpha-\rho(0))>c_{0}.
    \end{equation*}
    Then, $\sigma$ is guaranteed to be both bounded and nonconstant. However, because of the way we defined $\sigma$, we also have
    \begin{equation*}
        \mathscr{N}_{\sigma}(\mathbb{R}^{p},\mathbb{R}^{n})\subseteq \mathscr{N}_{\rho}(\mathbb{R}^{p},\mathbb{R}^{n}),
    \end{equation*}
    and the conclusion now follows by (i). Finally, assume that (iii) holds and let
    \begin{equation*}
        \tilde{\rho}(x):=a\rho(x)+b\rho(-x).
    \end{equation*}
    As before, we have
        $\mathscr{N}_{\tilde{\rho}}(\mathbb{R}^{p},\mathbb{R}^{n})\subseteq \mathscr{N}_{\rho}(\mathbb{R}^{p},\mathbb{R}^{n}),$
    and the conclusion follows by (ii).    
\end{proof}

\setcounter{proposition}{0}
\begin{proposition}
    Let $\operator:(W,\|\cdot\|_{W})\to(V,\|\cdot\|_{V})$ be an operator between two normed spaces. We have the following,
    \begin{itemize}
        \item [i)] $\|\partial\operator\|_{V}(w)<+\infty$ for all $w\in K\subset W$ $\iff$ $\operator$ is locally Lipschitz over $K$;\\
        
        \item [ii)] if $K\subset W$ is compact and $\|\partial\operator\|_{V}(w)<+\infty$ for all $w\in K$, then 
        $\mathcal{G}$ is Lipschitz over $K$;\\

        \item [iii)] if $C\subseteq W$ is convex, then 
        $$L_{C}:=\sup_{w\in C}\|\partial\operator\|_{V}(w)<+\infty$$
        if and only if $\operator$ is $L_{C}$-Lipschitz over $C$.\\

        \item [iv)] if $\operator$ is Fréchet differentiable at $w\in W$, then $\|\partial\operator\|_{V}(w)$ coincides with the operator norm of the Fréchet derivative of $\operator$ at $w.$\\

        \item [v)] given any $\mathcal{F}:(V,\|\cdot\|)_{V}\to(Y,\|\cdot\|_{Y})$, one has the chain-rule inequality
        $$\|\partial(\mathcal{F}\circ \operator)\|_{Y}(w)\le\|\partial \mathcal{F}\|_{Y}(\operator(w))\cdot\|\partial\operator\|_{V}(w),$$
        for all $w\in W.$
    \end{itemize}
\end{proposition}

\begin{proof}
Since (i) is trivial, we skip its proof.
    \begin{itemize}
        \item [ii)] Let $K\subset W$ be a compact subset. Seeking contraddiction, let us assume that $\operator$ is not Lipschitz continuous over $K$. Then, there exists two sequences $\{w_{n}\}_{n},$ and $\{v_{n}\}_{n}$ such that
        \begin{equation}
            \label{eq:absurd}
            \frac{\|\operator(w_{n}) - \operator(v_{n}) \|_{V}}{\|w_{n}-v_{n}\|_{W}}\to+\infty
        \end{equation}
        as $n\to+\infty$. Since $\operator$ is locally Lipschitz, it also continuous, and thus bounded over $K$: therefore, the above implies $\|w_{n}-v_{n}\|_{W}\to0$. At the same time, by compactness, there exists a subsequence $\{w_{n_{k}}\}_{k}\subseteq\{w_{n}\}_{n}$ and an element $w\in K$ such that $w_{n_{k}}\to w$. In particular, we have $w_{n_{k}},v_{n_{k}}\to w$. But then  \eqref{eq:absurd} would yield $\|\partial\operator\|_{V}(w)=+\infty$, absurd.\\
        
        \item [iii)] Let $C$ be convex and assume that $L_{C}:=\sup_{w\in C}\|\partial\operator(w)\|_{V}<+\infty$. Given any two points $v,v'\in C$, let $K$ be the segment between the two. Since $K$ is compact, it follows from (ii) that $\operator$ is Lipschitz continuous over $K$. Furthermore, since $K\subseteq C$,
        $$\sup_{w\in K}\|\partial\operator(w)\|_{V}\le L_{C}$$
        Exploiting the very definition of limit supremum, for all $w\in K$, let $B(w, r_{w})$ be a ball of radius $r_{w}>0$ such that
        $$\left|\frac{\|\operator(w+h) - \operator(w) \|_{V}}{\|h\|_{W}}-\|\partial\operator\|_{V}(w)\right|<\varepsilon\quad\forall h\in B(w, r_{w}).$$
        Then, up to rewriting the above, for all $w\in K$ we have
        $$h\in B(w, r_{w})\implies\|\operator(w+h) - \operator(w) \|_{V}\le (L_{C}+\varepsilon)\|h\|_{W}.$$
        We now note that, since $K\subset\cup_{w\in K}B(w,r_{w})$ and $K$ is compact, there exists a finite sequence of points $w_{1},\dots,w_{n}\in K$, and finite sequence of radii $r_{1},\dots,r_{k}>0$, such that $K\subset\cup_{i=1}^{n}B(w_{i},r_{i})$. Furthermore, upto to removing some of the balls, since $K$ is actually a segment,  we can sort the subcover so that
        $$K\cap B(w_{i},r_{i})\cap B(w_{i+1},r_{i+1})\neq\emptyset\quad\forall i=1,\dots,n-1.$$
        For each $i=1,\dots,n-1$, let $w^{*}_{i}\in K\cap B(w_{i},r_{i})\cap B(w_{i+1},r_{i+1}).$ We note that, since $S_{v,v'}$ is a straightline, we have $$\|v-v'\|_{W}=\|v-w^{*}_{1}\|_{W}+\dots+\|w^{*}_{n-1}-v'\|_{W}.$$
        In particular,
        \begin{equation*}
        \begin{aligned}
        \|\operator(v)-\operator(v')\|_{W}&\le\
            \\&\le\|\operator(v)-\operator(w^{*}_{1})\|_{W}+\dots+\|\operator(w^{*}_{n})-\operator(v')\|_{W}\le\\
            &\le(L_{C}+\varepsilon)\left(\|v-w^{*}_{1}\|_{W}+\dots+\|w^{*}_{n-1}-v'\|_{W}\right)=\\
            &=(L_{C}+\varepsilon)\|v-v'\|_{W}.
            \end{aligned}
        \end{equation*}
        Since $v,v'\in C$ and $\varepsilon>0$ were arbitrary, this concludes the proof (the other implication, $"\impliedby"$, is trivial and left to the reader).\\

        \item[iv)] Assume that $\operator$ is Fréchet differentiable at $w\in W$, and let $\delta\operator[w]:W\to V$ be the linear operator representing its derivative. Let $\{h_{n}\}_{n}$ be a sequence of unitary increments, $\|h_{n}\|_{W}=1$, such that $$\|\delta\operator[w](h_{n})\|_{V}\longrightarrow\|\delta\operator[w]\|_{W,V}$$
        where $\|\cdot\|_{W,V}$ is the operator norm for linear maps going from $W$ to $V$. Then,
        \begin{multline*}
            \hspace{2.5em}\|\partial\operator\|_{W}(w)\ge \lim_{n\to+\infty}\frac{\|\operator(w+n^{-1}h_{n})-\operator(w)\|_{V}}{\|n^{-1}h_{n}\|_{W}}=\vspace{1em}\\\hspace{3em}=\lim_{n\to+\infty}\frac{\|\delta\operator[w](n^{-1}h_{n})\|_{V}}{\|n^{-1}h_{n}\|_{W}}=\lim_{n\to+\infty}\frac{\|\delta\operator[w](h_{n})\|_{V}}{\|h_{n}\|_{W}}=\\=\|\delta\operator[w]\|_{W,V}.
        \end{multline*}
        Conversely, let $\{\tilde{h}_{n}\}_{n}$ be such that $\tilde{h}_{n}\to0$ and
        $$\|\partial\operator\|_{W}(w)= \lim_{n\to+\infty}\frac{\|\operator(w+\tilde{h}_{n})-\operator(w)\|_{V}}{\|\tilde{h}_{n}\|_{W}}.$$
        Then, 
        $$\|\partial\operator\|_{W}(w)= \ldots=\lim_{n\to+\infty}\frac{\|\delta\operator[w](\tilde{h}_{n})\|_{V}}{\|\tilde{h}_{n}\|_{W}}\le\|\delta\operator[w]\|_{W,V},$$
        thus $\|\partial\operator\|_{W}(w)=\|\delta\operator[w]\|_{W,V}$ as claimed.\\

    \item [v)] Let $w\in W$. If either $\|\partial \mathcal{F}\|_{Y}(\operator(w))=+\infty$ or $\|\partial\operator\|_{V}(w)=+\infty$, then the claim is trivially true. Thus, we assume both quantities to be finite. Fix any $\varepsilon>0.$ Then, by definition of limit supremum, there exists $r_{\varepsilon}>0$ such that $\mathcal{F}$ is $(\|\partial\mathcal{F}\|_{Y}(\operator(w))+\varepsilon)$-Lipschitz over the open ball of radius $r_{\varepsilon}$ centered at $\operator(w)$. It follows that
    \begin{multline*}
    \lim_{r\to0^{+}}\;\sup_{0< \|h\|_{W} \le r} \frac{\|\mathcal{F}(\operator(w+h)) - \mathcal{F}(\operator(w)) \|_{V}}{\|h\|_{W}}\le\\\le
    \lim_{r\to0^{+}}\;\sup_{0< \|h\|_{W} \le r} \frac{(\|\partial\mathcal{F}\|_{Y}(\operator(w))+\varepsilon)\|\operator(w+h) - \operator(w)\|_{V}}{\|h\|_{W}}=\\=(\|\partial\mathcal{F}\|_{Y}(\operator(w))+\varepsilon)\|\partial\operator\|_{V}(w),
    \end{multline*}
    as $r\le r_{\varepsilon}$ definitely. Since $\varepsilon$ was arbitrary, the conclusion follows.    
    \end{itemize}
\end{proof}

\begin{proposition}
    Let $\Omega\subset\mathbb{R}^{d}$ be a bounded domain with Lipschitz boundary and let $f\in H^{-1}(\Omega)$ be given. For any $\sigma\in L^{\infty}(\Omega)$, let $u=u_{\sigma}$ be the solution to the following boundary value problem,
    $$\begin{cases}
    -\nabla\cdot(e^{\sigma}\nabla u)=f &  \textnormal{in}\;\Omega\\
            u=0 & \textnormal{on}\;\partial\Omega
    \end{cases}.$$
    Let $\operator: L^{\infty}(\Omega)\to L^{2}(\Omega)$ be the operator that maps $\sigma$ to $u$. Then, for all $\sigma,\sigma'\in L^{\infty}(\Omega)$
  \begin{equation}
    \label{eq:darcyboundAP}
      \|\operator(\sigma)-\operator(\sigma')\|_{L^{2}(\Omega)}\le C\|f\|_{H^{-1}(\Omega)} e^{3\|\sigma\|_{L^{\infty}(\Omega)}+3\|\sigma'\|_{L^{\infty}(\Omega)}}\|\sigma-\sigma'\|_{L^{\infty}(\Omega)}
  \end{equation}
    and, in particular, \begin{equation*}
    \|\partial\operator(\sigma)\|_{L^{2}(\Omega)}\le C\|f\|_{H^{-1}(\Omega)}e^{6\|\sigma\|_{L^{\infty}(\Omega)}},\end{equation*}
    where $C=C(\Omega)$ is some positive constant.
\end{proposition}
\begin{proof}
    By classical energy estimates on elliptic PDEs, see e.g. Lemma C.1. in \cite{franco2022deep}, we have
    \begin{multline*}
        \|\operator(\sigma)-\operator(\sigma')\|_{L^{2}(\Omega)}\le C\|\operator(\sigma)-\operator(\sigma')\|_{H^{1}_{0}(\Omega)}\le\\\le
        C\left[\min_{\x\in\Omega}e^{\sigma}(\x)\right]^{-1}\left[{\min_{\x\in\Omega}e^{\sigma'}(\x)}\right]^{-1}\|f\|_{H^{-1}(\Omega)}\|e^{\sigma}-e^{\sigma'}\|_{L^{\infty}(\Omega)}.
    \end{multline*}
    The latter is bounded by
    $$
        Ce^{\|\sigma\|_{L^{\infty}(\Omega)}+\|\sigma'\|_{L^{\infty}(\Omega)} }\|f\|_{H^{-1}(\Omega)}\|e^{\sigma'}\|_{L^{\infty}(\Omega)}\|e^{\sigma-\sigma'}-1\|_{L^{\infty}(\Omega)}.
    $$
    We now note that, for all $a\in\mathbb{R}$ one has $|e^{a}-1|\le|a|e^{|a|}$. Also, $\|e^{\sigma'}\|_{L^{\infty}(\Omega)}=e^{\|{\sigma'}\|_{L^{\infty}(\Omega)}}$ by monotinicity of the exponential. It follows that
    \begin{multline*}
        \|\operator(\sigma)-\operator(\sigma')\|_{L^{2}(\Omega)}\le \\
        \le
        Ce^{\|\sigma\|_{L^{\infty}(\Omega)}+2\|\sigma'\|_{L^{\infty}(\Omega)} }\|f\|_{H^{-1}(\Omega)}\|e^{\sigma-\sigma'}\|_{L^{\infty}(\Omega)}\|\sigma-\sigma'\|_{L^{\infty}(\Omega)}.
    \end{multline*}
    Since $\|e^{\sigma-\sigma'}\|_{L^{\infty}(\Omega)}\le \|e^{\sigma}\|_{L^{\infty}(\Omega)}\|e^{\sigma'}\|_{L^{\infty}(\Omega)}=e^{\|\sigma\|_{L^{\infty}(\Omega)}+\|\sigma'\|_{L^{\infty}(\Omega)} }$, \eqref{eq:darcyboundAP} easily follows.
\end{proof}

\begin{lemma}
    Let $\Omega\subset\mathbb{R}^{d}$ be pre-compact, and let $Z$ be a mean zero Gaussian random field defined over $\Omega$. Assume that, for some $0<\alpha\le 1$, the covariance kernel of the process, $$\cov:\Omega\times\Omega\to\mathbb{R},$$
    \begin{equation*}
        \cov(\x,\y):=\mathbb{E}\left[Z(\x)Z(\y)\right],
    \end{equation*}
    is $\alpha$-Hölder continuous, with Hölder constant $L>0$. Then, $Z$ is sample-continuous, that is $\mathbb{P}(Z\in\mathcal{C}(\Omega))=1.$ Furthermore, for
    $\sigma^{2}:=\max_{\x\in\Omega}\cov(\x,\x),$
    one has
    \begin{equation}
        \label{eq:supintAP}
        \mathbb{E}^{1/2}\|Z\|_{L^{\infty}(\Omega)}^{2} \le c_{1}\sigma\left(1+\sqrt{\log^{+}(1/\sigma)}\right)\quad\textnormal{and}\quad
        \mathbb{E}\left[e^{\beta\|Z\|_{L^{\infty}(\Omega)}}\right]=c_{2}<+\infty,
    \end{equation}
    for all $\beta>0$, where $c_{1}=c_{1}(d,L,\alpha,\Omega)$ and $c_{2}=c_{2}(d,L,\alpha,\sigma,\beta,\Omega)$ are constants that depend continuously on their parameters (domain excluded). Here, $$\log^{+}(a):=\max\{\log a,0\}.$$
\end{lemma}
\begin{proof} If $\sigma=0$, the proof is trivial; thus, we let $\sigma>0$. Let
    $$d(\x,\y):=\mathbb{E}^{1/2}|Z(\x)-Z(\y)|^{2}=\sqrt{\cov(\x,\x)-2\cov(\x,\y)+\cov(\y,\y)},$$
    be a metric over $\Omega$ induced by the Gaussian process $Z$. We note that the Hölder continuity of the covariance kernel implies
    $$d(\x,\y)\le\sqrt{\cov(\x,\x)-\cov(\x,\y)}+\sqrt{\cov(\y,\y)-\cov(\x,\y)}\le \sqrt{2L|\x-\y|^{\alpha}}.$$
    In particular, the balls $B_{d}(\x,\;\varepsilon):=\{\y\in\Omega:\;d(\x,\y)<\varepsilon\}$ induced by the metric $d$ satisfy
    \begin{equation}
        \label{eq:inclusion}
        B_{d}(\x,\;\varepsilon)\supseteq B\left(\x,\;2^{-1/\alpha}L^{-1/\alpha}\varepsilon^{2/\alpha}\right),
    \end{equation}
    where $B(\x,\;\epsilon)$ is the Euclidean ball of radius $\epsilon$ centered at $\x$. In fact,
    $$|\x-\y|\le 2^{-1/\alpha}L^{-1/\alpha}\varepsilon^{2/\alpha}\implies d(\x,\y)\le \sqrt{2L(2L)^{-1}\varepsilon^{2}}=\varepsilon.$$
    Now, for any $\varepsilon>0$, let $N_{d}(\varepsilon)$ be the minimum number of $d$-balls of radius $\varepsilon$ that are required for covering $\Omega$. Similarly, let $N(\varepsilon)$ be the covering number associated to the Euclidean metric. It is straightforward to see that \eqref{eq:inclusion} implies
    \begin{equation}
    \label{eq:covering}1\ge N_{d}(\varepsilon)\le N(2^{-1/\alpha}L^{-1/\alpha}\varepsilon^{2/\alpha})\le \max\left\{\left(\frac{C\varepsilon^{2}}{2L}\right)^{-d/\alpha},1\right\},\end{equation}
    where $C=C(d,\alpha,\Omega)>0$ is an absolute constant. In fact, $$N(\epsilon)\le\max\{3^{d}\text{diam}(\Omega)^{d}\epsilon^{-d},1\},$$ where $\text{diam}(\Omega)$ is the domain diameter under the Euclidean metric. In particular, we have $N_{d}(\varepsilon)<+\infty$ for all $\varepsilon>0$, meaning that $\Omega$ is $d$-compact. Furthermore, \eqref{eq:covering} also implies that $N_{d}(\varepsilon)=1$ for all $\varepsilon\ge\sqrt{C/2L}.$
    \\\\
    We now note that for all $\x,\y\in\Omega$ one has
    $$d(\x,\y)\le\mathbb{E}^{1/2}|Z(\x)|^{2}+\mathbb{E}^{1/2}|Z(\y)|^{2}\le2\sigma.$$
    In particular, the $d$-diameter of $\Omega$ is bounded by $2\sigma$. Then, Theorem 1.3.3 in \cite{adler2007random} implies
    $$\mathbb{E}\left[\sup_{\x\in\Omega}Z(\x)\right]\le K\int_{0}^{\sigma}\log^{1/2}N_{d}(\varepsilon)d\varepsilon,$$
    where $K$ is a universal constant.
    Then, the above together with \eqref{eq:covering} yields
    $$\mathbb{E}\left[\sup_{\x\in\Omega}Z(\x)\right]\le K\int_{0}^{\min\{\sigma,\sqrt{C/2L}\}}\sqrt{-\frac{d}{\alpha}\log \left(\frac{C}{2L}\varepsilon^{2}\right)}d\varepsilon.$$
    By operating the change of variables $\epsilon:=\sqrt{C/2L}\varepsilon$, we may rewrite the previous as
    $$\mathbb{E}\left[\sup_{\x\in\Omega}Z(\x)\right]\le K\sqrt{\frac{4Ld}{C'\alpha}}\int_{0}^{\min\{\sqrt{C/2L}\sigma,1\}}\sqrt{\log(1/\epsilon)}d\epsilon.$$
    At this point, it useful to note that for any $0<a<1$ one has
    $$\int_{0}^{a}\sqrt{\log\epsilon}d\epsilon=a\sqrt{\log(1/a)}+\frac{\sqrt{\pi}}{2}(1-\text{erf}(\sqrt{\log(1/a)}))\le a\left(\frac{\sqrt{\pi}}{2}+\sqrt{\log(1/a)}\right)$$
    where $\text{erf}$ is the error function, which is known to satisfy $1-\text{erf}(z)\le e^{-z^{2}}.$ 
    It follows immediately that
    \begin{multline}\mathbb{E}\left[\sup_{\x\in\Omega}Z(\x)\right]\le K\sqrt{\frac{2d}{\alpha}}\sigma\left(\frac{\sqrt{\pi}}{2}+\sqrt{-\log\min\left\{\sqrt{\frac{C}{2L}\sigma},1\right\}}\right)\le\\\le C''\sigma\left(1+\sqrt{\log^{+}(1/\sigma)}\right),\end{multline}
    where $C'=C'(d,\alpha,\Omega,L)$ grows less than logarithmically in $L$.
    \\\\
    The first consequence of this fact, is that the paths of $Z$ are almost-surely uniformly continuous (c.f. Theorem 1.3.5 in \cite{adler2007random}). Additionally, since the Gaussian process $\tilde{Z}:=-Z$ satisfies the same upper bound, it is straightforward to conclude that
    $$\mathbb{E}\|Z\|_{L^{\infty}(\Omega)}\le \tilde{C}\sigma\left(1+\sqrt{\log^{+}(1/\sigma)}\right),$$
    where $\tilde{C}:=2C'.$ 
    We now recall the celebrated Borell-TIS inequality, see Theorem 2.1.1 in \cite{adler2007random},
    $$\mathbb{P}(\|Z\|_{L^{\infty}(\Omega)}-\mathbb{E}\|Z\|_{L^{\infty}(\Omega)}>z)\le \exp\left(-\frac{z^{2}}{2\sigma^{2}}\right),$$
    from which it is straightforward to prove that
    $\text{Var}\;\|Z\|_{L^{\infty}(\Omega)}\le4\sigma^{2}.$ Then,
    $$\mathbb{E}^{1/2}\|Z\|_{L^{\infty}(\Omega)}^{2}\le\sqrt{\text{Var}\;\|Z\|_{L^{\infty}(\Omega)}}+\mathbb{E}\|Z\|_{L^{\infty}(\Omega)}\le c_{1}\sigma\left(1+\sqrt{\log^{+}(1/\sigma)}\right),$$
    as claimed. Finally, the statement concerning $\mathbb{E}\left[e^{\beta\|Z\|_{L^{\infty}(\Omega)}}\right]$ follows directly from Theorem 2.1.2 in \cite{adler2007random}: see also \eqref{eq:equiv} and the reasoning explained thereby.    
    \end{proof}

\begin{lemma}
    Let $\Omega\subset\mathbb{R}^{d}$ be a compact subset and let $Z$ be a mean zero Gaussian random field defined over $\Omega$. Assume that the covariance kernel of $Z$, $\cov$, is Lipschitz continuous. Then, there exists a nonincreasing summable sequence $\lambda_{1}\ge\lambda_{2}\ge\dots\ge0$ and a sequence of Lipschitz continuous maps, $\{\varphi_{i}\}_{i=1}^{+\infty}$, forming an orthonormal basis of $L^{2}(\Omega)$, such that
    \begin{equation}
    \label{eq:mercerAP}
    \cov(\x,\y)=\sum_{i=1}^{+\infty}\lambda_{i}\varphi_{i}(\x)\varphi_{i}(\y)
    \end{equation}
    for all $\x,\y\in\Omega$. Furthermore, there exists a sequence of independent standard normal random variables, $\{\eta_{i}\}_{i=1}^{+\infty}$, such that
    \begin{equation} 
    \label{eq:karhunenAP} Z=\sum_{i=1}^{+\infty}\sqrt{\lambda_{i}}\eta_{i}\varphi_{i}\end{equation}
    almost surely. Finally, the truncated kernels,
    $$\cov_{p,q}(\x,\y):=\sum_{i=p}^{q}\lambda_{i}\varphi_{i}(\x)\varphi_{i}(\y),$$
    defined for varying $1\le p\le q\le+\infty$, 
    \begin{itemize}
        \item [i)] converge uniformly as $p,q\to+\infty$;
        \item [ii)] are all 1/2-Hölder continuous, with a common Hölder constant.
    \end{itemize}
\end{lemma}

\begin{proof}
    Without loss of generality, we shall assume that $\lambda_{i}>0$ for all $i\in\mathbb{N}$. The series expansion in \eqref{eq:mercerAP}, and the uniform convergence claimed in (i), are a consequence of Mercer's Theorem \cite{mercer1909xvi}. There, the basis functions $\varphi_{i}$ are obtained by  solving the eigenvalue problem below
\begin{equation}
    \label{eq:eigen}
    \lambda_{i}\varphi_{i}(\x)=\int_{\Omega}\cov(\x,\y)\varphi(\y)d\y,
\end{equation}
    where the equality holds for almost every $\x\in\Omega$, and, similarly, \eqref{eq:mercerAP} is shown to hold almost everywhere. However, the right-hand-side of \eqref{eq:eigen} is easily shown to be Lipschitz continuous in $\x$, since
\begin{multline}
    \left|\int_{\Omega}\cov(\x,\y)\varphi(\y)d\y-\int_{\Omega}\cov(\x',\y)\varphi(\y)d\y\right|\le\\\le L|\x-\x'|\int_{\Omega}|\varphi_{i}(\y)|d\y\le L|\x-\x'||\Omega|^{1/2},
\end{multline}
    where $L$ is the Lipschitz constant of $\cov$. Thus, without loss of generality, we may pick the $\varphi_{i}$ to be Lipschitz continuous. Since the series in \eqref{eq:mercerAP} converges uniformly, both the left-hand-side and the right-hand-side of \eqref{eq:mercerAP} are continuous: as they coincide a.e. in $\Omega$, they must be equal everywhere. To conclude, we shall now prove (ii), as \eqref{eq:karhunenAP} is just the well-known statement of the Kosambi-Karhunen-Loeve Theorem. Pick any $0\le p\le q\le+\infty$. For $\x,\x',\y,\y'\in\Omega$ we have
    \begin{multline}
        \label{eq:diseg}
        \left|\sum_{i=p}^{q}\lambda_{i}\varphi_{i}(\x)\varphi_{i}(\y)-\sum_{i=p}^{q}\lambda_{i}\varphi_{i}(\x')\varphi_{i}(\y')\right|\le \sum_{i=p}^{q}\lambda_{i}|\varphi_{i}(\x)\varphi_{i}(\y)-\varphi_{i}(\x')\varphi_{i}(\y')|\le\\
        \le
        \sum_{i=p}^{q}\lambda_{i}|\varphi_{i}(\x)-\varphi_{i}(\x')||\varphi_{i}(\y)|+
        \sum_{i=p}^{q}\lambda_{i}|\varphi_{i}(\y)-\varphi_{i}(\y')||\varphi_{i}(\x')|.
    \end{multline}
Applying the Cauchy-Schwartz inequality, and using $\lambda_{i}=\sqrt{\lambda_{i}}\sqrt{\lambda_{i}}$, allows us to continue inequality \eqref{eq:diseg} as
\begin{multline}
    \dots\le\sqrt{\sum_{i=p}^{q}\lambda_{i}|\varphi_{i}(\x)-\varphi_{i}(\x')|^{2}}\sqrt{\sum_{i=p}^{q}\lambda_{i}\varphi_{i}(\y)^{2}}\\+\sqrt{\sum_{i=p}^{q}\lambda_{i}|\varphi_{i}(\y)-\varphi_{i}(\y')|^{2}}\sqrt{\sum_{i=p}^{q}\lambda_{i}\varphi_{i}(\x')^{2}}.
\end{multline}
Since all the sums involved concern positive values, we may further bound the above by letting $p=0$ and $q=+\infty$. Then, thanks to \eqref{eq:mercerAP}, we get
\begin{multline}
        \left|\sum_{i=p}^{q}\lambda_{i}\varphi_{i}(\x)\varphi_{i}(\y)-\sum_{i=p}^{q}\lambda_{i}\varphi_{i}(\x')\varphi_{i}(\y')\right|\le\vspace{0.5em}\\
        \le
        \sqrt{\cov(\x,\x)-2\cov(\x,\x')+\cov(\x',\x')}\sqrt{\cov(\y,\y)}\vspace{0.5em}\\+
        \sqrt{\cov(\y,\y)-2\cov(\y,\y')+\cov(\y',\y')}\sqrt{\cov(\x',\x')}\le\vspace{0.5em}\\
        \le
        \sqrt{2LM|\x-\x'|}+\sqrt{2LM|\y-\y'|},
    \end{multline}
    where $M:=\max_{\x\in\Omega}\cov(\x,\x)$. This shows that the truncated kernel is 1/2-Hölder continuous with Hölder coefficient bounded by $\sqrt{2LM}$. As the latter is independent on both $p$ and $q$, this concludes the proof.
\end{proof}

\begin{lemma}
    Let $\Omega\subset\mathbb{R}^{d}$ be a compact subset and let $Z$ be a mean zero Gaussian random field defined over $\Omega$. Assume that the covariance kernel of $Z$, $\cov$, is square-integrable over $\Omega\times\Omega$. Then, there exists a nonincreasing summable sequence $\lambda_{1}\ge\lambda_{2}\ge\dots\ge0$ and an orthonormal basis of $L^{2}(\Omega)$, $\{\varphi_{i}\}_{i=1}^{+\infty}$,  such that
    \begin{equation*}
    \cov(\x,\y)=\sum_{i=1}^{+\infty}\lambda_{i}\varphi_{i}(\x)\varphi_{i}(\y)
    \end{equation*}
    for almost every $(\x,\y)\in\Omega\times\Omega$. Furthermore, there exists a sequence of independent standard normal random variables, $\{\eta_{i}\}_{i=1}^{+\infty}$, such that
    \begin{equation*}  Z=\sum_{i=1}^{+\infty}\sqrt{\lambda_{i}}\eta_{i}\varphi_{i}\end{equation*}
    almost surely. Finally, the $L^{2}$-norm of the process is exponentially integrable, i.e.
    \begin{equation}
        \label{eq:expintAP}
        \mathbb{E}\left[e^{\beta\|Z\|_{L^{2}(\Omega)}}\right]<+\infty
    \end{equation}
    for all $\beta>0.$
\end{lemma}

\begin{proof}
    We only need to prove \eqref{eq:expintAP}, as the rest of the Lemma simply follows from the Kosambi-Karhunen-Loeve Theorem. To this end, we first note that for any random variable $X$ one has
    \begin{equation}
        \label{eq:equiv}
        \exists\epsilon>0\text{ such that } \expe\left[e^{\epsilon X^{2}}\right]<+\infty\implies \expe\left[e^{\beta |X|}\right]<+\infty\;\forall\beta>0.
    \end{equation}
    In fact, given any two positive numbers $\epsilon$ and $\beta$, there exists some $M>0$ such that $M+\epsilon x^{2}>\beta|x|$ for all $x\in\mathbb{R}$. In light of this, let $\epsilon>0$ be a parameter, whose value shall be chosen later on. By orthonormality, we have
    \begin{equation*}
        \|Z\|^{2}_{L^{2}(\Omega)}=\sum_{i=1}^{+\infty}\lambda_{i}\eta_{i}^{2}.
    \end{equation*}
    As the $\eta_{i}$'s are independent, it follows that
    \begin{equation}
    \label{eq:prod}
        \expe\left[e^{\epsilon \|Z\|^{2}_{L^{2}(\Omega)}}\right] = 
        \expe\left[e^{\epsilon \sum_{i=1}^{+\infty}\lambda_{i}\eta_{i}^{2}}\right]=\prod_{i=1}^{+\infty}\mathbb{E}\left[e^{\epsilon\lambda_{i}\eta_{i}^{2}}\right].
    \end{equation}
    For each index $i$, if $\epsilon\lambda_{i}-1/2 < 0$, we have
    \begin{equation*}
\mathbb{E}\left[e^{\epsilon\lambda_{i}\eta_{i}^{2}}\right] = \frac{1}{\sqrt{2\pi}}\int_{-\infty}^{+\infty}e^{\epsilon\lambda_{i} z^{2}}e^{-z^{2}/2}dz = \sqrt{\frac{1}{1-2\epsilon\lambda_{i}}}.
    \end{equation*}
    We thus choose $\epsilon<1/2\lambda_{1}$, so that, by monotonicity, the above holds for all $i$. Resuming \eqref{eq:prod}, we get

    \begin{equation*}
        \mathbb{E}\left[e^{\epsilon\|Z\|^{2}_{L^{2}(\Omega)}}\right] = \prod_{i=1}^{+\infty}\sqrt{\frac{1}{1-2\epsilon\lambda_{i}}}=\exp\left(-\frac{1}{2}\sum_{i=1}^{+\infty}\log(1-2\epsilon\lambda_{i})\right).
    \end{equation*}
    Since,  for $i\to+\infty$, $-\frac{1}{2}\log(1-2\epsilon\lambda_{i})$ is asymptotic to $\epsilon\lambda_{i}$, which is a summable sequence, the conclusion now follows by \eqref{eq:equiv}.
\end{proof}

\begin{lemma}
    Let $(V,\|\cdot\|)$ be a separable Hilbert space and let $u$ be a squared integrable $V$-valued random variable, $\mathbb{E}\|u\|^{2}<+\infty$. Then, there exists an orthonormal basis $\{v_{i}\}_{i=1}^{+\infty}\subset V$, a sequence of (scalar) random variables $\{\omega_{i}\}_{i=1}^{+\infty}$, with $\mathbb{E}[\omega_{i}\omega_{j}]=\delta_{i,j}$, and
    a nonincreasing summable sequence $\lambda_{1}\ge\lambda_{2}\ge\dots\ge0$
    such that
    $$u=\sum_{i=1}^{+\infty}\sqrt{\lambda_{i}}\omega_{i}v_{i}$$
    almost-surely.
\end{lemma}
\begin{proof}
    Consider the operator $\mathscr{C}:V\to V$ defined as $\mathscr{C}(v)=\expe[\langle u, v\rangle u]$, where the expectation is intended in the Bochner sense. We show that $\mathscr{C}$ is a symmetric semi-positive definite trace class operator. Indeed, for any $v,v'\in V$ one has
    $$\langle \mathscr{C}(v), v'\rangle = \langle\expe[\langle u, v\rangle u],v'\rangle=\expe[\langle u, v\rangle \langle u,v'\rangle] = \langle v, \mathscr{C}(v')\rangle,$$
    and
    $$\langle \mathscr{C}(v), v\rangle = \expe[\langle u, v\rangle^{2}]\ge0.$$
    Furthermore, given any orthonormal basis $\{e_{i}\}_{i=1}^{+\infty}$ we have
    $$\sum_{i=1}^{+\infty}\langle\mathscr{C}(e_{i}),e_{i}\rangle=\sum_{i=1}^{+\infty}\expe[\langle u, e_{i}\rangle^{2}]=\expe\left[\sum_{i=1}^{+\infty}\langle u, e_{i}\rangle^{2}\right]=\expe\|u\|^{2}<+\infty,$$
    by monotone convergence. Thus, by the well-known Spectral Theorem, there exists an orthonormal basis $\{v_{i}\}_{i=1}^{+\infty}\subset V$ and a nonincreasing summable sequence $\lambda_{1}\ge\lambda_{2}\ge\dots\ge0$ such that
    $$\mathscr{C}(v)=\sum_{i=1}^{+\infty}\lambda_{i}\langle v, v_{i}\rangle v_{i}\quad\quad\forall v\in V.$$
    Furthermore, as directly implied by the above, $\mathscr{C}(v_{i})=\lambda_{i}v_{i}$, meaning that the $\lambda_{i}$'s and the $v_{i}$'s are the eigenvalues and eigenvectors of the (uncentered) covariance operator $\mathscr{C}$, respectively.
    Let now $$\omega_{i}:=\frac{1}{\sqrt{\lambda_{i}}}\langle u,v_{i}\rangle.$$
    It is straightforward to see that for all $i,j\in\mathbb{N}$ we have $$\expe[\omega_{i}\omega_{j}]=\frac{1}{\sqrt{\lambda_{i}\lambda_{i}}}\expe[\langle u, v_{i}\rangle\langle u, v_{j}\rangle]=\frac{1}{\sqrt{\lambda_{i}\lambda_{i}}}\langle \mathscr{C}(v_{i}),v_{j}\rangle=\frac{\lambda_{i}}{\sqrt{\lambda_{i}\lambda_{i}}}\langle v_{i},v_{j}\rangle=\delta_{i,j}.$$
    Finally, 
    \begin{equation}
    \label{eq:kklproven}    
    u=\sum_{i=1}^{+\infty}\sqrt{\lambda_{i}}\omega_{i}v_{i}
    \end{equation}
    by definition of the $\omega_{i}$'s. To this end, we also note that, since
    $$\sum_{i=1}^{+\infty}\expe\left\|\sqrt{\lambda_{i}}\omega_{i}v_{i}\right\|^{2}=\expe\left[\sum_{i=1}^{+\infty}\left\|\sqrt{\lambda_{i}}\omega_{i}v_{i}\right\|^{2}\right]=\expe\|u\|^{2}<+\infty,$$
    the series in \eqref{eq:kklproven} is $L^{2}(\mathbb{P}; V)$ convergent (in the Bochner sense \cite{evans2022partial}), where $\mathbb{P}$ is the probability law of $u$. Thus, \eqref{eq:kklproven} holds $\mathbb{P}$-almost surely.
\end{proof}

\section{Complementary results for Section \ref{sec:theory}}
\label{sec:appendix:B}
\newtheoremstyle{lemmaB}
{18pt plus2pt minus1pt}
{18pt plus2pt minus1pt}
{\itshape}
{0pt}
{\bfseries}
{.}
{.5em}
{\thmname{#1} B\thmnumber{#2}%
  \thmnote{ {\the\thm@notefont(#3)}}}
\newtheoremstyle{lemmaC}
{18pt plus2pt minus1pt}
{18pt plus2pt minus1pt}
{\itshape}
{0pt}
{\bfseries}
{.}
{.5em}
{\thmname{#1} C\thmnumber{#2}%
  \thmnote{ {\the\thm@notefont(#3)}}}
\theoremstyle{lemmaB}
\newtheorem{lemmaB}{Lemma}
\newtheorem{corollaryB}{Corollary}
\newtheorem{definitionB}{Definition}
\labelformat{lemmaB}{B#1}
\labelformat{corollaryB}{B#1}
\labelformat{definitionB}{B#1}

\theoremstyle{lemmaC}
\newtheorem{lemmaC}{Lemma}
\newtheorem{corollaryC}{Corollary}
\newtheorem{definitionC}{Definition}
\labelformat{lemmaC}{C#1}
\labelformat{corollaryC}{C#1}
\labelformat{definitionC}{C#1}

\begin{lemmaB}
\label{lemma:convdens}
    Let $(X,\|\cdot\|)$ be a normed space. If $A\subseteq X$ is dense in $X$, then:
    \begin{itemize}
        \item [i)] $A\cap O$ is dense in $O$ for all open sets $O\subseteq X$;\\
        \item [ii)] $A\cap C$ is dense in $C$ for all convex closed sets $C\subseteq X$.
    \end{itemize}
\end{lemmaB}
\begin{proof}\;

\begin{itemize}
\item [i)] Let $O\subseteq X$ be open and let $U$ be an open subset in the subspace topology of $O$. Then, $U=\tilde{U}\cap O$ for some $\tilde{U}\subseteq X$ open in $X$. In particular, $U$ is also open in the topology of the larger space, $X$. Thus, $U\cap A\neq\emptyset$, and the conclusion follows.\\
    
    \item [ii)] Let $C\subseteq X$ be convex and closed. Fix any $c_{0}$ in $\text{int}(C)$, the interior of $C$. Let $c\in C$. It is well known that, under these hypothesis, the segment $\{(1-t)c_{0}+tc\}_{t\in[0,1]}$ can only, at most, intersect $\partial C$ at $c$, as all the remaining points lie in the interior of the set. Let $c_{n}:=c_{0}/n+(1-1/n)c$, so that $c_{n}$ is a sequence of interior points converging to $c$. Since $A\cap\text{int}(C)$ is dense in $\text{int}(C)$, see (i), for every $n$ there exists $c^{*}_{n}\in A\cap\text{int}(C)\in A\cap C$ such that $|c_{n}-c_{n}^{*}|\le1/n.$ Then, $c_{n}^{*}\to c$ as $n\to+\infty$, as wished.
\end{itemize}
\end{proof}

\begin{lemmaB}
    \label{lemma:densityK}
    Let $\sigma$ be a finite measure over $\mathbb{R}^{N}$, and let $B\subset\mathbb{R}^{N}$ be a bounded set. Let $\rho:\mathbb{R}\to\mathbb{R}$ satisfy the assumptions in Lemma \ref{lemma:density}. Consider the following functional space
    $$\mathscr{V}:=\left\{f\in\mis(\mathbb{R}^{N},\mathbb{R}^{n})\;\textnormal{s.t.}\;f_{|B}\in\mathcal{C}(B),\;\|f\|_{\mathscr{V}}:=\|f\|_{\mathcal{C}(B)}+\int_{\mathbb{R}^{N}\setminus B}|f(\x)|\sigma(d\x)<+\infty\right\}.$$
    Then, $\mathscr{N}_{\rho}(\mathbb{R}^{N},\mathbb{R}^{n})$ is $\|\cdot\|_{\mathscr{V}}$-dense in $\mathscr{V}$. 
\end{lemmaB}

\begin{proof}
    For the sake of simplicity, we only prove the case $n=1$. We note that $\mathscr{V}\cong \mathcal{C}(B)\times L^{1}_{\sigma}(\mathbb{R}^{N}\setminus B)$ in the natural way. As a consequence, the dual space of $\mathscr{V}$ can be given as
     $$\mathscr{V}'\cong \mathcal{R}(B)\times L_{\sigma}^{\infty}(\mathbb{R}^{N}\setminus B)$$
     where $\mathcal{R}(B)$ is the set of (signed) Radon measures over $B$, considered with the total variation norm. In particular, for every $F\in\mathscr{V}'$ there exist $\nu\in\mathcal{R}(B)$ and $g\in L_{\sigma}^{\infty}(\mathbb{R}^{N}\setminus B)$ such that
     \begin{equation}
     \label{eq:dualsplit}
         F(f)=\int_{B}fd\nu + \int_{\mathbb{R}^{N}\setminus B}fgd\sigma\end{equation}
     for all $f\in\mathscr{V}.$ Assume that $\mathscr{N}_{\rho}(\mathbb{R}^{N},\mathbb{R}^{n})$ is not dense in $\mathscr{V}$. Then, there exists some $F\in\mathscr{V}'\setminus\{0\}$ such that $F\equiv0$ over $\mathscr{N}_{\rho}(\mathbb{R}^{N},\mathbb{R}^{n})\subset\mathscr{V}$. Let $\nu$ and $g$ be as in \eqref{eq:dualsplit}, and define $\hat{\nu}:=\nu+gd\sigma.$ Then, we have $\hat{\nu}\in\mathcal{R}(\mathbb{R}^{N})$ and
     $$\int_{\mathbb{R}^{N}}fd\hat{\nu}=0\quad\quad\forall f\in\mathscr{N}_{\rho}(\mathbb{R}^{N},\mathbb{R}^{n}).$$
     However, thanks to our assumptions on $\rho$, the above implies $\hat{\nu}\equiv 0\implies F\equiv 0$ (cf. Theorem 5 in \cite{hornik1991approximation}), thus yielding a contraddiction.
\end{proof}

\begin{lemmaB}
    \label{lemma:densityL}
    Let $\sigma$ be a finite measure over $\mathbb{R}^{n}$ and let $Q:=[-M,M]^{n}$ be a given hypercube, $M>0$. Let $\rho:\mathbb{R}\to\mathbb{R}$ be the $\alpha$-leaky ReLU activation, $|\alpha|<1$. Consider the following functional space
    $$\mathscr{W}:=\left\{f\in\mis(\mathbb{R}^{n},\mathbb{R}^{N})\;\textnormal{s.t.}\;\|f\|_{\mathscr{W}}:=\|f\|_{W^{1,\infty}(Q)}+\int_{\mathbb{R}^{n}\setminus Q}|f(\x)|\sigma(d\x)<+\infty\right\}.$$
    Then, $\mathscr{N}_{\rho}(\mathbb{R}^{n},\mathbb{R}^{N})$ is $\|\cdot\|_{\mathscr{W}}$-dense in $\mathscr{W}$. 
\end{lemmaB}

\begin{proof}
    We recall that, since $\rho$ is the $\alpha$-leaky ReLU activation, the set $\mathscr{N}_{\rho}(\mathbb{R}^{n},\mathbb{R}^{N})$ contains all those functions $f:\mathbb{R}^{n}\to\mathbb{R}^{N}$ that are piecewise linear over polyhedra.

    Following the same idea as in the proof of Lemma \ref{lemma:densityK}, let us assume that $\mathscr{N}_{\rho}(\mathbb{R}^{n},\mathbb{R}^{N})$ is not dense in $\mathscr{W}$. Then, there exists some nontrivial functional $F\in\mathscr{W}'$ that vanishes over $\mathscr{N}_{\rho}(\mathbb{R}^{n},\mathbb{R}^{N})$. Since, $\mathscr{W}\cong W^{1,+\infty}(Q)\times L_{\sigma}^{1}(\mathbb{R}^{n}\setminus Q)$ in the natural way, 
    we have 
    $$F(f)=F_{1}(f_{|Q})+F_{2}(f_{|\mathbb{R}^{n}\setminus Q})$$
    for some $F_{1}\in W^{1,+\infty}(Q)'$ and $F_{2}\in L_{\sigma}^{1}(\mathbb{R}^{n}\setminus Q)'.$ Let now $g\in W^{1,+\infty}(Q)$. Since $Q$ is a polyhedron, it is straightforward to see that for every $\varepsilon>0$ there exists $\phi_{\varepsilon}\in\mathscr{N}_{\rho}(\mathbb{R}^{n},\mathbb{R}^{N})$ such that
    $$\|\phi_{\varepsilon}-g\|_{W^{1,\infty}(Q)}<\varepsilon,\quad\|\phi_{\varepsilon}\|_{\mathcal{C}(\mathbb{R}^{n})}\le\|g\|_{\mathcal{C}(Q)},\quad\phi_{\varepsilon}\equiv0\;\;\textnormal{on}\;\;\mathbb{R}^{n}\setminus Q_{\varepsilon},$$
    where $Q_{\varepsilon}:=(-M-\varepsilon,\;M+\varepsilon)^{n}.$ In fact, such a piecewise linear approximation is easily constructed and it is guaranteed to be a member of $\mathscr{N}_{\rho}(\mathbb{R}^{n},\mathbb{R}^{N})$. Then, $\phi_{\varepsilon}\to g\cdot\mathbbm{1}_{Q}$ in $\|\cdot\|_{\mathscr{W}}$-norm as $\varepsilon\to 0$. Thus, we have
    $$F_{1}(g)=F_{1}(g)+F_{2}(0)=\lim_{\varepsilon\to 0}F(\phi_{\varepsilon})= 0,$$
    proving that $F_{1}\equiv0.$ In particular, for $F$ to vanish over $\mathscr{N}_{\rho}(\mathbb{R}^{n},\mathbb{R}^{N})$ we must have
    $$F_{2}(\phi_{|\mathbb{R}^{n}\setminus Q})=0\quad\quad\forall\phi\in\mathscr{N}_{\rho}(\mathbb{R}^{n},\mathbb{R}^{N}).$$
    However, since $\rho$ satisfies the assumptions in Lemma 1, this would imply $F_{2}\equiv0$, ultimately yielding a contraddiction.
\end{proof}

\begin{corollaryB}
    \label{corollary:density} Let $\sigma$ be a probability measure over $\mathbb{R}^{N},$ with finite moment and absolutely continuous with respect to the Lebesgue measure. Let $B\subset\mathbb{R}^{N}$ be a bounded set, and let $M,L>0$, $n\in\mathbb{N}$. Consider the families
    \begin{equation*}
    \encoders_{B,M}(\mathbb{R}^{N},\;\mathbb{R}^{n}):=\left\{\Psi'\in\mis(\mathbb{R}^{N},\;\mathbb{R}^{n})\;\;\textnormal{s.t.}\;\;\int_{\mathbb{R}^{N}}|\Psi'(\mathbf{v})|\sigma(d\mathbf{v})<+\infty,\;\;\sup_{\mathbf{v}\in B}|\Psi'(\mathbf{v})|\le M\right\},
\end{equation*}
\begin{equation*}
    \decoders_{M,L}(\mathbb{R}^{n},\;V):=\left\{\Psi\in\mis(\mathbb{R}^{n},\;\mathbb{R}^{N})\;\;\textnormal{s.t.}\;\;\sup_{\mathbf{c}\in [-M,M]^{n}}\|\partial \Psi\|_{V}(\mathbf{c})\le L\right\}.
\end{equation*}
Let $\rho$ be the $\alpha$-leakyReLU, with $|\alpha|<1.$
Then, for every $\varepsilon>0$ and every pair $\Psi'\in\encoders_{B,M}(\mathbb{R}^{N},\mathbb{R}^{n})$, $\Psi\in\decoders_{M,L}(\mathbb{R}^{n},\mathbb{R}^{N})$ such that
\begin{equation}
    \label{eq:finiteRE}
    \int_{\mathbb{R}^{N}}|\mathbf{u}-\Psi(\Psi'(\mathbf{u}))|\sigma(d\mathbf{u})<+\infty,
\end{equation}
there exists $\hat{\Psi}'\in\encoders_{B,M}(\mathbb{R}^{N},\mathbb{R}^{n})\cap\mathscr{N}_{\rho}(\mathbb{R}^{N},\mathbb{R}^{n})$ and $\hat{\Psi}\in\decoders_{M,L}(\mathbb{R}^{n},\mathbb{R}^{N})\cap \mathscr{N}_{\rho}(\mathbb{R}^{n},\mathbb{R}^{N})$ such that
$$\left|\int_{\mathbb{R}^{N}}|\mathbf{u}-\Psi(\Psi'(\mathbf{u}))|\sigma(d\mathbf{u})-\int_{\mathbb{R}^{N}}|\mathbf{u}-\hat{\Psi}(\hat{\Psi}'(\mathbf{u}))|\sigma(d\mathbf{u})\right|<\varepsilon,$$
\end{corollaryB}
\begin{proof}
    Let $\Psi'\in\encoders_{B,M}(\mathbb{R}^{N},\mathbb{R}^{n})$ and $\Psi\in\decoders_{M,L}(\mathbb{R}^{n},\mathbb{R}^{N})$ satisfy \eqref{eq:finiteRE}, and let $\varepsilon>0.$ 
    Define $\nu$ as the push-forward measure of $\sigma$ through the encoder $\Psi'$, and let $(\mathscr{W},\|\cdot\|_{\mathscr{W}})$ be the normed space in Lemma \ref{lemma:densityL}, defined with respect to $\nu$ and $M.$ We note that, since $\sigma$ has finite moment, Eq. \eqref{eq:finiteRE} implies
    $$\int_{\mathbb{R}^{N}}|\Psi(\mathbf{c})|\nu(d\mathbf{c})<+\infty\implies\Psi\in L^{1}_{\nu}(\mathbb{R}^{n}).$$
    In particular, $\Psi\in\decoders_{M,L}(\mathbb{R}^{n},\mathbb{R}^{N})\cap L^{1}_{\nu}(\mathbb{R}^{n})\subseteq\mathscr{W}$. Since the set $\decoders_{M,L}(\mathbb{R}^{n},\mathbb{R}^{N})\cap L^{1}_{\nu}(\mathbb{R}^{n})$ is both convex and closed in $(\mathscr{W},\|\cdot\|_{\mathscr{W}})$, by Lemmas \ref{lemma:convdens} and \ref{lemma:densityL}, there exists some $\hat{\Psi}\in\decoders_{M,L}(\mathbb{R}^{n},\mathbb{R}^{N})\cap \mathscr{N}_{\rho}(\mathbb{R}^{n},\mathbb{R}^{N})\cap L^{1}_{\nu}(\mathbb{R}^{n})$ such that $\|\Psi-\hat{\Psi}\|_{\mathscr{W}}<\varepsilon/2.$ Let now $\ell$ be the (global) Lipschitz constant of $\hat{\Psi}$. Following the same computations as in the proof of Theorem \ref{theorem:equiv}, we get 
    \begin{multline}
        \label{eq:nearlydone}
        \left|\int_{\mathbb{R}^{N}}|\mathbf{u}-\Psi(\Psi'(\mathbf{u}))|\sigma(d\mathbf{u})-\int_{\mathbb{R}^{N}}|\mathbf{u}-\hat{\Psi}(\hat{\Psi}'(\mathbf{u}))|\sigma(d\mathbf{u})\right|\le\\\le\int_{\mathbb{R}^{n}}|\Psi(\mathbf{c})-\hat{\Psi}(\mathbf{c})|\nu(d\mathbf{c})+\ell\int_{\mathbb{R}^{N}}|\Psi'(\mathbf{v})-\hat{\Psi'}(\mathbf{v})|\sigma(d\mathbf{v})\le\\
        \le \|\Psi(\mathbf{c})-\hat{\Psi}(\mathbf{c})\|_{\mathscr{W}}+\ell\int_{\mathbb{R}^{N}}|\Psi'(\mathbf{v})-\hat{\Psi'}(\mathbf{v})|\sigma(d\mathbf{v})\le\\
        \le \frac{\varepsilon}{2}+\ell\int_{\mathbb{R}^{N}}|\Psi'(\mathbf{v})-\hat{\Psi'}(\mathbf{v})|\sigma(d\mathbf{v}),
    \end{multline}
    for all $\hat{\Psi}'\in\mathscr{N}_{\rho}(\mathbb{R}^{N},\mathbb{R}^{n})$, where we exploited the fact that $\nu([-M,M]^{n})\le1$ due to $\nu$ being a probability measure.
    \\\\
    To conclude the proof, we now wish to bound the second term in \eqref{eq:nearlydone}. This, however, requires some additional care: in fact, we cannot directly repeat the same ideas and apply Lemma \ref{lemma:densityK}, as $\Psi'$ might be discontinuous over $B$. To account for this, we shall introduce a proper smoothing step. Let $$\varphi(\x):=C\exp^{-1/(1-|\x|^{2})}\mathbbm{1}_{\{|\y|\le1\}}(\x)$$ be the canonical mollifier in $\mathbb{R}^{N}$, where $C>0$ ensures that $\int_{\mathbb{R}^{N}}\varphi(\x)d\x=1.$ For any $\delta>0$, let $\Psi'_{\delta}:\mathbb{R}^{N}\to\mathbb{R}^{n}$ be defined as
    $$\Psi'_{\delta}(\mathbf{v}):=\Psi'(\mathbf{v})\mathbbm{1}_{\mathbb{R}^{N}\setminus B}(\mathbf{v})+(\Psi'*\varphi_{\delta})\mathbbm{1}_{B}(\mathbf{v}),$$
    where $\varphi_{\delta}(\x):=\delta^{-N}\varphi(\x/\delta)$ and $*$ denotes the convolution operator. Then, following classical arguments (see, e.g. Theorem 7, Appendix C.2, in \cite{evans2022partial}), it is straightforward to see that:
    \begin{itemize}
        \item [(i)] $\Psi'_{\delta}$ is continuous over $B$;\\
        \item [(ii)] $\Psi'_{\delta}\in\encoders_{B,M}(\mathbb{R}^{N},\mathbb{R}^{n})$, as $|\Psi'_{\delta}(\mathbf{v})|\le M\cdot\int_{\mathbb{R}^{N}}\varphi_{\delta}(\mathbf{v})\sigma(d\mathbf{v})=M$ for all $\mathbf{v}\in B$;\\
        \item [(iii)] $\Psi'_{\delta}\to\Psi'$ $\sigma$-almost everywhere, as $\sigma$ is absolutely continuous with respect to Lebesgue's measure;\\
        \item [(iv)] $\int_{\mathbb{R}^{N}}|\Psi'_{\delta}(\mathbf{v})-\Psi'(\mathbf{v})|\sigma(d\mathbf{v})\to0$ as $\delta\to0$, due to dominated convergence and (iii).

    \end{itemize}
    \;\\In light of (iv), let us fix $\delta>0$ such that $\int_{\mathbb{R}^{N}}|\Psi'_{\delta}(\mathbf{v})-\Psi'(\mathbf{v})|\sigma(d\mathbf{v})\le\varepsilon/4\ell.$ As before, let $(\mathscr{V},\|\cdot\|_{\mathscr{V}})$ be the normed space in Lemma \ref{lemma:densityK}, defined according to the measure $\sigma$. Then, due to (i) and (ii), we have $\Psi'_{\delta}\in\encoders_{B,M}(\mathbb{R}^{N},\mathbb{R}^{n})\cap\mathscr{V}$, with the latter being a convex closed subset of $(\mathscr{V},\|\cdot\|_{\mathscr{V}})$. In particular, by Lemmas \ref{lemma:convdens} and \ref{lemma:densityK}, there exists some $$\hat{\Psi'}_{\delta}\in\mathscr{N}_{\rho}(\mathbb{R}^{N},\mathbb{R}^{n})\cap\encoders_{B,M}(\mathbb{R}^{N},\mathbb{R}^{n})\cap\mathscr{V}$$ such that $\|\Psi'_{\delta}-\hat{\Psi}'_{\delta}\|_{\mathscr{V}}<\varepsilon/4\ell.$ Consequently,
    \begin{equation}
        \label{eq:nowitsdone}
        \int_{\mathbb{R}^{N}}|\Psi'(\mathbf{v})-\hat{\Psi'}_{\delta}(\mathbf{v})|\sigma(d\mathbf{v})\le\int_{\mathbb{R}^{N}}|\Psi'(\mathbf{v})-\Psi'_{\delta}(\mathbf{v})|\sigma(d\mathbf{v})+\|\Psi'_{\delta}-\hat{\Psi}'_{\delta}\|_{\mathscr{V}}<\frac{\varepsilon}{4l}
    \end{equation}
    where, as before, we exploited the fact that $\sigma$ is a probability measure, and thus $\sigma(B)\le1.$ Then, setting $\hat{\Psi}:=\hat{\Psi}_{\delta}$ and plugging \eqref{eq:nowitsdone} into \eqref{eq:nearlydone} yields the desired conclusion.    
\end{proof}

\section{Measurable selections and more}

In what follows, we use the term \textit{Polish space} to intend a complete separable metric space. We recall, in particular, that all separable Banach spaces are Polish spaces. Finally, given any set $X$, we shall write $2^{X}$ for its power set, that is, the collection of all subsets of $X$,
$$2^{X}:=\{A\;\;|\;\;A\subseteq X\}.$$

\begin{definitionC}
\label{def:setvalued}
Let  $(Y,\mathscr{M})$ be a measurable space and let $(X,d)$ be a Polish space. Let $F:Y\to 2^{X}$. We say that $F$ is a measurable set-valued map if the following conditions hold:
\begin{itemize}
    \item [a)] $F(y)$ is closed in $X$ for all $y\in Y$;\\
    \item [b)] for all open sets $A\subseteq X$ one has $\mathcal{S}_{A}\in\mathscr{M}$, where
\begin{equation}
    \label{eq:sets}
    \mathcal{S}_{A}:=\{y\in Y\;|\;F(y)\cap A\neq\emptyset\}.
\end{equation}
\end{itemize}
\end{definitionC}

\begin{lemmaC}
    \label{lemma:measurability}
    Let $(X,d)$ be a Polish space and let $(Y,\tau)$ be a topological space equipped with a suitable $\sigma$-field $\mathscr{M}$. Let $f:X\to Y$ be continuous. The following are equivalent:
    \begin{itemize}
        \item [i)] $f$ maps open sets onto $\mathscr{M}$-measurable sets;\\
        \item [ii)] the map $F:Y\to2^{X}$, defined as
        \begin{equation}
        \label{eq:Fmap}
        F(y):=
        \begin{cases}
        \{x\in X\;|\;f(x)=y\} &y\in f(X)\\
        X & \text{otherwise},
    \end{cases}
        \end{equation}
        is a measurable set-valued map.
    \end{itemize}
\end{lemmaC}

\begin{proof}
    Let $F$ be as in \eqref{eq:Fmap}. We first note that, since $f$ is continuous, the preimage of any singleton, $f^{-1}(\{y\})$, is closed. In particular, condition (a) in Definition \ref{def:setvalued} is met. Next, we note that for any $A\subseteq X$ one has
    \begin{equation}
        \label{eq:selection}
    \mathcal{S}_{A}=f(A)\cup f(X)^{c},\end{equation}
    where $B^{c}:=Y\setminus B$ denotes the complement of $B\subseteq Y$, and $\mathcal{S}_{A}$ is as in \eqref{eq:sets}.
    In fact,
    $$y\notin f(X)\implies F(y)=X\implies F(y)\cap A = A\neq\emptyset\implies y\in\mathcal{S}_{A},$$
    meaning that $\mathcal{S}_{A}\cap f(X)^{c}=f(X)^{c}.$ On the other hand,
    $$y\in f(X)\cap\mathcal{S}_{A}\iff \exists a\in A\text{ s.t. } a\in \{x\in X\;|\;f(x)=y\}\iff y\in f(A),$$
    implying that
  $\mathcal{S}_{A}\cap f(X)=f(A).$
    Since $\mathcal{S}_{A}=(\mathcal{S}_{A}\cap f(X)^{c})\cup (\mathcal{S}_{A}\cap f(X)),$ the identity in \eqref{eq:selection} easily follows.\\\\
    At this point, it is straightforward to see that $(i)\iff(ii)$. Assume, for instance, that $(i)$ holds. Then, for any open set $A\subseteq X$, $f(A)$ is $\mathscr{M}$-measurable. In particular, $f(X)\in\mathscr{M}\implies f(X)^{c}\in\mathscr{M}\implies \mathcal{S}_{A}\in\mathscr{M}$, meaning that condition (b) in Definition \ref{def:setvalued} is met. Conversely, say that $F$ is measurable, so that $\mathcal{S}_{A}\in\mathscr{M}$ for all open sets $A\subseteq X$. Let $A=\emptyset$. Then, $f(X)^{c}\in\mathscr{M}$. Since $f(A)$ and $f(X)^{c}$ are disjoint, $f(A)=\mathcal{S}_{A}\setminus f(X)^{c}$, proving that $f(A)\in\mathscr{M}$, as claimed.
\end{proof}

\begin{lemmaC}
    \label{lemma:analytic}
    Let $f:(X,d_{X})\to(Y,d_{Y})$ be a continuous map between two Polish spaces. Let $\mathscr{M}$ be the $\mathbb{P}$-completion of the Borel $\sigma$-field defined over $Y$, where $\mathbb{P}$ is a given probability distribution. Then $f$ maps open sets onto $\mathscr{M}$-measurable sets.
\end{lemmaC}
\begin{proof}
    This is a standard result in the theory of Analytic sets, see, e.g., Theorem 4.3.1 in \cite{srivastava2008course}. In fact, any open set $A\subseteq X$ is Borel measurable in $X$. Thus, by very definition, $f(A)$ is an analytic set in $Y$. Since $\mathscr{M}$ contains all Borel sets of $Y$ and it is also $\mathbb{P}$-complete, it follows that $f(A)\in\mathscr{M}$. 
\end{proof}

\begin{corollaryC}
    \label{cor:selection0}
    Let $f:(X,d_{X})\to(Y,d_{Y})$ be a continuous map between two Polish spaces. Let $\mathscr{M}$ be the $\mathbb{P}$-completion of the Borel $\sigma$-field defined over $Y$, where $\mathbb{P}$ is a given probability distribution. Then, $f$ admits an $\mathscr{M}$-measurable right-inverse, that is, a measurable map $f^{-1}:Y\to X$ for which
    $$f(f^{-1}(y))=y\quad\quad\forall y\in f(X).$$
\end{corollaryC}
\begin{proof}
    The nontrivial part of the statement lies in the measurability of $f^{-1}$, as the existence of a generic right-inverse is already guaranteed by the Axiom of Choice. Let $F:Y\to 2^{X}$ be as in \eqref{eq:Fmap}, so that, according to Lemmae \ref{lemma:measurability} and \ref{lemma:analytic}, $F$ is a measurable set-valued map. Then, a famous result by Kuratowski and Ryll-Nardzewski states that $F$ admits a measurable selection (see, e.g., Theorem 8.1.3 and Definition 8.1.1 in \cite{aubin2009set}). That is, there exists a measurable map $g:Y\to X$ such that $g(y)\in F(y)$ for all $y\in Y$. It is straightforward to see that such map retains all the desired properties: in fact, for all $y\in f(X)$ the condition $g(y)\in F(y)$ implies $f(g(y))=y.$
\end{proof}

\begin{corollaryC}
    \label{corollary:selection}
    Let $f:(X,d_{X})\to(Y,d_{Y})$ be a continuous map between two Polish spaces. Let $\mathscr{M}$ be the $\mathbb{P}$-completion of the Borel $\sigma$-field defined over $Y$, where $\mathbb{P}$ is a given probability distribution. Let $C\subseteq X$ be a closed subset. Then, there exists an $\mathscr{M}$-measurable map $f^{-1}:Y\to X$ such that 
    $$f(f^{-1}(y))=y\quad\forall y\in f(X),\quad\textnormal{and}\quad f^{-1}(f(C))\subseteq C.$$
\end{corollaryC}
\begin{proof}
    In agreement with Corollary \ref{cor:selection0}, let $g_{0}:Y\to X$ be an $\mathscr{M}$-measurable right-inverse of $f$. Consider the metric subspace $(C,d)\subseteq(X,d)$. Since $C$ is closed, $(C,d)$ is a Polish space. Thus, we may exploit Corollary \ref{cor:selection0} once again to construct an $\mathscr{M}$-measurable map $g_{1}:Y\to C$ that operates as a right-inverse of $f_{|C}$, the restriction of $f$ to $C$. Define $g:Y\to X$ as
    $$g(y):=g_{1}(y)\cdot\mathbbm{1}_{f(C)}(y)+g_{0}(y)\cdot\mathbbm{1}_{Y\setminus f(C)}(y).$$
    Then $f^{-1}:=g$ fulfills all the requirements.
\end{proof}

\begin{lemmaC}
    \label{lemma:optselec}
    Let $(X,d_{X})$ be a Polish space and let $(C,d_{C})$ be a compact metric space. Let $J:X\times C\to \mathbb{R}$ be continuous. Then, there exists a Borel measurable map $f:X\to C$ such that
    $$J(x,f(x))=\min_{c\in C}J(x,c)$$
    for all $x\in X$.    
\end{lemmaC}

\begin{proof} 
First of all, we note that the statement is well-defined as all the minima are attained by compactness of $C$ and continuity of $J$. Let now $F:X\to 2^{C}$ be the following set-valued map
$$F:x\to\left\{c\in C\;\text{such that}\;J(x,c)=\min_{c'\in C}J(x,c')\right\},$$
that assigns a (nonempty) subset of $C$ to each $x\in X$. We aim at showing that $F$ is a measurable set-valued map. To this end, we start by noting that $F(x)$ is closed in $C$ for all $x\in X$. To see this, fix any $x\in X$ and let $j_{x}:=J(x,\cdot)$, so that $j_{x}:C\to\mathbb{R}$ is continuous. Then,
$$F(x)=j_{x}^{-1}\left(\left\{\min_{c'\in C}j_{x}(c')\right\}\right)$$
is closed as it is the pre-image of a singleton under a continuous transformation. 

Conversely, we now claim that, for any compact subset $K\subseteq C$, the set
$$\mathcal{S}_{K}:=\{x\in X\;:\;F(x)\cap K\neq\emptyset\}$$
is closed. Indeed, let $\{x_{n}\}_{n}\subseteq \mathcal{S}_{K}$ be a sequence converging to some $x\in X$. By definition of $\mathcal{S}_{K}$, for each $x_{n}$ there exists a $c_{n}\in K$ such that $c_{n}\in F(x_{n})$, i.e. for which $J(x_{n},c_{n})=\min_{c'} J(x_{n},c')$. Since $K$ is compact, up to passing to a subsequence, there exists some $c\in K$ such that $c_{n}\to c$. Let now $\tilde{c}\in C$ be a minimizer for $x$, i.e. a suitable element for which $J(x,\tilde{c})=\min_{c'\in C}J(x,c')$. By continuity, we have
\begin{equation*}    J(x,c)=\lim_{n\to+\infty}J(x_{n},c_{n})=\lim_{n\to+\infty}\min_{c' \in C}J(x_{n},c')\le\lim_{n\to+\infty}J(x_{n},\tilde{c})=J(x,\tilde{c}),
\end{equation*}
implying that $c$ is also a minimizer for $x$. As a consequence, we have $c\in K\cap F(x)$ and thus $x\in \mathcal{S}_{K}$. In particular, $\mathcal{S}_{K}$ is closed. It is now straightforward to prove that $\mathcal{S}_{A}$ is Borel measurable whenever $A\subseteq C$ is open. In fact, any open set $A\subseteq C$ can be written as the countable union of compact sets, $A=\cup_{n\in\mathbb{N}}K_{n}$, and clearly $\mathcal{S}_{A}=\cup_{n}\mathcal{S}_{K_{n}}.$

We have then proven that $F$ is a measurable set-valued map. In particular, we may now invoke the measurable selection theorem by Kuratowski–Ryll-Nardzewski  \cite{aubin2009set}, which ensures the existence of a measurable map $f:X\to C$ such that $f(x)\in F(x)$, i.e. $J(x,f(x))=\min_{c\in C}J(x,c)$, as wished.    
\end{proof}

\section{Architectures and training details}
\label{sec:architectures}

We report in this Section the technical details concerning the design of the autoencoder modules (Tables \ref{tab:design1}-\ref{tab:design2}) and the training of the DL-ROMs (Table \ref{tab:training}). As for the latter, Tables \ref{tab:phi1}-\ref{tab:phi2} complete the picture, with a description of the arcitectures employed for reduced map, $\phi$.

We mention that, both for the case of Darcy flow in a porous medium and Burger's equation, we exploited a combination of classical dense layers together with mesh-informed layers. The latter are a particular class of sparse architectures first introduced in \cite{franco2023mesh} as a way to handle mesh-based functional data. In short, they are obtained through sparsification of dense architectures, which is achieved by means of a mesh-dependent pruning strategy: for further details, we refer the reader to \cite{franco2023mesh}. As for our purposes, it is sufficient to know that mesh-informed layers are characterized by a \textit{support} hyperparameter: the smaller the support, the sparser the architecture; for large supports, the module collapses to a classical dense layer. 

\begin{table}[h!]
    \centering
    \begin{tabular}{lllll}
    \hline
        \bf Layer \# & \bf Type & \bf Input dim. & \bf Output dim. & \bf Activation\\
        \hline\hline
        1 & Dense & 2601 & 500 & $\rho$\\
        2 & Dense & 500 & $n$ & $\rho$\\\hdashline
        3 & Dense & $n$ & 500 & $\rho$\\
        4 & Dense & 500 & 2601 & -
        \\\hline\\
    \end{tabular}
    \caption{Autoencoder architecture for Darcy's law example,  §\ref{subsec:diff}. The encoder and decoder modules, here presented together, are divided by a dashed line. $n$ = latent dimension; $\rho$ = 0.1-leakyReLU.
    }
    \label{tab:design1}
\end{table}

\begin{table}[h!]
    \centering
    \begin{tabular}{lllll}
    \hline
        \bf Layer \# & \bf Type & \bf Input dim. & \bf Output dim. & \bf Activation\\
        \hline\hline
        1 & Dense & 500 & $n$ & $\rho$\\\hdashline
        2 & Dense & $n$ & 200 & $\rho$\\
        3 & Dense & 200 & 500 & $\tilde{\rho}$
        \\\hline\\
    \end{tabular}
    \caption{Autoencoder architecture for Burger's example,  §\ref{subsec:burger}. Entries read as in Table \ref{tab:design1}; $\tilde{\rho}(x):=\rho(0.5-\rho(0.5-x)).$
    }
    \label{tab:design2}
\end{table}

\begin{table}[h!]
    \centering
    \begin{tabular}{llllll}
    \hline
        \bf Layer \# & \bf Type & \bf Input dim. & \bf Output dim. & \bf Support & \bf Activation\\
        \hline\hline
        1 & Mesh-informed & 2601 & 676 & 0.125 & tanh\\
        2 & Mesh-informed & 676 & 169 & 0.25 & $\rho$\\
        3 & Dense & 169 & $n$ & - & $\rho$\\\hline\\
    \end{tabular}
    \caption{Architecture of the reduced map network, $\phi$, for Darcy's law example,  §\ref{subsec:diff}. Mesh-informed layers are constructed by relying on uniform structured grids (with as many dofs as declared in the input-output dimensions). $n$ = latent dimension; $\rho$ = 0.1-leakyReLU.
    }
    \label{tab:phi1}
\end{table}

\begin{table}[h!]
    \centering
    \begin{tabular}{llllll}
    \hline
        \bf Layer \# & \bf Type & \bf Input dim. & \bf Output dim. & \bf Support & \bf Activation\\
        \hline\hline
        1 & Mesh-informed & 500 & 250 & 0.25 & $\rho$\\
        2 & Mesh-informed & 250 & 125 & 0.5 & $\rho$\\
        3 & Dense & 125 & $n$ & - & $\rho$\\\hline\\
    \end{tabular}
    \caption{Architecture of the reduced map network, $\phi$, for Burger's example,  §\ref{subsec:burger}. Mesh-informed layers are constructed by relying on uniform grids (with as many dofs as declared in the input-output dimensions). Entries read as in Table \ref{tab:phi1}.
    }
    \label{tab:phi2}
\end{table}

\begin{table}[h!]
    \centering
    \begin{tabular}{llllll}
    \hline
         \bf Problem &   $\alpha_{1}$ & $\alpha_{2}$ & $\alpha_{3}$ & \bf Training time & \bf Epochs\\
         \hline\hline
         Darcy's law  §\ref{subsec:diff}& 1/5 & 1/5 & 1/16 & 
         7m 20.17s & 300 \\
         Burger's  §\ref{subsec:burger} & rel & 1 & 1/16 & 2m 22.40s & 300\\
         \hline\\
    \end{tabular}
    \caption{Training of the DL-ROMs: technical details. The $\alpha_{j}$'s define the loss function as in Eq. \eqref{eq:loss}; "rel" means that the term $\alpha_{j}\|\y_{i}-\hat{\y}_{i}\|^{2}$ is replaced with a relative error $\|\y_{i}-\hat{\y}_{i}\|/\|\y_{i}\|$.}
    \label{tab:training}
\end{table}

\end{appendices}


\bibliography{sn-bibliography}

\end{document}